\documentclass[10pt]{article} 

\usepackage[accepted]{rlj} 
\usepackage{graphicx} 
\usepackage{amsfonts}
\usepackage{amsthm}
\usepackage{amssymb}
\newtheorem{theorem}{Theorem}

\newtheorem{assumption}{Assumption}
\hypersetup{
    colorlinks=true, 
    linkcolor=blue,  
    filecolor=magenta, 
    urlcolor=cyan,    
    citecolor=blue    
}
\usepackage{breqn} 
\usepackage{subcaption}
\usepackage{algorithm}
\usepackage[noend]{algorithmic}
\newcommand{\argmax}{\mathop{\rm arg~max}\limits}
\newcommand{\argmin}{\mathop{\rm arg~min}\limits}
\definecolor{scarlet}{RGB}{255, 36, 0}   
\definecolor{violet}{RGB}{143, 0, 255}   

\title{Which Experiences Are Influential for RL Agents?\\ Efficiently Estimating The Influence of Experiences}

\setrunningtitle{Which Experiences Are Influential for RL Agents?}

\author{Takuya Hiraoka\textsuperscript{1,2}, Guanquan Wang\textsuperscript{2,3}, Takashi Onishi\textsuperscript{1,2},\\ Yoshimasa Tsuruoka\textsuperscript{2,3}}

\emails{ \{takuya-h1, takashi.onishi\}@nec.com, \\ \{guanquan-wang@g.ecc, tsuruoka@logos.t\}.u-tokyo.ac.jp}

\affiliations{
$^{1}$\textbf{NEC Corporation}\\ 
$^{2}$\textbf{National Institute of Advanced Industrial Science and Technology}\\ 
$^{3}$\textbf{The University of Tokyo}
}

\contribution{
For the first time, we propose a method that efficiently (i) estimates the influence of individual experiences (i.e., data) on the performance (e.g., empirical returns) of an RL agent and (ii) disables that influence when necessary (Section~\ref{sec:proposed_method}). 

\textbf{Why is this contribution valuable?} 
In many RL settings, we must manage experiences of different quality levels. 
For example, in an off-policy RL setting, experiences collected from multiple policies—ranging from random to near-optimal—are used to learn policies or Q-functions. 
When experiences of different quality are intermixed, the ability to estimate their influence on performance and disable any harmful influences is highly beneficial for many purposes. 
For instance, (i) if an RL agent’s performance is degraded by specific detrimental experiences, disabling their influence can help improve the agent’s overall performance. 
(ii) In safety-critical applications (e.g., human-in-the-loop robotics or autonomous driving), this ability may ensure safety by disabling the influence of experiences that degrade safety performance before deployment. 
In addition, (iii) in memory-intensive scenarios like image-based RL, where each experience consumes substantial computational memory, this ability may enable efficient management of experiences by screening out less useful experiences. 
Finally, (iv) when refining RL task design, analyzing influential experiences may provide valuable insights for improving reward functions or state representations, leading to better task design.\\ 
}
{
(i) No prior work has addressed the efficient estimation and disabling of the influence of experiences in the online RL context. 
(ii) As a first step, this paper focuses on verifying the proposed method’s effectiveness within single-task off-policy RL settings (MuJoCo and DMC) (Section~\ref{sec:experiments} and \ref{sec:application}, and Appendix~\ref{app:adversarial_dmc}). 
}

\keywords{reinforcement learning, data influence estimation} 

\summary{
In reinforcement learning (RL) with experience replay, experiences stored in a replay buffer influence the RL agent's performance. 
Information about how these experiences influence the agent's performance is valuable for various purposes, such as identifying experiences that negatively influence underperforming agents. 
One method for estimating the influence of experiences is the leave-one-out (LOO) method. 
However, this method is usually computationally prohibitive. 
In this paper, we present Policy Iteration with Turn-over Dropout (PIToD), which efficiently estimates the influence of experiences. 
We evaluate how correctly PIToD estimates the influence of experiences and its efficiency compared to LOO. 
We then apply PIToD to amend underperforming RL agents, i.e., we use PIToD to estimate negatively influential experiences for the RL agents and to delete the influence of these experiences. 
We show that RL agents' performance is significantly improved via amendments with PIToD. 
Our code is available at: \url{https://github.com/TakuyaHiraoka/Which-Experiences-Are-Influential-for-RL-Agents}
}

\begin{document}

\makeCover  
\maketitle  

\begin{abstract}
In reinforcement learning (RL) with experience replay, experiences stored in a replay buffer influence the RL agent's performance. 
Information about how these experiences influence the agent's performance is valuable for various purposes, such as identifying experiences that negatively influence underperforming agents. 
One method for estimating the influence of experiences is the leave-one-out (LOO) method. 
However, this method is usually computationally prohibitive. 
In this paper, we present Policy Iteration with Turn-over Dropout (PIToD), which efficiently estimates the influence of experiences. 
We evaluate how correctly PIToD estimates the influence of experiences and its efficiency compared to LOO. 
We then apply PIToD to amend underperforming RL agents, i.e., we use PIToD to estimate negatively influential experiences for the RL agents and to delete the influence of these experiences. 
We show that RL agents' performance is significantly improved via amendments with PIToD. 
Our code is available at: \url{https://github.com/TakuyaHiraoka/Which-Experiences-Are-Influential-for-RL-Agents}
\end{abstract}

\section{Introduction}\label{sec:introduction}
In reinforcement learning (RL) with experience replay, the performance of an RL agent is influenced by experiences. 
Experience replay~\citep{lin1992self} is a data-generation mechanism indispensable in modern off-policy RL methods~\citep{mnih2015human,hessel2018rainbow,haarnoja2018softa,NEURIPS2020}. 
It allows an RL agent to learn from past experiences. 
These experiences influence the RL agent's performance (e.g., cumulative reward)~\citep{fedus2020revisiting}. 
Estimating how each experience influences the RL agent's performance could provide useful information for many purposes. 
For example, we could improve the RL agent's performance by identifying and deleting negatively influential experiences. 
The capability to estimate the influence of experience will be crucial, as RL is increasingly applied to tasks where agents must learn from experiences of diverse quality (e.g., a mixture of experiences from both expert and random policies)~\citep{fu2020d4rl,yu2020meta,agarwal2022reincarnating,TWiRLSmith2023,liu2024offlinesaferl,tirumala2024replay}. 

However, estimating the influence of experiences with feasible computational cost is not trivial. 
One might consider estimating it by a leave-one-out (LOO) method (left part of Figure~\ref{fig:Summary}), which retrains an RL agent for each possible experience deletion. 
As we will discuss in Section~\ref{sec:problem_of_loo}, this method has quadratic time complexity and quickly becomes intractable due to the necessity of retraining. 

In this paper, we present PIToD, a policy iteration (PI) method that efficiently estimates the influence of experiences (right part of Figure~\ref{fig:Summary}). 
PI is a fundamental method for many RL methods (Section~\ref{sec:preliminaries}). 
PIToD is PI augmented with turn-over dropout (ToD)~\citep{kobayashi2020efficient} to efficiently estimate the influence of experiences without retraining an RL agent (Section~\ref{sec:proposed_method}). 
We evaluate how correctly PIToD estimates the influence of experiences and its efficiency compared to the LOO method (Section~\ref{sec:experiments}). 
We then apply PIToD to amend underperforming RL agents by identifying and deleting negatively influential experiences (Section~\ref{sec:application}). 
To our knowledge, our work is the first to: (i) estimate the influence of experiences on the performance of RL agents with feasible computational cost, and (ii) modify RL agents' performance simply by deleting influential experiences. 
\begin{figure}[t!]
\begin{center}
\includegraphics[clip, width=0.99\hsize]{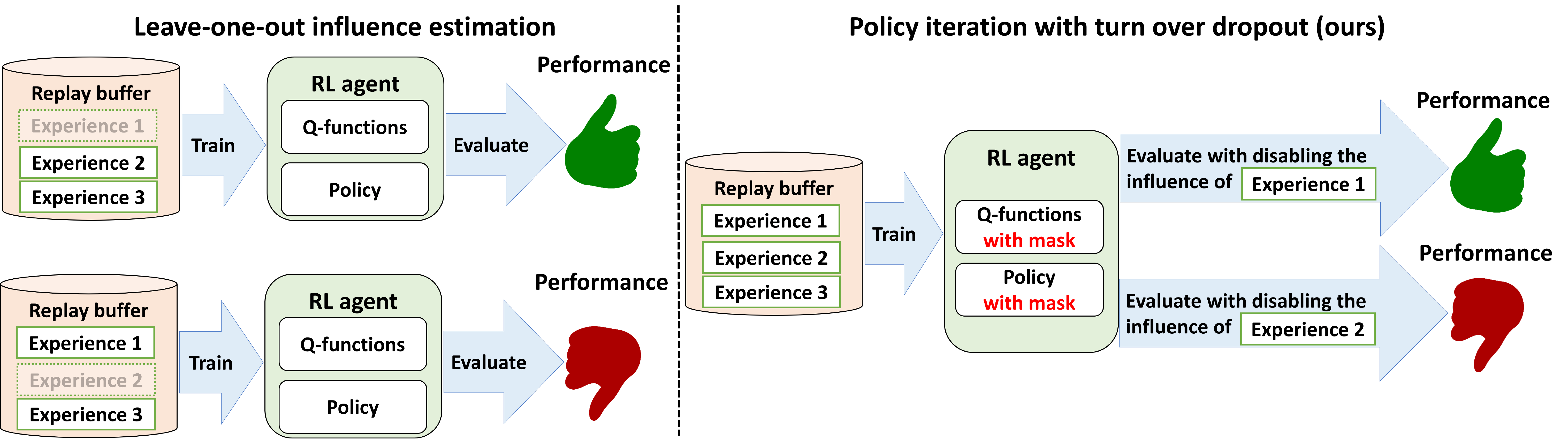}
\end{center}
\vspace{-0.7\baselineskip}
\caption{Leave-one-out (LOO) influence estimation method (left part) and our method (right part). 
LOO estimates the influence of experiences by retraining an RL agent for each experience deletion. 
In contrast, our method estimates the influence of experiences without retraining. 
}
\label{fig:Summary}
\end{figure}

\section{Preliminaries}\label{sec:preliminaries}
In Section~\ref{sec:proposed_method}, we will introduce our PI method for estimating the influence of experiences in an RL problem. 
As preliminaries, we explain the RL problem, PI, and influence estimation. 

\textbf{Reinforcement learning (RL).} 
RL addresses the problem of an agent learning to act in an environment. 
The environment provides the agent with a state $s$. The agent responds by selecting an action $a$, and then the environment provides a reward $r$ and the next state $s'$. 
This interaction between the agent and environment continues until the agent reaches a terminal state. 
%
The agent aims to find a policy $\pi: \mathcal{S} \times \mathcal{A} \rightarrow [0, 1]$ that maximizes the cumulative reward (return). 
A Q-function $Q: \mathcal{S} \times \mathcal{A} \rightarrow \mathbb{R}$ is used to estimate the expected return. 

\textbf{Policy iteration (PI).} 
PI is a method for solving RL problems. 
It updates the policy and Q-function by iteratively performing policy evaluation and improvement. 
Various implementations of policy evaluation and improvement have been proposed (e.g., \citet{lillicrap2015continuous,fujimoto2018addressing,haarnoja2018softa}). 
In this paper, we adopt the policy evaluation and improvement used in Deep Deterministic Policy Gradient (DDPG)~\citep{lillicrap2015continuous}\footnote{
We adopt the policy evaluation and improvement used in DDPG to keep our explanation concise. Note that in our actual experiments (e.g., Section~\ref{sec:experiments}), we use an implementation based on the policy evaluation and improvement used in Soft Actor-Critic (SAC)~\citep{haarnoja2018softa}, which extends the DDPG formulation with entropy regularization and clipped double Q-functions.}. 
In policy evaluation in DDPG, the Q-function $Q_{\phi}: \mathcal{S} \times \mathcal{A} \rightarrow \mathbb{R}$, parameterized by $\phi$, is updated as: 
\begin{dmath}
    \phi \leftarrow \phi - \nabla_{\phi} \mathbb{E}_{(s, a, r, s') \sim \mathcal{B}, ~a' \sim \pi_\theta(\cdot | s' )} \left[ \left( r + \gamma Q_{\bar{\phi}}(s', a') - Q_{\phi}(s, a) \right)^2 \right], \label{eq:Obj_Qfuncs} 
\end{dmath}
where $\mathcal{B}$ is a replay buffer containing the collected experiences, and $Q_{\bar{\phi}}$ is a target Q-function. 
In policy improvement in DDPG, policy $\pi_\theta$, parameterized by $\theta$, is updated as: 
\begin{dmath}
    \theta \leftarrow \theta + \nabla_{\theta} \mathbb{E}_{s \sim \mathcal{B}, ~a_\theta \sim \pi_\theta(\cdot | s)} \left[ Q_{\phi}(s, a_\theta) \right]. \label{eq:Obj_Policy} 
\end{dmath}

\textbf{Estimating the influence of experiences.} 
Given the policy and Q-functions updated through PI, we aim to estimate the influence of experiences on performance. 
Formally, letting $e_i$ be the $i$-th experience contained in the replay buffer $\mathcal{B}$, we evaluate the influence of $e_i$ as 
\begin{dmath}
    L\left(Q_{\phi, \mathcal{B}\backslash{\{e_i\}}}, \pi_{\theta, \mathcal{B}\backslash{\{e_i\}}} \right) - L\left( Q_{\phi, \mathcal{B}}, \pi_{\theta, \mathcal{B}} \right), \label{eq:influence}
\end{dmath}
where $L$ is a metric for evaluating the performance of the Q-function and policy, 
$Q_{\phi, \mathcal{B}}$ and $\pi_{\theta, \mathcal{B}}$ are the Q-function and policy updated with all experiences contained in $\mathcal{B}$, 
and $Q_{\phi, \mathcal{B}\backslash{\{e_i\}}}$ and $\pi_{\theta, \mathcal{B}\backslash{\{e_i\}}}$ are those updated with $\mathcal{B}$ excluding $e_i$. 
%
$L$ is defined according to the focus of our experiments. 
In this paper, we define $L$ as policy and Q-function loss for the experiments in Section~\ref{sec:experiments}, and as empirical return and Q-estimation bias for the applications in Section~\ref{sec:application}.

\section{Leave-one-out (LOO) influence estimation}\label{sec:problem_of_loo}
What method can be used to estimate the influence of experiences? 
%
One straightforward method is based on the LOO algorithm (Algorithm~\ref{alg1:Loo}). 
This algorithm estimates the influence of experiences by retraining the RL agent's components (i.e., policy and Q-functions) for each experience deletion. 
Specifically, it retrains the policy $\pi_{\theta'}$ and Q-function $Q_{\phi'}$ using $\mathcal{B} \backslash \{ e_i \}$ through $I$ policy iterations (lines 4--6). 
Here, $I$ equals the number of policy iterations required for training the original policy $\pi_{\theta}$ and Q-function $Q_{\phi}$. 
After retraining the components, the influence of $e_i$ is evaluated using Eq.~\ref{eq:influence_LOO} with $\pi_{\theta'}$, $Q_{\phi'}$ and $\pi_{\theta}$, $Q_{\phi}$ (line 7). 
\begin{algorithm*}[t!]
\caption{Leave-one-out influence estimation for policy iteration}
\label{alg1:Loo}
\begin{algorithmic}[1]
\STATE \textbf{given} replay buffer $\mathcal{B}$, learned parameters $\phi, \theta$, and number of policy iteration $I$. 
\FOR{$e_{i} \in \mathcal{B}$} 
    \STATE Initialize temporal parameters $\phi'$ and $\theta'$. 
    \FOR{$I$ iterations}
        \STATE Update $Q_{{\phi}'}$ with $\mathcal{B} \backslash \{ e_i \}$ (policy evaluation). 
        \STATE Update $\pi_{\theta'}$ with $\mathcal{B} \backslash \{ e_i \}$ (policy improvement). 
    \ENDFOR
    \STATE Evaluate the influence of $e_{i}$ as 
    \begin{dmath}
        L\left(Q_{\phi'}, \pi_{\theta'} \right)  - L\left( Q_{\phi}, \pi_{\theta} \right). \label{eq:influence_LOO} \nonumber
    \end{dmath}
\ENDFOR
\end{algorithmic}
\end{algorithm*}

However, in typical settings, Algorithm~\ref{alg1:Loo} becomes computationally prohibitive due to retraining. 
In typical settings (e.g., \citet{fujimoto2018addressing,haarnoja2018soft}), the size of the buffer $\mathcal{B}$ is small at the beginning of policy iteration and increases by one with each iteration. 
Consequently, the size of $\mathcal{B}$ is approximately equal to the number of iterations $I$ (i.e., $|\mathcal{B}| \approx I$). 
Since Algorithm~\ref{alg1:Loo} retrains the RL agent’s components through $I$ policy iterations for each $e_i$, the total number of policy iterations across the entire algorithm becomes $I^2$. 
The value of $I$ typically ranges between $10^3$ and $10^6$ (e.g., \citet{chen2021randomized,haarnoja2018soft}), which makes it difficult to complete all policy iterations in a realistic timeframe. 

In the next section, we will introduce a method to estimate the influence of experiences without retraining the RL agent's components. 

\section{Policy iteration with turn-over dropout (PIToD)}\label{sec:proposed_method}
In this section, we present \textbf{P}olicy \textbf{I}teration with \textbf{T}urn-\textbf{o}ver \textbf{D}ropout (PIToD), which estimates the influence of experiences without retraining. 
%
The concept of PIToD is shown in Figure~\ref{fig:Outline}, and an algorithmic description of PIToD is shown in Algorithm~\ref{alg2:PI_PIToD}. 
Inspired by ToD~\citep{kobayashi2020efficient}, PIToD uses masks and flipped masks to drop out the parameters of the policy and Q-function. 
Further details are provided in the following paragraphs. 
\begin{figure*}[t!]
\begin{center}
\includegraphics[clip, width=0.99\hsize]{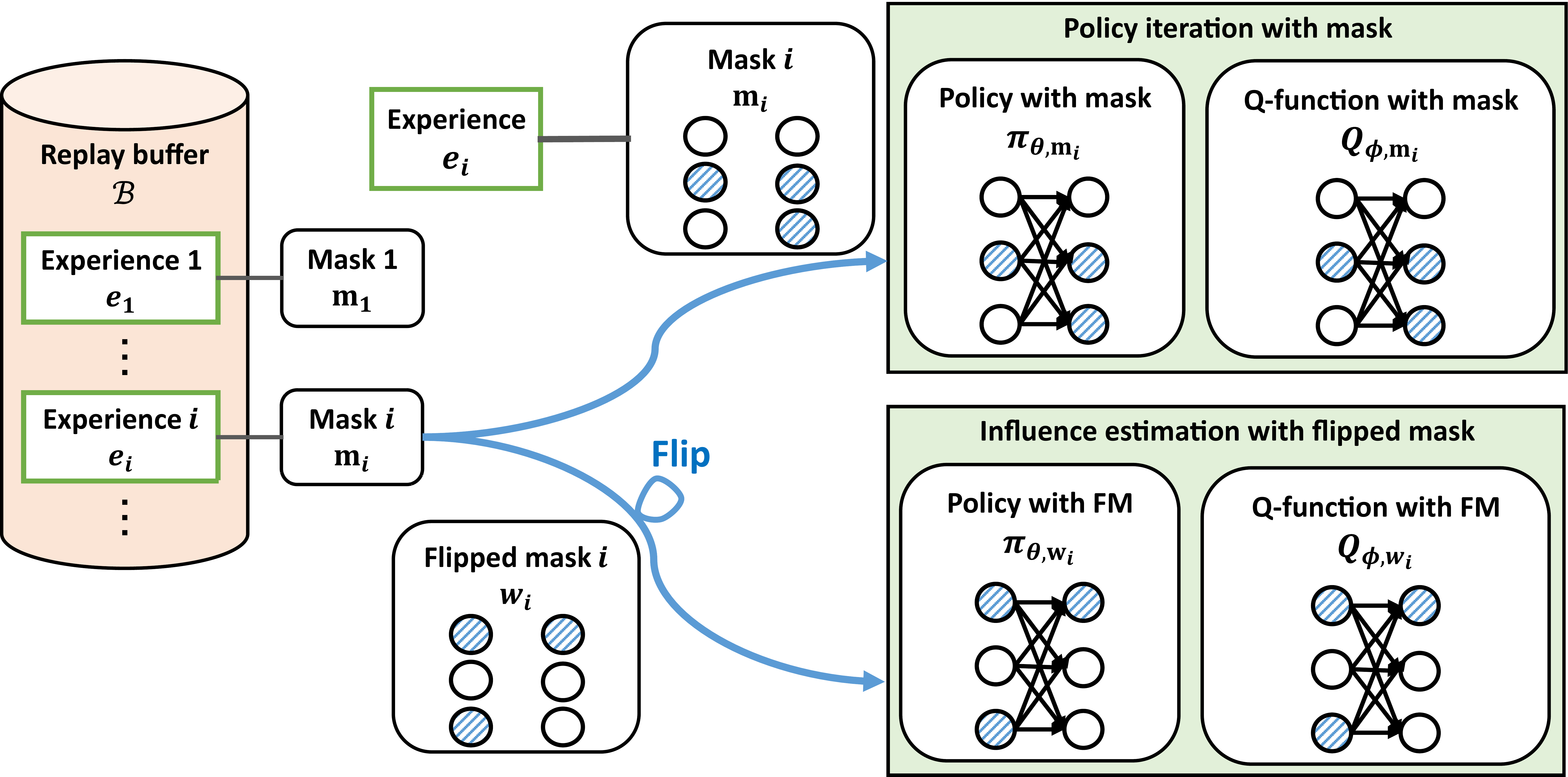}
\end{center}
\vspace{-0.7\baselineskip}
\caption{
The concept of PIToD. 
PIToD uses mask $\mathbf{m}_{i}$ and flipped mask $\mathbf{w}_i$. 
It applies $\mathbf{m}_{i}$ to the policy and Q-function for PI with $e_i$. 
Additionally, it applies $\mathbf{w}_{i}$ to the policy and Q-function for estimating the influence of $e_i$. 
}
\label{fig:Outline}
\end{figure*}

\textbf{Masks and flipped masks.} 
PIToD uses mask $\mathbf{m}_i$ and flipped mask $\mathbf{w}_i$, which are binary vectors uniquely associated with experience $e_i$. 
The mask $\mathbf{m}_i$ consists of elements randomly initialized to $0$ or $1$. 
$\mathbf{m}_i$ is used to drop out the parameters of the policy and Q-function during PI with $e_i$. 
Additionally, the flipped mask $\mathbf{w}_i$ is the negation of $\mathbf{m}_i$, i.e., $\mathbf{w}_i = \mathbf{1} - \mathbf{m}_i$. 
$\mathbf{w}_i$ is used to drop out the parameters of the policy and Q-function for estimating the influence of $e_i$. 

\textbf{Policy iteration with the mask (lines 5--6 in Algorithm~\ref{alg2:PI_PIToD}).} 
PIToD applies $\mathbf{m}_i$ to the policy and Q-function during PI with $e_i$. 
It executes PI with variants of policy evaluation (Eq.~\ref{eq:Obj_Qfuncs}) and improvement (Eq.~\ref{eq:Obj_Policy}) where masks are applied to the parameters of the policy and Q-function. 
The policy evaluation for PIToD is 
\begin{dmath}
    \phi \leftarrow \phi - \nabla_{\phi} \mathbb{E}_{e_i = (s, a, r, s', i) \sim \mathcal{B},~ a' \sim \pi_{\theta, \mathbf{m}_i}(\cdot | s' )} \left[ \left( r + \gamma Q_{\bar{\phi}, \mathbf{m}_i}(s', a') - Q_{\phi, \mathbf{m}_i}(s, a) \right)^2 \right]. \label{eq:Obj_Qfuncs_ToD} 
\end{dmath}
The policy improvement for PIToD is 
\begin{dmath}
    \theta \leftarrow \theta + \nabla_{\theta} \mathbb{E}_{e_i = (s, i) \sim \mathcal{B}, ~ a_{\theta, \mathbf{m}_i} \sim \pi_{\theta, \mathbf{m}_i}(\cdot | s)} \left[ Q_{\phi, \mathbf{m}_i}(s, a_{\theta, \mathbf{m}_i}) \right]. \label{eq:Obj_Policy_ToD} 
\end{dmath}
Here, $Q_{\phi, \mathbf{m}_i}$ and $\pi_{\theta, \mathbf{m}_i}$ are the Q-function and policy to which the mask $\mathbf{m}_i$ is applied. 
In Eq.~\ref{eq:Obj_Qfuncs_ToD} and Eq.~\ref{eq:Obj_Policy_ToD}, for inputs from $e_i$, $Q_{\phi, \mathbf{m}_i}$ and $\pi_{\theta, \mathbf{m}_i}$ compute their outputs without using the parameters that are dropped out by $\mathbf{m}_i$. 
Thus, the parameters dropped out by $\mathbf{m}_i$ (i.e., the parameters obtained by applying $\mathbf{w}_i$) are expected not to be influenced by $e_i$. 
More theoretically, if $Q_{\phi, \mathbf{m}_i}$ and $\pi_{\theta, \mathbf{m}_i}$ are dominantly influenced by $e_i$, the parameters obtained by $\mathbf{w}_i$ are provably not influenced by $e_i$ (see Appendix~\ref{app:theory} for details). 
Based on this theoretical property, we estimate the influence of $e_i$ by applying $\mathbf{w}_i$ to policy and Q-functions (see the next paragraph for details). 
\begin{algorithm*}[t!]
\caption{Policy iteration with turn-over dropout (PIToD)}
\label{alg2:PI_PIToD}
\begin{algorithmic}[1]
\STATE Initialize policy parameters $\theta$, Q-function parameters $\phi$, and an empty replay buffer $\mathcal{B}$; Set influence estimation interval $I_{\text{ie}}$. 
\FOR{$i' = 0, ..., I$ iterations} 
    \STATE Take action $a \sim \pi_\theta(\cdot | s)$; Observe reward $r$ and next state $s'$. Define an experience using $i'$ as: $e_{i'} = (s, a, r, s', i')$; $\mathcal{B} \leftarrow \mathcal{B} \bigcup \{ e_{i'}\}$. 
    \STATE Sample experiences $\{ (s, a, r, s', i), ... \}$ from $\mathcal{B}$ (Here, $e_i = (s, a, r, s', i)$). 
    \STATE Update $\phi$ with gradient descent using
    \vspace{-0.7\baselineskip}\begin{equation}
        \nabla_\phi \sum_{(s, a, r, s', i)} \left( r + \gamma Q_{\bar{\phi}, \mathbf{m}_i}(s', a') - Q_{\phi, \mathbf{m}_i}(s, a) \right)^2, ~~ a' \sim \pi_{\theta, \mathbf{m}_i}(\cdot | s' ). \nonumber
    \end{equation}\vspace{-0.7\baselineskip}
    \STATE Update $\theta$ with gradient ascent using 
    \begin{equation}
        \nabla_\theta \sum_{(s, i)} Q_{\phi, \mathbf{m}_i}(s, a_{\theta, \mathbf{m}_i}), ~~~ a_{\theta, \mathbf{m}_i} \sim \pi_{\theta, \mathbf{m}_i}(\cdot | s). \nonumber 
    \end{equation}\vspace{-0.7\baselineskip}
    \IF{$ i' \% I_{\text{ie}} = 0 $}
        \STATE For $e_i \in \mathcal{B}$, estimate the influence of $e_i$ using  
        \begin{dmath*}
             L\left(Q_{\phi, \mathbf{w}_i}, \pi_{\theta, \mathbf{w}_i} \right) - L\left(Q_{\phi}, \pi_{\theta} \right) ~~~\text{or}~~~  L\left(Q_{\phi, \mathbf{w}_i}, \pi_{\theta, \mathbf{w}_i} \right) - L\left(Q_{\phi, \mathbf{m}_i}, \pi_{\theta, \mathbf{m}_i}\right). \nonumber
        \end{dmath*}
    \ENDIF
\ENDFOR
\end{algorithmic}
\end{algorithm*}

\textbf{Estimating the influence of an experience with the flipped mask (lines 7--8 in Algorithm~\ref{alg2:PI_PIToD}).} 
PIToD periodically estimates the influence of $e_i$ by applying $\mathbf{w}_i$ to the policy and Q-function. 
Specifically, PIToD estimates the influence of $e_i$ (Eq.~\ref{eq:influence}) as 
\begin{dmath}
    L\left(Q_{\phi, \mathbf{w}_i}, \pi_{\theta, \mathbf{w}_i} \right) - L\left(Q_{\phi}, \pi_{\theta} \right), 
    \label{eq:approxed_influence} 
\end{dmath}
where the first term is the performance when $e_i$ is deleted, and the second term is the performance with all experiences. 
$Q_{\phi, \mathbf{w}_i}$ and $\pi_{\theta, \mathbf{w}_i}$ are the Q-function and policy with dropout based on $\mathbf{w}_i$. 
In contrast, $Q_{\phi}$ and $\pi_{\theta}$ are the Q-function and policy without dropout. 
For the second term, if we want to highlight the influence of $e_i$ more significantly, the term can be evaluated by using the masked policy and Q-functions: $L\left(Q_{\phi, \mathbf{m}_i}, \pi_{\theta, \mathbf{m}_i}\right)$. 
The influence is estimated every $I_{\text{ie}}$ iterations (line 7 in Algorithm~\ref{alg2:PI_PIToD}). 
These influence estimations by PIToD do not require retraining for each experience deletion, unlike the LOO method. 

\textbf{Implementation details for PIToD.} 
For the experiments in Sections~\ref{sec:experiments} and \ref{sec:application}, each mask element is initialized by drawing from a discrete uniform distribution over $\{0, 1\}$ to minimize overlap between the masks (see Appendix~\ref{sec:analyzing_overlap_in_mask} for details). 
Additionally, we implemented PIToD based on Soft Actor-Critic~\citep{haarnoja2018soft} for these experiments (see Appendix~\ref{sec:practical_implementation} for details). 
We also introduce a practical implementation design for applying the masks to the policy and Q-function (see Appendix \ref{sec:practical_implementation} for details). 

\section{Evaluations for PIToD}\label{sec:experiments}
In the previous section, we introduced PIToD, a method that efficiently estimates the influence of experiences. 
In this section, we evaluate how it correctly estimates the influence  (Section~\ref{sec:eval_self_influence}) and its computational efficiency (Section~\ref{sec:eval_computational_time}). 

\subsection{How correctly does PIToD estimate the influence of experiences? Evaluations with self-influence}\label{sec:eval_self_influence}
In this section, we evaluate how correctly PIToD estimates the influence of experiences by focusing on their self-influence. 
Self-influence~\citep{kobayashi2020efficient,thakkar-etal-2023-self,bejan-etal-2023-make} is the influence of an experience on prediction performance using that same experience. 
We define self-influences on policy evaluation and on policy improvement. 
The self-influence of an experience $e_i := (s, a, r, s', i)$ on policy evaluation is 
\begin{eqnarray}
    && L_{\text{pe}, i}(Q_{\phi, \mathbf{w}_i}) -  L_{\text{pe}, i}(Q_{\phi, \mathbf{m}_i}). \label{eq:self_infl_q_func} 
\end{eqnarray}
Here, $L_{\text{pe}, i}$ represents the temporal difference error based on $e_i$, and it is defined as 
\begin{eqnarray}
    && L_{\text{pe}, i}(Q) = \left( r + \gamma  Q_{\bar{\phi}, \mathbf{m}_i}(s', a') - Q(s, a) \right)^2, ~~ a' \sim \pi_{\theta, \mathbf{m}_i}(\cdot | s'). \nonumber
\end{eqnarray}
The value of Eq.~\ref{eq:self_infl_q_func} is expected to be positive. 
$Q_{\phi, \mathbf{m}_i}$ is optimized by PIToD to minimize $L_{\text{pe}, i}$ (cf. line 5 in Algorithm~\ref{alg2:PI_PIToD}), while $Q_{\phi, \mathbf{w}_i}$ is not. 
Therefore, $L_{\text{pe}, i}(Q_{\phi, \mathbf{m}_i}) \leq L_{\text{pe}, i}(Q_{\phi, \mathbf{w}_i})$, implying that  
\begin{eqnarray}
    && \underbrace{L_{\text{pe}, i}(Q_{\phi, \mathbf{w}_i}) -  L_{\text{pe}, i}(Q_{\phi, \mathbf{m}_i})}_{\text{Eq.~\ref{eq:self_infl_q_func} }} \geq 0. \label{eq:expected_val_self_infl_q_func}
\end{eqnarray}
The self-influence of $e_i$ on policy improvement is 
\begin{eqnarray}
    && L_{\text{pi}, i}(\pi_{\theta, \mathbf{w}_i}) -  L_{\text{pi}, i}(\pi_{\theta, \mathbf{m}_i}). \label{eq:self_infl_policy_func}
\end{eqnarray}
Here, $L_{\text{pi}, i}$ represents the Q-value estimate based on $e_i$, and it is defined as 
\begin{eqnarray}
    && L_{\text{pi}, i}(\pi) =  Q_{\phi, \mathbf{m}_i}(s, a'), ~~~~ a' \sim \pi(\cdot | s). \nonumber
\end{eqnarray}
The value of Eq.~\ref{eq:self_infl_policy_func} is expected to be negative. 
$\pi_{\theta, \mathbf{m}_i}$ is optimized by PIToD to maximize $L_{\text{pi}, i}$ (cf. line 6 in Algorithm~\ref{alg2:PI_PIToD}), while $\pi_{\theta, \mathbf{w}_i}$ is not. 
Therefore, $L_{\text{pi}, i}(\pi_{\theta, \mathbf{m}_i}) \geq L_{\text{pi}, i}(\pi_{\theta, \mathbf{w}_i})$, which implies that 
\begin{eqnarray}
    && \underbrace{L_{\text{pi}, i}(\pi_{\theta, \mathbf{w}_i}) -  L_{\text{pi}, i}(\pi_{\theta, \mathbf{m}_i})}_{\text{Eq.~\ref{eq:self_infl_policy_func} }} \leq 0. \label{eq:expected_val_self_infl_policy_func}
\end{eqnarray}

We evaluate whether PIToD has correctly estimated the influence of experiences based on whether Eq.~\ref{eq:expected_val_self_infl_q_func} and Eq.~\ref{eq:expected_val_self_infl_policy_func} are satisfied~\footnote{Here, ``correct'' refers to whether the estimated influence has the theoretically expected sign, not to its exact magnitude. This means that estimates satisfying the inequalities are considered non-contradictory, rather than guaranteed to be ground-truth correct.}. 
We periodically evaluate the ratio of experiences for which PIToD has correctly estimated self-influence in the MuJoCo environments~\citep{todorov2012mujoco}. 
The MuJoCo tasks for this evaluation are Hopper, Walker2d, Ant, and Humanoid. 
In this evaluation, 5000 policy iterations (i.e., lines 3--6 of Algorithm~\ref{alg2:PI_PIToD}) constitute one epoch, with 125 epochs allocated for Hopper and 300 epochs for the others. 
At each epoch, we perform the following steps:
(i) for each experience in the replay buffer, we check whether Eq.~\ref{eq:self_infl_q_func} and Eq.~\ref{eq:self_infl_policy_func} satisfy Eq.~\ref{eq:expected_val_self_infl_q_func} and Eq.~\ref{eq:expected_val_self_infl_policy_func}, respectively; and 
(ii) we record the ratio of experiences that satisfy these equations. 

Evaluation results (Figure~\ref{fig:raio_of_experiences_with_pos_neg_influence}) show that the ratio of experiences whose self-influence (Eq.~\ref{eq:self_infl_q_func} and Eq.~\ref{eq:self_infl_policy_func}) is correctly estimated exceeds the chance rate of 0.5. 
For self-influence on policy evaluation (Eq.~\ref{eq:self_infl_q_func}), the ratio of correctly estimated experiences (i.e., those satisfying Eq.~\ref{eq:expected_val_self_infl_q_func}) is higher than 0.9 across all environments. 
Furthermore, for self-influence on policy improvement (Eq.~\ref{eq:self_infl_policy_func}), the ratio of correctly estimated experiences (i.e., those satisfying Eq.~\ref{eq:expected_val_self_infl_policy_func}) exceeds 0.7 in Hopper, 0.8 in Walker2d and Ant, and 0.9 in Humanoid. 
These results suggest that PIToD estimates the influence of experiences more correctly than random estimation. 

\textbf{Supplementary analysis.} 
How are experiences that exhibit significant self-influence distributed? 
Figure~\ref{fig:distribution_of_self_influences} shows the distribution of self-influence across experiences. 
From the figure, we see that in policy evaluation, the self-influence of older experiences (with smaller normalized experience indices) becomes more significant as the epoch progresses. 
This tendency can be seen as a primacy bias~\citep{pmlr-v162-nikishin22a}, suggesting that the RL agent is overfitting more significantly to the experiences of the early stages of learning. 
Conversely, for policy improvement, we do not observe a clear tendency for primacy bias. 
These observations may indicate that policy improvement is less prone to causing overfitting to older experiences than policy evaluation. 
\begin{figure*}[t!]
\begin{minipage}{1.0\hsize}
\includegraphics[clip, width=0.49\hsize]{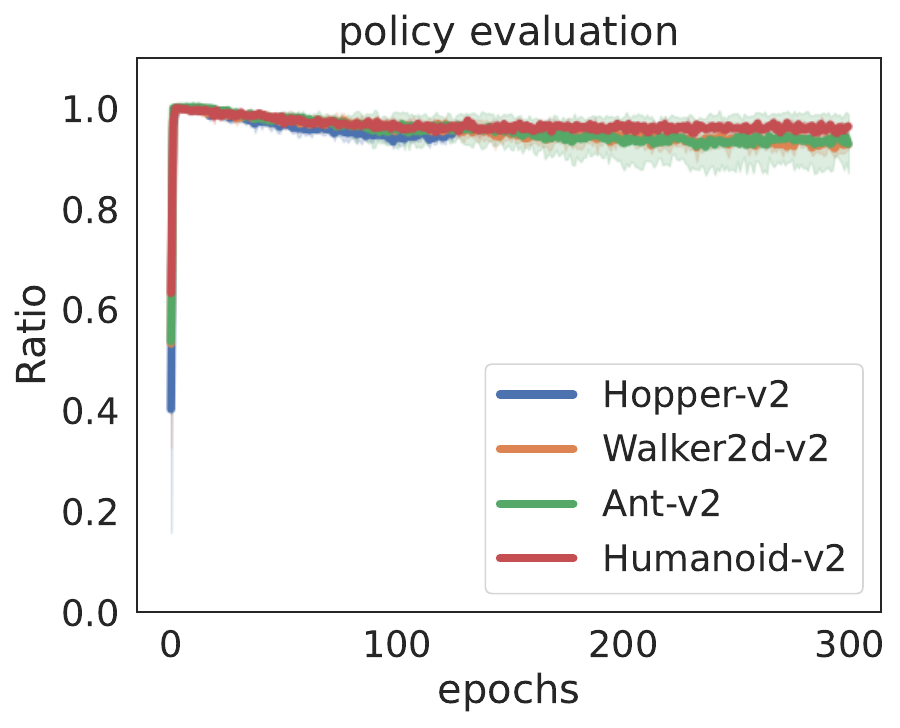}
\includegraphics[clip, width=0.49\hsize]{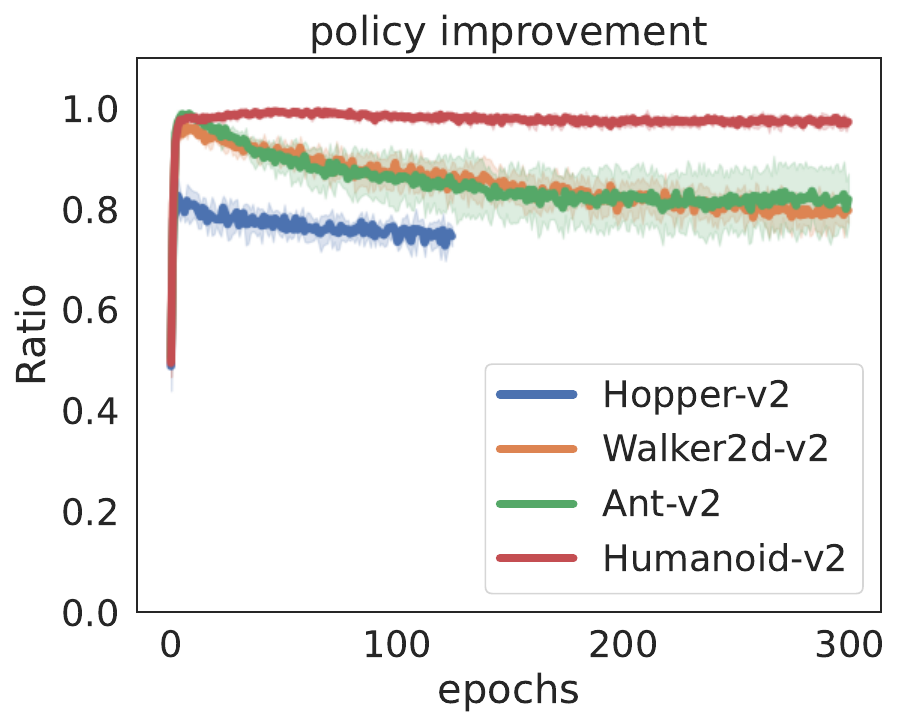}
\end{minipage}
\vspace{-0.7\baselineskip}
\caption{
The ratio of experiences for which PIToD correctly estimated self-influence. 
The left-hand figure displays this ratio in policy evaluation cases, where a self-influence value is expected to be positive (i.e., Eq.~\ref{eq:self_infl_q_func} $\geq 0$). 
The right-hand figure displays the ratio in policy improvement cases, where a self-influence value  is expected to be negative (i.e., Eq.~\ref{eq:self_infl_policy_func} $\leq 0$). 
In both figures, the vertical axis represents the ratio of correctly estimated experiences, and the horizontal axis shows the number of epochs. 
Each line represents the mean of ten trials, and the shaded region represents one standard deviation around the mean. 
The figure shows that the ratio of correctly estimated experiences surpasses the chance rate of 0.5. 
Note that ``correctly estimated'' here means that the estimated self-influence has the expected sign (positive or negative), based on Eq.~\ref{eq:self_infl_q_func} and Eq.~\ref{eq:self_infl_policy_func}.  
It does not imply exact correctness with respect to the true influence magnitude. 
}
\label{fig:raio_of_experiences_with_pos_neg_influence}
\end{figure*}
\begin{figure*}[t!]
\begin{minipage}{1.0\hsize}
\includegraphics[clip, width=0.245\hsize]{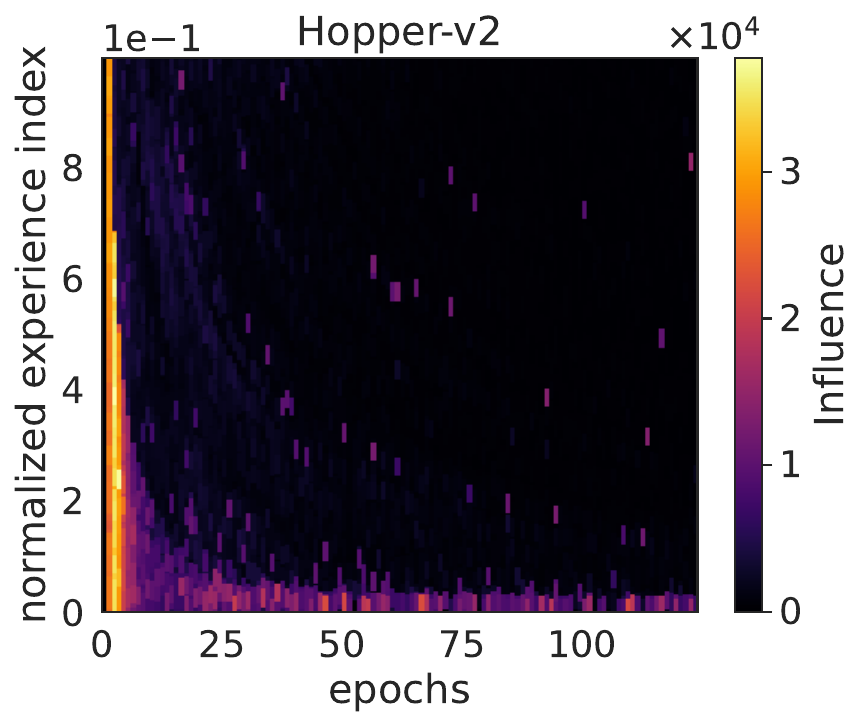}
\includegraphics[clip, width=0.245\hsize]{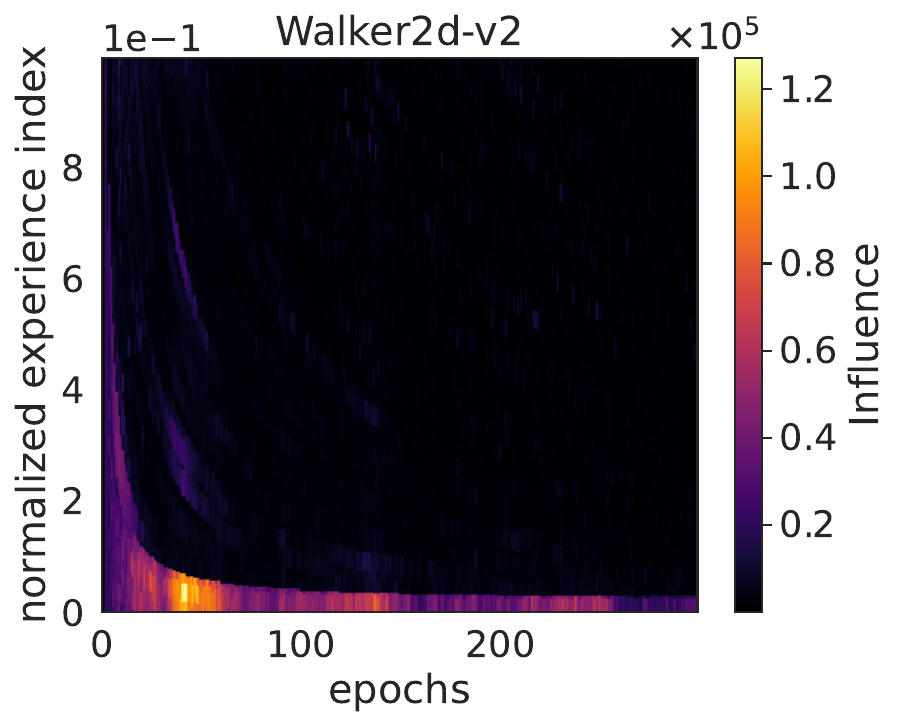}
\includegraphics[clip, width=0.245\hsize]{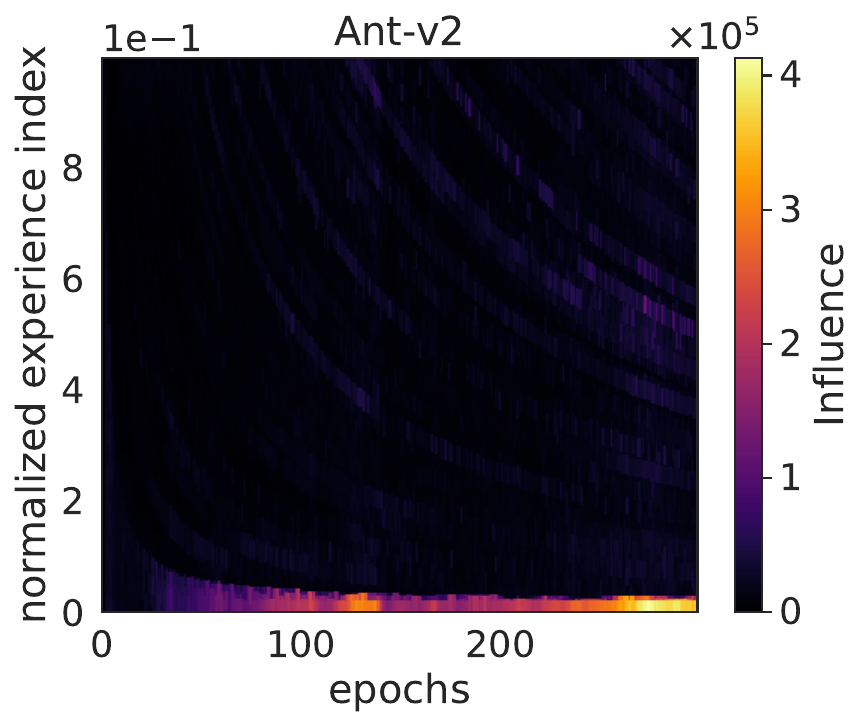}
\includegraphics[clip, width=0.245\hsize]{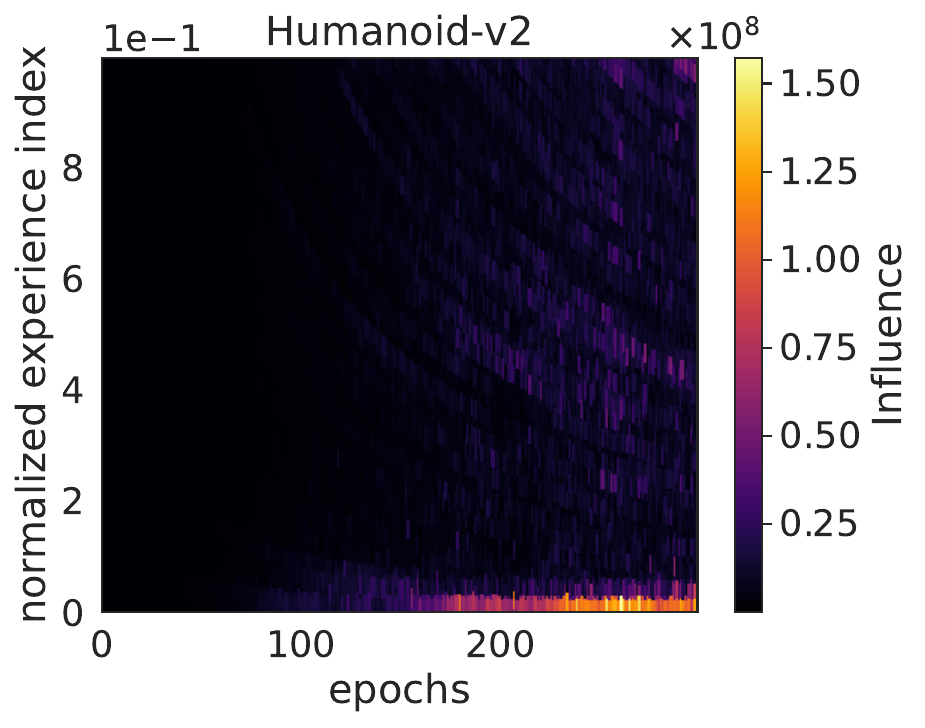}
\subcaption{Distribution of self-influence on policy evaluation (Eq.~\ref{eq:self_infl_q_func}).}
\end{minipage}
\begin{minipage}{1.0\hsize}
\includegraphics[clip, width=0.245\hsize]{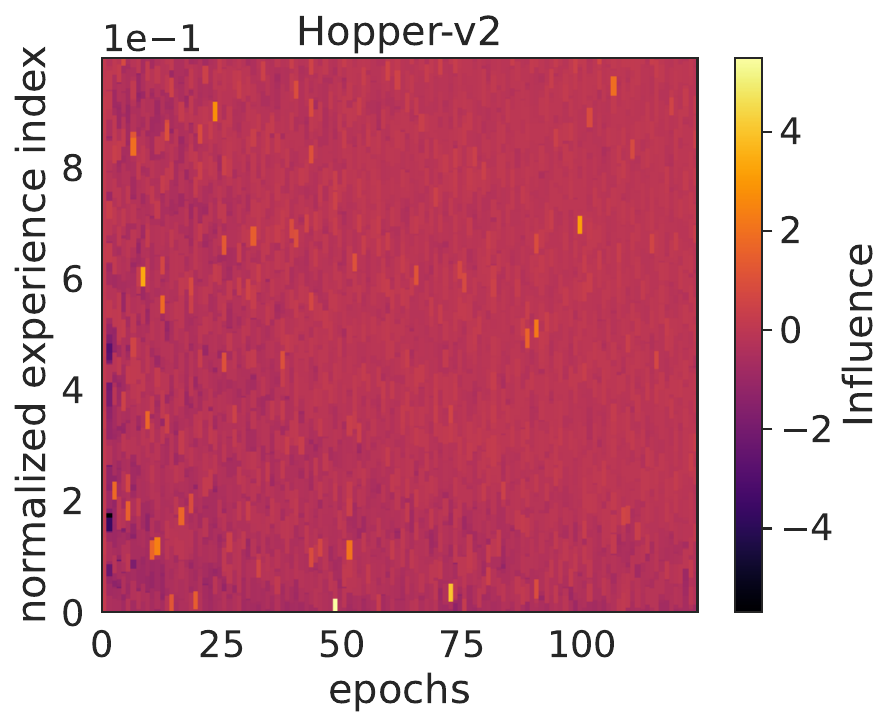}
\includegraphics[clip, width=0.245\hsize]{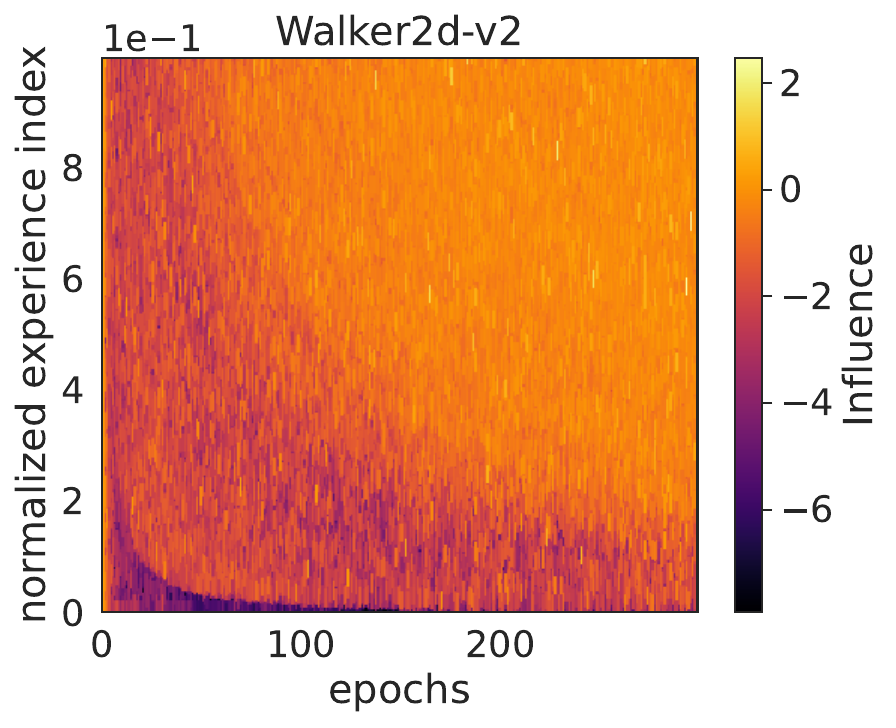}
\includegraphics[clip, width=0.245\hsize]{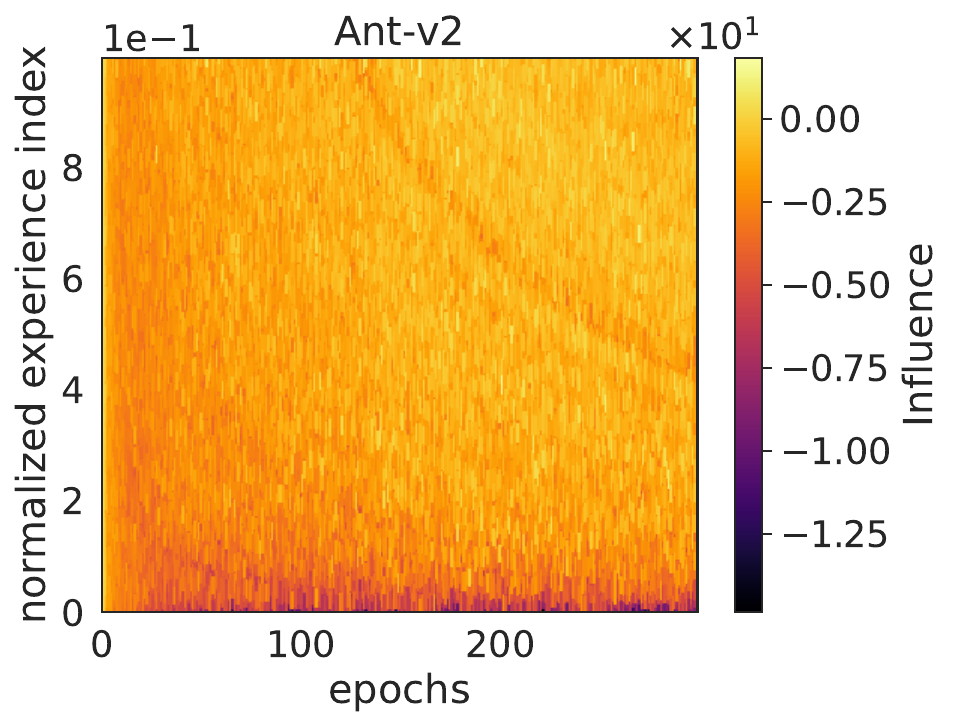}
\includegraphics[clip, width=0.245\hsize]{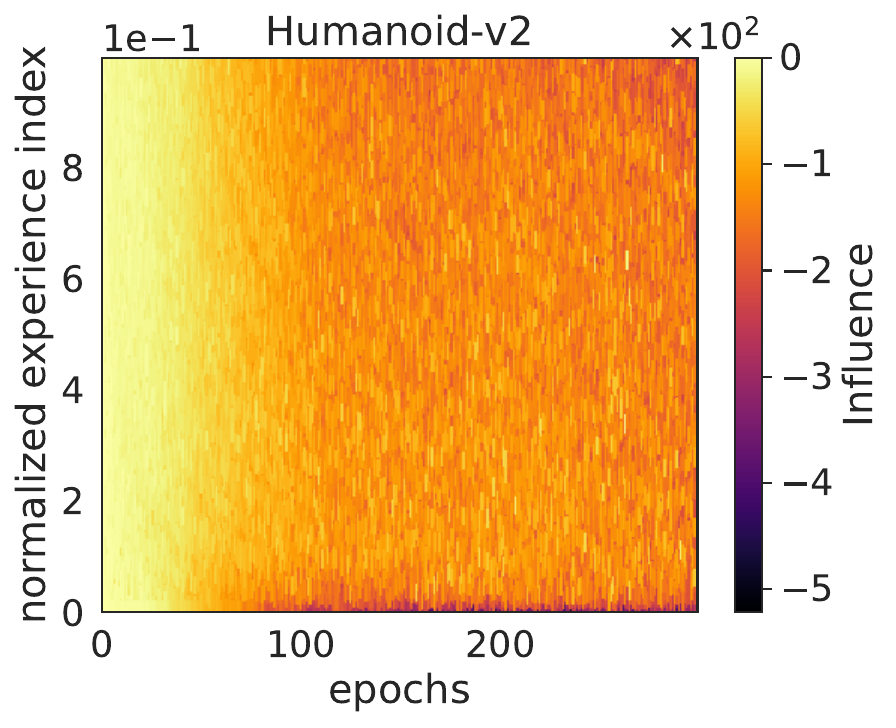}
\subcaption{Distribution of self-influence on policy improvement (Eq.~\ref{eq:self_infl_policy_func}).}
\end{minipage}
\vspace{-0.7\baselineskip}
\caption{
Distribution of self-influence on policy evaluation and policy improvement. 
The horizontal axis represents the number of epochs. 
The vertical axis represents the normalized experience index, which corresponds to the relative age of experiences stored in the replay buffer.
Specifically, at each epoch, the oldest experience is mapped to $0.0$ and the most recent to $1.0$.
As a result, the vertical axis is fully populated at all epochs, even though the number of stored experiences may vary across epochs. 
The color bar represents the value of self-influence. 
\textbf{Interpretation of this figure:} 
For example, if the value of self-influence for $e_i$ in policy evaluation cases is $2 \cdot 10^8$, this indicates that the value of $L_{\text{pe}, i}(Q_{\phi, \mathbf{w}_i})$ is $2 \cdot 10^8$ larger than that of $L_{\text{pe}, i}(Q_{\phi, \mathbf{m}_i})$. 
\textbf{Key insight:} In policy evaluation, experiences with high self-influence tend to concentrate on older ones (with smaller normalized experience indices) as the epochs progress. 
}
\label{fig:distribution_of_self_influences}
\end{figure*}

\subsection{How efficiently does PIToD estimate the influence of experiences? Computational time evaluation}\label{sec:eval_computational_time} 
We evaluate the computational time required for influence estimation with PIToD and compare it to the estimated time for LOO. 
To measure the computational time for PIToD, we run the method under the same settings as in the previous section and record its wall-clock time. 
For comparison, we also evaluate the estimated time required for influence estimation using LOO (Section~\ref{sec:problem_of_loo}). 
To estimate the time for LOO, we record the average time required for one policy iteration with PIToD and multiply this by the total number of policy iterations required for LOO~\footnote{The total number of policy iterations for LOO is $I^2$, as discussed in Section~\ref{sec:problem_of_loo}. 
However, in the practical implementation of PIToD used in our experiments, we divide the experiences in the buffer into groups of 5000 experiences and estimate the influence of each group (Appendix~\ref{sec:practical_implementation}). 
For a fair comparison with this implementation, we use $\frac{I^2}{5000}$ instead of $I^2$ as the total number of policy iterations for LOO.}. 

The evaluation results (Figure~\ref{fig:training_time}) show that PIToD significantly reduces computational time compared to LOO. 
The time required for LOO increases quadratically as epochs progress, taking, for example, more than $4 \cdot 10^7$ seconds ($\approx 462$ days) up to 300 epochs in Humanoid. 
In contrast, the time required for PIToD increases linearly, taking about $1.4 \cdot 10^5$ seconds ($\approx$ one day) for 300 epochs in Humanoid. 
\begin{figure*}[t!]
\begin{minipage}{0.99\hsize}
\includegraphics[clip, width=0.49\hsize]{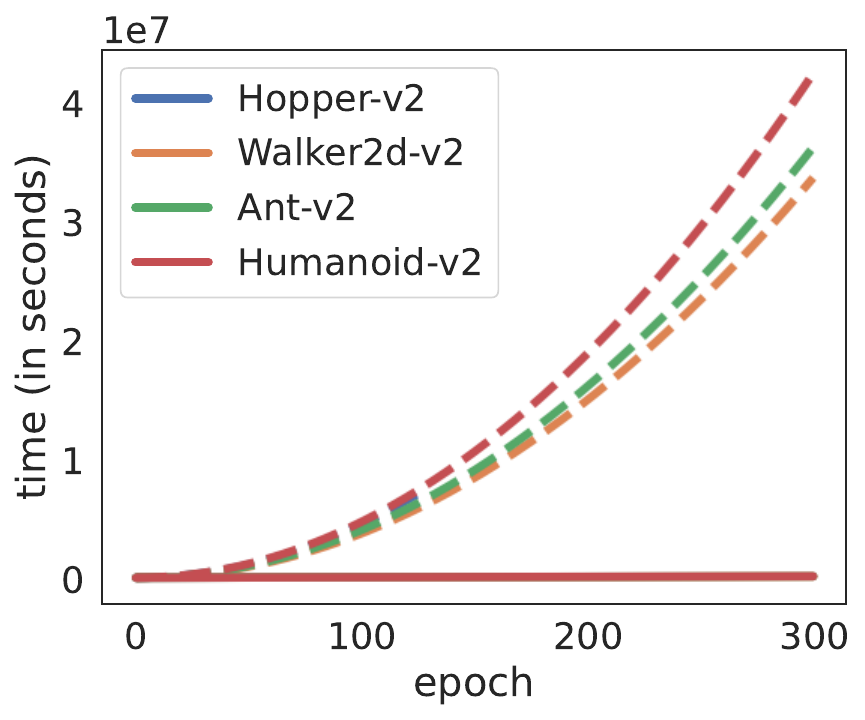}
\includegraphics[clip, width=0.49\hsize]{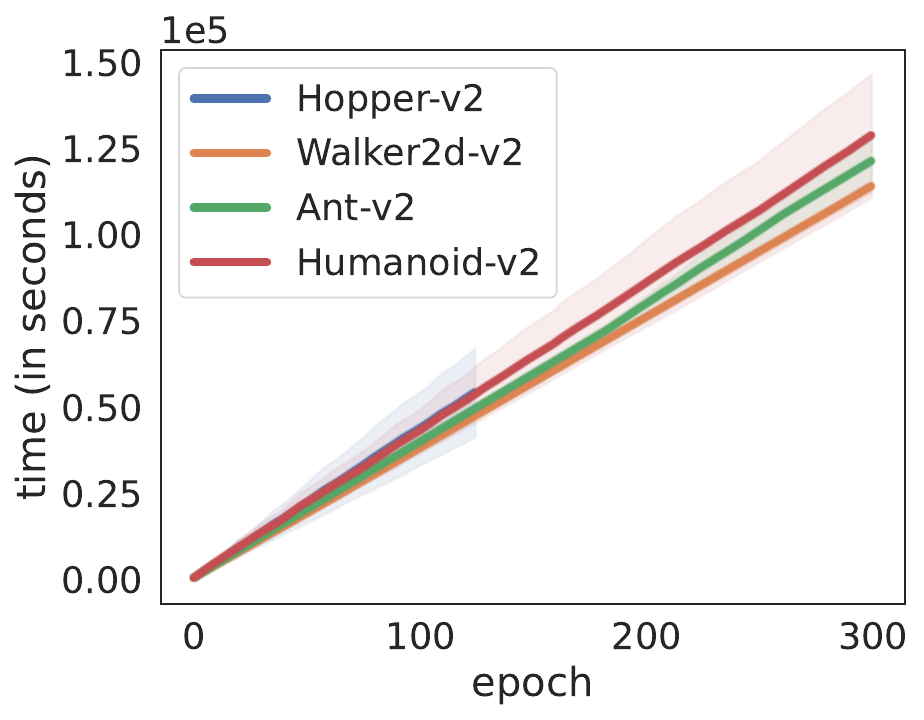}
\end{minipage}
\vspace{-0.7\baselineskip}
\caption{
Wall-clock time required for influence estimation by PIToD and LOO. 
The solid lines represent the mean time for PIToD, averaged over ten trials, 
and the dashed lines represent the estimated time for LOO.
The shaded regions indicate one standard deviation around the mean time for PIToD. 
The left figure shows the time for both PIToD and LOO, 
while the right figure shows the time for PIToD alone to more clearly illustrate its time requirements.  
The results show that the time required for LOO increases quadratically with the number of epochs, whereas the time required for PIToD increases linearly. 
}
\label{fig:training_time}
\end{figure*}

\section{Application of PIToD: amending policies and Q-functions by deleting negatively influential experiences}\label{sec:application}
In the previous section, we demonstrated that PIToD can correctly and efficiently estimate the influence of experiences. 
What scenarios might benefit from this capability? 
In this section, we demonstrate how PIToD can be used to amend underperforming policies and Q-functions. 

We amend policies and Q-functions by deleting experiences that negatively influence performance. 
We evaluate the performance of policies and Q-functions based respectively on returns and Q-estimation biases~\citep{fujimoto2018addressing,chen2021randomized}. 
The influence of an experience $e_i$ on the return, $L_{\text{ret}}$, is evaluated as follows: 
\begin{eqnarray}
    L_{\text{ret}}(\pi_{\theta, \mathbf{w}_i}) - L_{\text{ret}}(\pi_{\theta}), ~~\text{where}~ L_{\text{ret}}(\pi) = \mathbb{E}_{a_t \sim \pi(\cdot | s_t )} \left[ \sum_{t=0} \gamma^t r(s_t, a_t) \right]. \label{eq:influence_return} 
\end{eqnarray}
Here, $s_t$ is sampled from an environment. 
In our setup, $L_{\text{ret}}$ is estimated using Monte Carlo returns collected by rolling out policies $\pi_{\theta, \mathbf{w}_i}$ and $\pi_{\theta}$. 
The influence of $e_i$ on Q-estimation bias, $L_{\text{bias}}$, is evaluated as follows: 
\begin{dmath}
    L_{\text{bias}}(Q_{\phi, \mathbf{w}_i}) - L_{\text{bias}}(Q_{\phi}), \\
    ~~\text{where}~ L_{\text{bias}}(Q) = \mathbb{E}_{a_t \sim \pi_{\theta}(\cdot | s_t), a_{t'} \sim \pi_{\theta}(\cdot | s_{t'})} \left[ \sum_{t=0} \frac{ \left| Q(s_t, a_t) - \sum_{t'=t} \gamma^{t'}
 r(s_{t'}, a_{t'}) \right| }{ \left| \sum_{t'=t} \gamma^{t'} r(s_{t'}, a_{t'}) \right| } \right]. \label{eq:influence_bias} 
\end{dmath}
Here, $L_{\text{bias}}$ quantifies the discrepancy between the estimated and true Q-values using their L1 distance. 
Based on Eq.~\ref{eq:influence_return} and Eq.~\ref{eq:influence_bias}, we identify and delete the experience $e_{*}$ that has the strongest negative influence on them. 
We apply $\mathbf{w}_{*}$, which maximizes Eq.~\ref{eq:influence_return}, to the policy to delete $e_{*}$. 
Additionally, we apply $\mathbf{w}_{*}$, which minimizes Eq.~\ref{eq:influence_bias}, to the Q-function to delete $e_{*}$. 
The algorithmic description of our amendment process is presented in Algorithm~\ref{alg4:Amendment} in Appendix~\ref{app:alg_amend}. 

We evaluate the effect of the amendments on trials in which the policy and Q-function underperform. 
We run ten learning trials with the amendments (Algorithm~\ref{alg4:Amendment}) and evaluate (i) $L_{\text{ret}}(\pi_{\theta, \mathbf{w}_*})$ for the two trials in which the policy scores the lowest returns $L_{\text{ret}}(\pi_{\theta})$ and (ii) $L_{\text{bias}}(Q_{\phi, \mathbf{w}_*})$ for the two trials in which the Q-function scores the highest biases $L_{\text{bias}}(Q_{\phi})$~\footnote{We focus on the two worst-performing trials to more clearly assess whether the proposed amendments can improve underperforming cases. We use two trials, rather than just one, to allow for minimal variance estimation.}. 
The average scores of $L_{\text{ret}}(\pi_{\theta, \mathbf{w}_*})$ and $L_{\text{bias}}(Q_{\phi, \mathbf{w}_*})$ for these underperforming trials are shown in Figure~\ref{fig:cleansing_results}. 
The average scores of $L_{\text{ret}}(\pi_{\theta, \mathbf{w}_*})$ and $L_{\text{bias}}(Q_{\phi, \mathbf{w}_*})$ for all ten trials are shown in Figure~\ref{fig:cleansing_results_average_case} in Appendix~\ref{app:additional_results}. 

The results of the policy and Q-function amendments (Figure~\ref{fig:cleansing_results}) show that performance is improved through the amendments. 
From the policy amendment results (left part of Figure~\ref{fig:cleansing_results}), we see that the return ($L_{\text{ret}}$) is significantly improved in Hopper and Walker. 
For example, in Hopper, the return before the amendment (the blue dashed line) is approximately 1000, but after the amendment (the blue solid line), it exceeds 3000. 
Additionally, from the Q-function amendment results (right part of Figure~\ref{fig:cleansing_results}), we see that the Q-estimation bias ($L_{\text{bias}}$) is significantly reduced in Ant and Humanoid. 
For example, in Humanoid, the estimation bias of the Q-function before the amendment (the red dashed line) is approximately 30 during epochs 250--300, but after the amendment, it is reduced to approximately 20 (the red solid line). 
\begin{figure*}[t!]
\begin{minipage}{1.0\hsize}
\includegraphics[clip, width=0.49\hsize]{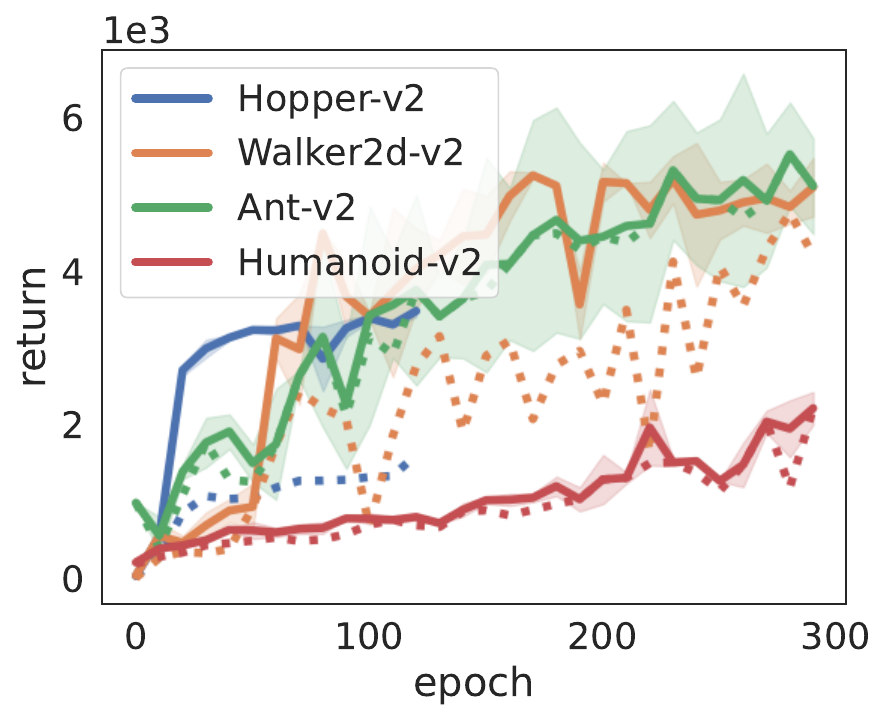}
\includegraphics[clip, width=0.49\hsize]{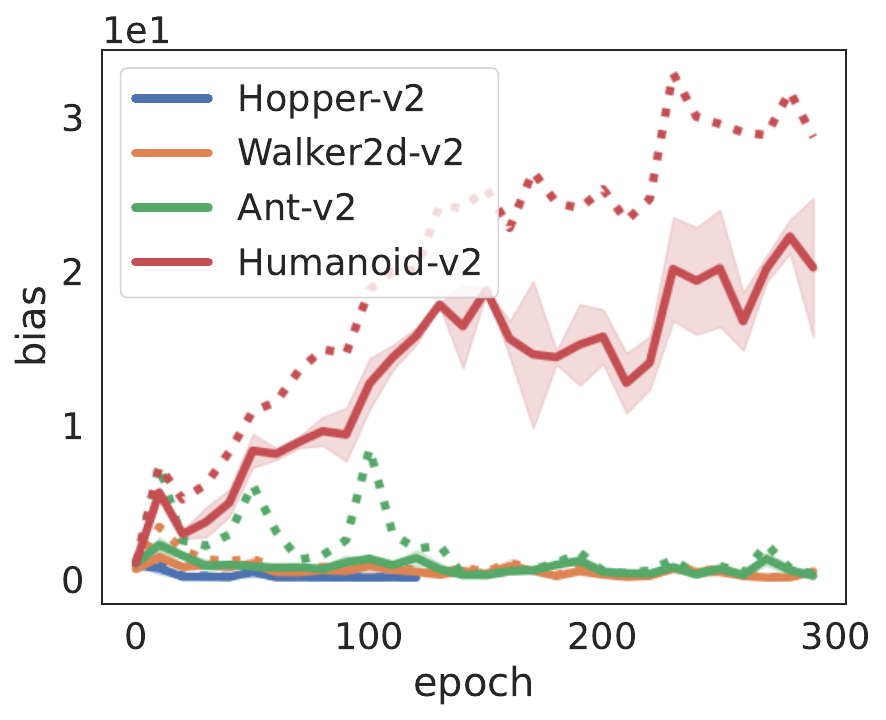}
\end{minipage}
\vspace{-0.7\baselineskip}
\caption{
Results of policy amendments (left) and Q-function amendments (right) in underperforming trials. 
The solid lines represent the post-amendment performances: return for the policy (left; i.e., $L_{\text{ret}}(\pi_{\theta, \mathbf{w}_*})$) and bias for the Q-function (right; i.e., $L_{\text{bias}}(Q_{\phi, \mathbf{w}_*})$). 
The dashed lines show the pre-amendment performances: return (left; i.e., $L_{\text{ret}}(\pi_{\theta})$) and bias (right; i.e., $L_{\text{bias}}(Q_{\phi})$). 
These results are averaged over two underperforming trials. 
The shaded regions around the solid lines represent one standard deviation around the mean. 
These figures demonstrate that the amendments improve returns in Hopper and Walker2d, and reduce biases in Ant and Humanoid.
}
\label{fig:cleansing_results}
\end{figure*}

What kinds of experiences negatively influence policy or Q-function performance?
\textbf{Policy performance:} Some experiences negatively influencing returns are associated with stumbling or falling. 
An example of such experiences in Hopper can be seen in this video: \url{https://github.com/user-attachments/assets/07d14535-bf16-4069-893a-f08f9ee9c7c7}
\textbf{Q-function performance:} Experiences negatively influencing Q-estimation bias tend to be older experiences. 
The lower part of Figure~\ref{fig:distribution_of_bias_return} in Appendix~\ref{app:additional_results} shows the distribution of influences on Q-estimation bias in each environment. 
For example, in the Humanoid environment, we observe that older experiences often have a negative influence (highlighted in darker colors). 

\textbf{Additional analyses and experiments.} 
We analyzed the correlation between experience influence values (Appendix~\ref{app:correlation_of_influence}). 
Additionally, we performed amendments using LOO (Appendix~\ref{sec:loo_prelim}). 
Furthermore, we performed amendments for other environments and RL agents using PIToD (Appendix ~\ref{app:adversarial_dmc} and Appendix~\ref{app:amend-droq-reset}). 

\section{Related work}\label{sec:related}
\textbf{Influence estimation in supervised learning.} 
Our research builds upon prior studies that estimate the influence of data within the supervised learning (SL) regime. 
In Section~\ref{sec:proposed_method}, we introduced our method for estimating the influence of data (i.e., experiences) in RL settings. 
Methods that estimate the influence of data have been extensively studied in the SL research community. 
Typically, these methods require SL loss functions that are twice differentiable with respect to model parameters (e.g., \citet{pmlr-v70-koh17a,yeh2018representer,hara2019data,koh2019accuracy,guo2020fastif,chen2021hydra,schioppa2022scaling}). 
However, these methods are not directly applicable to our RL setting, as such SL loss functions are unavailable. 
In contrast, turn-over dropout (ToD)~\citep{kobayashi2020efficient} estimates the influence without requiring differentiable SL loss functions. 
We extended ToD to RL settings (Sections~\ref{sec:proposed_method}, \ref{sec:experiments}, and \ref{sec:application}). 
For this extension of ToD, we provided a theoretical justification (Appendix~\ref{app:theory}) and considered practical implementations (Appendix~\ref{sec:practical_implementation}). 

\textbf{Influence estimation in off-policy evaluation (OPE)}. 
A few studies in the OPE community have focused on efficiently estimating the influence of experiences~\citep{gottesman2020interpretable,lobo2022data}. 
These studies are limited to estimating the influence on policy evaluation using nearest-neighbor or linear Q-functions. 
In contrast, our study estimates influence on a broader range of performance metrics (e.g., return or Q-estimation bias) using neural-network-based Q-functions and policies. 

\textbf{Prioritized experience replay (PER).}
In PER, the importance of experiences is estimated to prioritize experiences during experience replay. 
The importance of experiences is estimated based on criteria such as TD-error~\citep{Schaul2016,fedus2020revisiting} or on-policyness~\citep{novati2019remember,sun2020attentive}. 
Some readers might think that PER resembles our method. 
However, PER fundamentally differs from our method, as it cannot efficiently estimate or disable the influence of experiences in hindsight. 

\textbf{Interpretable RL.} 
Our method (Section~\ref{sec:proposed_method}) estimates the influence of experiences, thereby providing interpretability. 
Previous studies in the RL community have proposed interpretable methods based on symbolic (or relational) representation~\citep{dvzeroski2001relational,ijcai2018p675,lyu2019sdrl,garnelo2016towards,NIPS2017_fb89fd13,konidaris2018skills}, interpretable proxy models (e.g., decision trees)~\citep{degris2006learning,liu2019toward,coppens2019distilling,zhu2022extracting}, saliency explanation~\citep{zahavy2016graying,greydanus2018visualizing,mott2019towards,wang2020attribution,anderson2020mental}, and sparse kernel models~\citep{dao2018deep}~\footnote{For a comprehensive review of interpretable RL, see \citet{mliani2024explain}.}. 
Unlike these studies, our study proposes a method to estimate the influence of experiences on RL agent performance. 
This method helps us, for example, identify influential experiences when an RL agent performs poorly, as demonstrated in Section~\ref{sec:application}. 

\section{Conclusion and limitations}\label{sec:conclusion}
In this paper, we proposed PIToD, a policy iteration (PI) method that efficiently estimates the influence of experiences (Section~\ref{sec:proposed_method}). 
We demonstrated that PIToD (i) correctly estimates the influence of experiences (Section~\ref{sec:eval_self_influence}), and (ii) significantly reduces computational time compared to the leave-one-out (LOO) method (Section~\ref{sec:eval_computational_time}). 
Furthermore, we applied PIToD to identify and delete negatively influential experiences, which improved the performance of policies and Q-functions (Section~\ref{sec:application}). 

We believe that our work provides a solid foundation for understanding the relationship between experiences and RL agent performance. 
However, it has several limitations, which we discuss along with future directions in Appendix~\ref{sec:limitations_future_work}. 

\bibliographystyle{rlj}
\bibliography{icml2023}

\clearpage
\appendix


\clearpage
\section{Important theoretical property of PIToD}\label{app:theory}
In this section, we theoretically prove the following property of PIToD: 
``Assuming that the policy $\pi_{\theta}$ and the Q-function $Q_{\phi}$ are updated according to Algorithm~\ref{alg2:PI_PIToD}, the functions $Q_{\phi, \mathbf{w}_i}$ and $\pi_{\theta, \mathbf{w}_i}$, which use the flipped mask $\mathbf{w}_i$, are unaffected by the gradients associated with experience $e_i$.'' 
This property is important as it justifies the use of the flipped mask $\mathbf{w}_i$ to estimate the influence of $e_i$ in PIToD.

First, we define key terms for our theoretical proof:\\
\textbf{Experience}: 
We define an experience $e_i$ as $e_i = (s, a, r, s', i)$, where $s$ is the state, $a$ is the action, $r$ is the reward, $s'$ is the next state, and $i$ is a unique identifier. 
We also define another experience as $e_{i'}$, where $i'$ is a unique identifier.\\
\textbf{Parameters}: 
At the $j$-th iteration of Algorithm~\ref{alg2:PI_PIToD} (lines 3--6), we define the parameters of the Q-function and policy that are not dropped by the mask $\mathbf{m}_{i'}$ as $\phi_{j, \mathbf{m}_{i'}}$ and $\theta_{j, \mathbf{m}_{i'}}$, respectively. 
Additionally, we define the parameters dropped by $\mathbf{m}_{i'}$ as $\phi_{j, \mathbf{w}_{i'}}$ and $\theta_{j, \mathbf{w}_{i'}}$.\\
\textbf{Policy and Q-function}: We define the policy and Q-function, where all parameters except $\phi_{j, \mathbf{m}_{i'}}$ and $\theta_{j, \mathbf{m}_{i'}}$ are set to zero (i.e., dropped), as $Q_{\phi_{j, \mathbf{m}_{i'}}}$ and $\pi_{\theta_{j, \mathbf{m}_{i'}}}$. 
Similarly, the policy and Q-function, where all parameters except $\phi_{j, \mathbf{w}_{i'}}$ and $\theta_{j, \mathbf{w}_{i'}}$ are zero, are defined as $Q_{\phi_{j, \mathbf{w}_{i'}}}$ and $\pi_{\theta_{j, \mathbf{w}_{i'}}}$. 

Next, we introduce two assumptions required for our proof. 
The first assumption is for the policy and Q-function with masks. 
\begin{assumption}\label{ass:sparsity}
    $Q_{\phi_{j, \mathbf{m}_{i'}}}$ and $\pi_{\theta_{j, \mathbf{m}_{i'}}}$ can be replaced by $Q_{\phi_{j, \mathbf{m}_{i'}}'}$ and $\pi_{\theta_{j, \mathbf{m}_{i'}}'}$, whose parameters $\phi_{j, \mathbf{m}_{i'}}'$ and $\theta_{j, \mathbf{m}_{i'}}'$ satisfy the following gradient properties: 

    The property of $\phi_{j, \mathbf{m}_{i'}}'$ is as follows: 
    \begin{eqnarray}
         && \nabla_{\phi_{j, \mathbf{m}_{i'}}'} \left( r + \gamma Q_{\bar{\phi}_{j, \mathbf{m}_{i}}'}(s', a') - Q_{\phi_{j, \mathbf{m}_{i}}'}(s, a) \right)^2, ~~ a' \sim \pi_{\theta_{j, \mathbf{m}_i}'}(\cdot | s' ) \nonumber\\
         & = & \nabla _{\phi_{j, \mathbf{m}_{i'}}'} \left( r + \gamma Q_{\bar{\phi}_{j, \mathbf{m}_i}'}(s', a') - Q_{\phi_{j, \mathbf{m}_i}'}(s, a) \right)^2 \cdot \mathbb{I}(i=i'), ~~ a' \sim \pi_{\theta_{j, \mathbf{m}_i}'}(\cdot | s' ). \nonumber
    \end{eqnarray}
    Here, $\mathbb{I}$ is an indicator function that returns $1$ if the specified condition (i.e., $i = i'$) is true and 0 otherwise. 
    
    The property of $\theta_{j, \mathbf{m}_{i'}}'$ is as follows: 
    \begin{eqnarray}
            & & \nabla_{\theta_{j, \mathbf{m}_{i'}}'} Q_{\phi_{j+1, \mathbf{m}_{i}}'}(s, a), ~~~ a \sim \pi_{\theta_{j, \mathbf{m}_{i}}'}(\cdot | s) \nonumber\\
            & = & \nabla_{\theta_{j, \mathbf{m}_{i'}}'} Q_{\phi_{j+1, \mathbf{m}_{i}}'}(s, a) \cdot \mathbb{I}( i=i' ), ~~~ a \sim \pi_{\theta_{j, \mathbf{m}_{i}}'}(\cdot | s). \nonumber
    \end{eqnarray}
\end{assumption}
Intuitively, Assumption~\ref{ass:sparsity} can be interpreted as follows: ``$Q_{\phi_{j, \mathbf{m}_{i'}}}$ and $\pi_{\theta_{j, \mathbf{m}_{i'}}}$ are predominantly influenced by the experience $e_{i'}$ (i.e., the influence of other experiences is negligible).'' 

The second assumption is for $\phi_{j, \mathbf{w}_{i'}}$ and $\theta_{j, \mathbf{w}_{i'}}$: 
\begin{assumption}\label{ass:flip}
    For the gradient with respect to $\phi_{j, \mathbf{w}_{i'}}$, the following equation holds: 
    \begin{eqnarray}
         && \nabla_{\phi_{j, \mathbf{w}_{i'}}} \left( r + \gamma Q_{\bar{\phi}_{j, \mathbf{m}_{i}}}(s', a') - Q_{\phi_{j, \mathbf{m}_{i}}}(s, a) \right)^2, ~~ a' \sim \pi_{\theta_{j, \mathbf{m}_{i}}'}(\cdot | s') \nonumber\\
         &=& \nabla_{\phi_{j, \mathbf{w}_{i'}}} \left( r + \gamma Q_{\bar{\phi}_{j, \mathbf{m}_{i}}}(s', a') - Q_{\phi_{j, \mathbf{m}_{i}}}(s, a) \right)^2 \cdot \mathbb{I}( i \neq i' ), ~~ a' \sim \pi_{\theta_{j, \mathbf{m}_{i}}'}(\cdot | s'). \label{eq:ass_flip_no_grad_qfunc}
    \end{eqnarray}

    For the gradient with respect to $\theta_{j, \mathbf{w}_{i'}}$, the following equation holds: 
    \begin{eqnarray}
        && \nabla_{\theta_{j, \mathbf{w}_{i'}}} Q_{\phi_{j+1, \mathbf{m}_{i}}}(s, a), ~~~ a \sim \pi_{\theta_{j-1, \mathbf{m}_{i}}}(\cdot | s) \nonumber\\
        &=& \nabla_{\theta_{j, \mathbf{w}_{i'}}} Q_{\phi_{j+1, \mathbf{m}_{i}}}(s, a) \cdot \mathbb{I}( i \neq i' ), ~~~ a \sim \pi_{\theta_{j, \mathbf{m}_{i}}}(\cdot | s). \label{eq:ass_flip_no_grad_policy}
    \end{eqnarray}
\end{assumption}
Intuitively, Assumption~\ref{ass:flip} can be interpreted as follows: ``When parameters are updated using experience $e_i$, the parameters dropped out (i.e., $\phi_{j, \mathbf{w}_i}$ and $\theta_{j, \mathbf{w}_i}$) remain unaffected by gradients computed from $e_i$.'' 

Based on the above assumptions, we will derive the property of PIToD described at the beginning of this section~\footnote{``Assuming that the policy $\pi_{\theta}$ and the Q-function $Q_{\phi}$ are updated according to Algorithm~\ref{alg2:PI_PIToD}, the functions $Q_{\phi, \mathbf{w}_i}$ and $\pi_{\theta, \mathbf{w}_i}$, which use the flipped mask $\mathbf{w}_i$, are unaffected by the gradients associated with experience $e_i$.'' }. 
Some readers may think that Assumption~\ref{ass:flip} corresponds to this property. 
However, in addition to Assumption~\ref{ass:flip}, we must guarantee that the components used to create target signals for Eq.~\ref{eq:ass_flip_no_grad_qfunc} and Eq.~\ref{eq:ass_flip_no_grad_policy} (i.e., the components highlighted in red below) are also not influenced by $e_i$ when $i \neq i'$. 
Otherwise, even when $i \neq i'$, the parameters $\phi_{j, \mathbf{w}_{i}}$ and $\theta_{j, \mathbf{w}_{i}}$ could still be influenced indirectly through components affected by $e_i$. 
\begin{eqnarray}
     && \nabla_{\phi_{j, \mathbf{w}_{i'}}} \left( r + \gamma \textcolor{red}{ Q_{\bar{\phi}_{j, \mathbf{m}_{i}}}(s', a')} - Q_{\phi_{j, \mathbf{m}_{i}}}(s, a) \right)^2 \cdot \mathbb{I}( i \neq i' ), ~~ a' \sim \textcolor{red}{ \pi_{\theta_{j, \mathbf{m}_{i}}'}(\cdot | s')}. \nonumber
\end{eqnarray}
\begin{eqnarray}
    && \nabla_{\theta_{j, \mathbf{w}_{i'}}} \textcolor{red}{ Q_{\phi_{j+1, \mathbf{m}_{i}}}(s, a) } \cdot \mathbb{I}( i \neq i' ), ~~~ a \sim \pi_{\theta_{j, \mathbf{m}_{i}}}(\cdot | s). \nonumber
\end{eqnarray}
Based on Assumption~\ref{ass:sparsity}, we can ensure that these red-highlighted components are not influenced by $e_i$ when $i \neq i'$. 

Based on Assumption~\ref{ass:sparsity}, the following theorem holds: 
\begin{theorem}
    Given that, for $j > 0$, the parameters $\phi_{j, \mathbf{m}_{i'}}'$ and $\theta_{j, \mathbf{m}_{i'}}'$ are updated in the same way as the original parameters $\phi_{j, \mathbf{m}_{i'}}$ and $\theta_{j, \mathbf{m}_{i'}}$, according to Eq.~\ref{eq:Obj_Qfuncs_ToD} and Eq.~\ref{eq:Obj_Policy_ToD}, the following equation holds: 
    \begin{eqnarray}
        \phi_{j, \mathbf{m}_{i'}}' &\leftarrow& \phi_{j-1, \mathbf{m}_{i'}}' - \sum_{(s, a, r, s', i)} \nabla _{\phi_{j-1, \mathbf{m}_{i'}}'} \left( r + \gamma Q_{\bar{\phi}_{j-1, \mathbf{m}_{i}}'}(s', a') - Q_{\phi_{j-1, \mathbf{m}_{i}}'}(s, a) \right)^2 \cdot \mathbb{I}(i=i'), \nonumber\\ && ~~ a' \sim \pi_{\theta_{j-1, \mathbf{m}_{i}}'}(\cdot | s' ). \nonumber
    \end{eqnarray}
    \begin{eqnarray}
        \theta_{j, \mathbf{m}_{i'}}' \leftarrow \theta_{j-1, \mathbf{m}_{i'}}' - \sum_{(s, a, r, s', i)} \nabla_{\theta_{j-1, \mathbf{m}_{i'}}'} Q_{\phi_{j, \mathbf{m}_{i}}'}(s, a) \cdot \mathbb{I}( i = i' ), ~~~ a \sim \pi_{\theta_{j-1, \mathbf{m}_{i}}'}(\cdot | s).  \nonumber
    \end{eqnarray}
\end{theorem}
\begin{proof}
    \begin{eqnarray}
        \phi_{j, \mathbf{m}_{i'}}' &\leftarrow& \phi_{j-1, \mathbf{m}_{i'}}' - \nabla _{\phi_{j-1, \mathbf{m}_{i'}}'} \sum_{(s, a, r, s', i)} \left( r + \gamma Q_{\bar{\phi}_{j-1, \mathbf{m}_{i}}'}(s', a') - Q_{\phi_{j-1, \mathbf{m}_{i}}'}(s, a) \right)^2, \nonumber\\ && ~~ a' \sim \pi_{\theta_{j-1, \mathbf{m}_{i}}'}(\cdot | s' ) \nonumber\\
         &\stackrel{(1)}{=}& \phi_{j-1, \mathbf{m}_{i'}}' - \sum_{(s, a, r, s', i)} \nabla _{\phi_{j-1, \mathbf{m}_{i'}}'} \left( r + \gamma Q_{\bar{\phi}_{j-1, \mathbf{m}_{i}}'}(s', a') - Q_{\phi_{j-1, \mathbf{m}_{i}}'}(s, a) \right)^2 \cdot \mathbb{I}(i=i'), \nonumber\\ && ~~ a' \sim \pi_{\theta_{j-1, \mathbf{m}_{i}}'}(\cdot | s' ) \nonumber
    \end{eqnarray}

    \begin{eqnarray}
        \theta_{j, \mathbf{m}_{i'}}' &\leftarrow& \theta_{j-1, \mathbf{m}_{i'}}' - \nabla_{\theta_{j-1, \mathbf{m}_{i'}}'} \sum_{(s, a, r, s', i)}  Q_{\phi_{j, \mathbf{m}_{i}}'}(s, a), ~~~ a \sim \pi_{\theta_{j-1, \mathbf{m}_{i}}'}(\cdot | s) \nonumber\\
        &\stackrel{(1)}{=}& \theta_{j-1, \mathbf{m}_{i'}}' - \sum_{(s, a, r, s', i)} \nabla_{\theta_{j-1, \mathbf{m}_{i'}}'} Q_{\phi_{j, \mathbf{m}_{i}}'}(s, a) \cdot \mathbb{I}( i=i' ), ~~~ a \sim \pi_{\theta_{j-1, \mathbf{m}_{i}}'}(\cdot | s) \nonumber
    \end{eqnarray}
    (1) Apply Assumption~\ref{ass:sparsity}. 
\end{proof}
This theorem implies that $Q_{\phi_{j, \mathbf{m}_{i'}}'}$ and $\pi_{\theta_{j, \mathbf{m}_{i'}}'}$ are dominantly influenced by the experience $e_{i'}$ for $j>0$. 
Thus, if the red-highlighted components above can be replaced with these components, we can say that $\phi_{j, \mathbf{w}_i}$ and $\theta_{j, \mathbf{w}_i}$ are not influenced by gradients depending on $e_i$ in both cases of $i = i'$ and $i \neq i'$. 
Below, we will show that such a replacement is doable. 

Based on Assumptions~\ref{ass:sparsity} and \ref{ass:flip}, the following theorem holds:
\begin{theorem}
For any $j > 0$, the parameters $\phi_{j, \mathbf{w}_{i'}}$ and $\theta_{j, \mathbf{w}_{i'}}$ in Algorithm~\ref{alg2:PI_PIToD} are updated as follows:
\begin{eqnarray}
     \phi_{j, \mathbf{w}_{i'}} & \leftarrow & \phi_{j-1, \mathbf{w}_{i'}} - \sum_{(s, a, r, s', i)}  \nabla_{\phi_{j-1, \mathbf{w}_{i'}}} \left( r + \gamma Q_{\bar{\phi}_{j-1, \mathbf{m}_{i}}'}(s', a') - Q_{\phi_{j-1, \mathbf{m}_{i}}}(s, a) \right)^2 \cdot \mathbb{I}( i \neq i' ), \nonumber\\ && ~~ a' \sim \pi_{\theta_{j-1, \mathbf{m}_{i}}'}(\cdot | s' ) \nonumber
\end{eqnarray}
\begin{eqnarray}
    \theta_{j, \mathbf{w}_{i'}} &\leftarrow& \theta_{j-1, \mathbf{w}_{i'}} - \sum_{(s, a, r, s', i)} \nabla_{\theta_{j-1, \mathbf{w}_{i'}}} Q_{\phi_{j, \mathbf{m}_{i}}'}(s, a) \cdot \mathbb{I}( i \neq i' ), ~~~ a \sim \pi_{\theta_{j-1, \mathbf{m}_{i}}}(\cdot | s) \nonumber
\end{eqnarray}
\end{theorem}
\begin{proof}
For $\phi_{j, \mathbf{w}_{i'}}$, 
\begin{eqnarray}
     \phi_{j, \mathbf{w}_{i'}} & \leftarrow & \phi_{j-1, \mathbf{w}_{i'}} - \nabla_{\phi_{j-1, \mathbf{w}_{i'}}} \sum_{(s, a, r, s', i)} \left( r + \gamma Q_{\bar{\phi}_{j-1, \mathbf{m}_{i}}}(s', a') - Q_{\phi_{j-1, \mathbf{m}_{i}}}(s, a) \right)^2, \nonumber\\ && ~~ a' \sim \pi_{\theta_{j-1, \mathbf{m}_{i}}}(\cdot | s' ) \nonumber\\
     & \stackrel{(1)}{=} & \phi_{j-1, \mathbf{w}_{i'}} - \sum_{(s, a, r, s', i)} \nabla_{\phi_{j-1, \mathbf{w}_{i'}}} \left( r + \gamma Q_{\bar{\phi}_{j-1, \mathbf{m}_{i}}}(s', a') - Q_{\phi_{j-1, \mathbf{m}_{i}}}(s, a) \right)^2 \cdot \mathbb{I}( i \neq i' ), \nonumber\\ && ~~ a' \sim \pi_{\theta_{j-1, \mathbf{m}_{i}}}(\cdot | s' ) \nonumber\\
     & \stackrel{(2)}{=} & \phi_{j-1, \mathbf{w}_{i'}} - \sum_{(s, a, r, s', i)}  \nabla_{\phi_{j-1, \mathbf{w}_{i'}}} \left( r + \gamma Q_{\bar{\phi}_{j-1, \mathbf{m}_{i}}'}(s', a') - Q_{\phi_{j-1, \mathbf{m}_{i}}}(s, a) \right)^2 \cdot \mathbb{I}( i \neq i' ), \nonumber\\ && ~~ a' \sim \pi_{\theta_{j-1, \mathbf{m}_{i}}'}(\cdot | s' ) \nonumber
\end{eqnarray}
(1) Apply Assumption~\ref{ass:flip}. 
(2) Apply Assumption~\ref{ass:sparsity}. 

Similarly, for $\theta_{j, \mathbf{w}_{i'}}$, 
\begin{eqnarray}
    \theta_{j, \mathbf{w}_{i'}} &\leftarrow& \theta_{j-1, \mathbf{w}_{i'}} - \nabla_{\theta_{j-1, \mathbf{w}_{i'}}} \sum_{(s, a, r, s', i)}  Q_{\phi_{j, \mathbf{m}_{i}}}(s, a), ~~~ a \sim \pi_{\theta_{j-1, \mathbf{m}_{i}}}(\cdot | s) \nonumber\\
    & \stackrel{(1)}{=}& \theta_{j-1, \mathbf{w}_{i'}} - \sum_{(s, a, r, s', i)} \nabla_{\theta_{j-1, \mathbf{w}_{i'}}} Q_{\phi_{j, \mathbf{m}_{i}}}(s, a) \cdot \mathbb{I}( i \neq i' ), ~~~ a \sim \pi_{\theta_{j-1, \mathbf{m}_{i}}}(\cdot | s) \nonumber\\
    & \stackrel{(2)}{=}& \theta_{j-1, \mathbf{w}_{i'}} - \sum_{(s, a, r, s', i)} \nabla_{\theta_{j-1, \mathbf{w}_{i'}}} Q_{\phi_{j, \mathbf{m}_{i}}'}(s, a) \cdot \mathbb{I}( i \neq i' ), ~~~ a \sim \pi_{\theta_{j-1, \mathbf{m}_{i}}}(\cdot | s) \nonumber
\end{eqnarray}
\end{proof}
This theorem implies that:
\begin{itemize}
    \item[(i)] When $i = i'$, neither $\theta_{j, \mathbf{w}_{i'}}$ nor $\phi_{j, \mathbf{w}_{i'}}$ is influenced by gradients dependent on experience $e_{i'}$. 
    \item[(ii)] When $i \neq i'$, $\theta_{j, \mathbf{w}_{i'}}$ and $\phi_{j, \mathbf{w}_{i'}}$ are updated without depending on the components that might be influenced by $e_{i'}$. 
\end{itemize}
Therefore, we conclude that ``$Q_{\phi, \mathbf{w}_{i'}}$ and $\pi_{\theta, \mathbf{w}_{i'}}$, and consequently $Q_{\phi, \mathbf{w}_i}$ and $\pi_{\theta, \mathbf{w}_i}$, are not influenced by the gradients related to the experiences $e_{i'}$ and $e_i$, respectively.''

\section{Analyzing and minimizing overlap in elements of masks}\label{sec:analyzing_overlap_in_mask}
In our method (Section~\ref{sec:proposed_method}), each experience is assigned a mask. 
If there is significant overlap in the elements of different masks, one experience could significantly interfere with other experiences. 
In this section, we discuss (i) the expected overlap between the masks of experiences $e_i$ and $e_{i'}$ and (ii) the dropout rate that minimizes this overlap. 

For discussion, we introduce the following definitions and assumptions. 
We define the mask size as $M$, and the number of overlapping elements between masks as $m$. 
We assume that each mask element is independently initialized as 0 with probability $p$ (i.e., dropout rate) and 1 with probability $1 - p$. 

Below, we derive the probability and expected number of overlaps in the mask elements.\\
\textbf{Probability of $m$ overlaps.} 
First, we calculate the probability that a specific position in the masks of $e_i$ and $e_{i'}$ has the same value. 
The probability that both elements of the masks have 0 at the same position is $p \cdot p = p^2$. 
Similarly, the probability that both elements have 1 at the same position is $(1-p) \cdot (1-p) = (1-p)^2$. 
Therefore, the probability $q$ that the values at a specific position in the masks are the same is 
\begin{eqnarray}
    q = p^2 + (1-p)^2 = 2p^2 - 2p + 1. \label{eq:prob_single_overlap}
\end{eqnarray}
The probability that the masks have $m$ overlaps follows the binomial distribution: 
\begin{eqnarray}
    \binom{M}{m} q^m (1-q)^{M-m}  \label{eq:prob_m_overlaps}. 
\end{eqnarray}
\textbf{Expected number of overlaps.} 
Using Eq.\ref{eq:prob_single_overlap} and Eq.\ref{eq:prob_m_overlaps}, the expected number of overlaps is calculated as 
\begin{dmath}
    \sum_{k=0}^{M} k \binom{M}{k} q^k (1-q)^{M-k} = M q = M (2p^2 - 2p + 1). \label{eq_expected_k}
\end{dmath}
For better understanding, we show a plot of Eq.~\ref{eq_expected_k} values with respect to $p$ and $M$ in Figure~\ref{fig:expected_overrap}. 
\begin{figure}[h!]
\begin{center}
\includegraphics[clip, width=0.60\hsize]{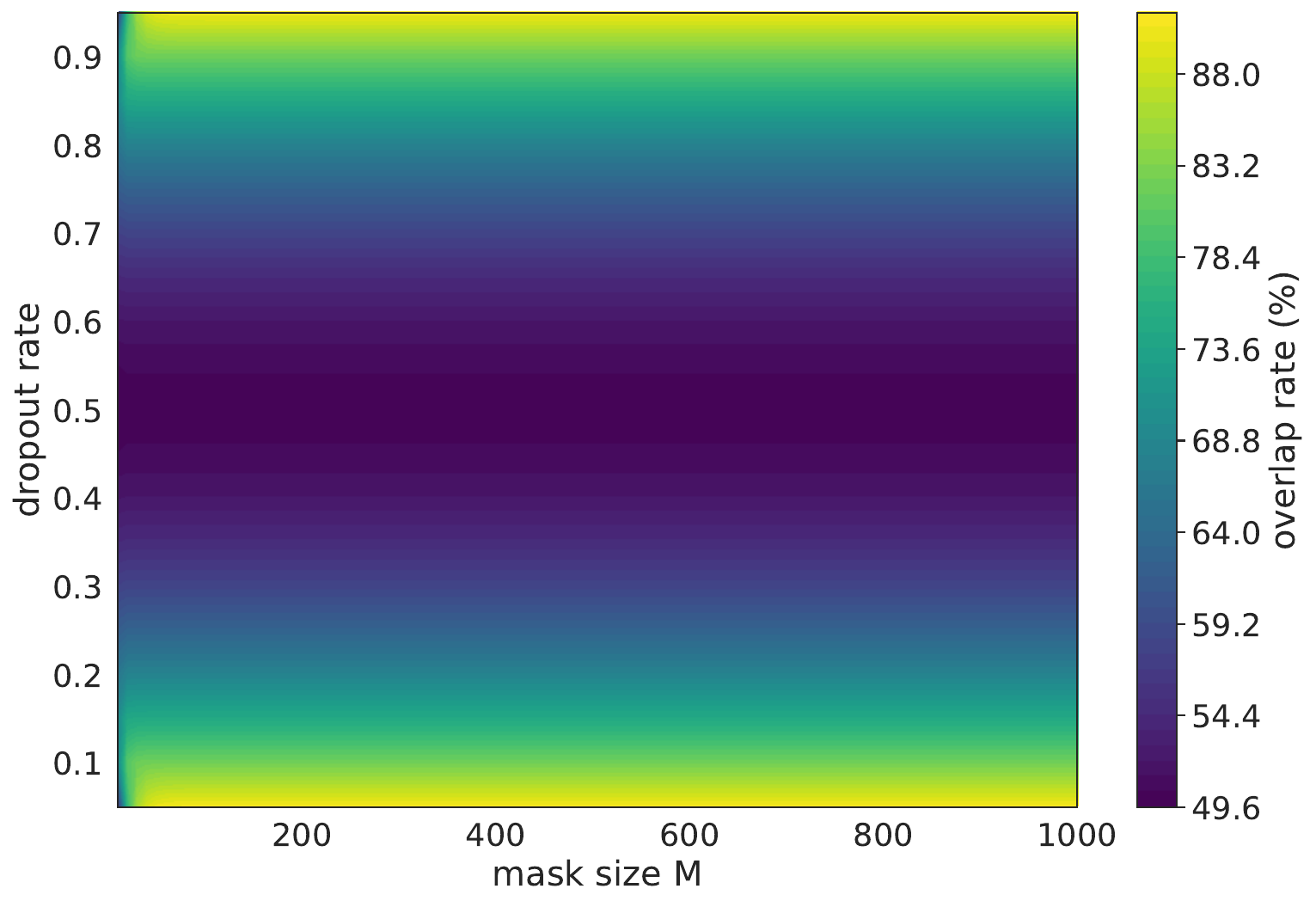}
\end{center}
\caption{The distribution of the expected number of overlaps (Eq.~\ref{eq_expected_k}) with respect to the dropout rate \( p \) and mask size \( M \). For clarity, we plot the expected overlap rate (\( m/M \)) instead of the expected number of overlaps \( m \).}
\label{fig:expected_overrap}
\end{figure}

\textbf{The dropout rate of $p=0.5$ minimizes the expected number of overlaps.} 
Since Eq.~\ref{eq_expected_k} is convex in $p$, the value of $p$ that minimizes the expected overlap is determined by solving $\frac{\mathrm{d} M (2p^2 - 2p + 1)}{\mathrm{d} p} = 0$. 
As a result, we find that $p = 0.5$ minimizes the expected overlap.
%
At \( p=0.5 \), the expected overlap between two masks is 50\%. 
Figure~\ref{fig:prob_overaprate_p05} shows the probability of the overlap rate \( m/M \) with \( p=0.5 \) for various values of \( M \). 
From this figure, we see that the probability of having a between 0-50\% overlap is very high, while the probability of having a between 50-100\% overlap is very low, regardless of the value of \( M \). 
\begin{figure}[h!]
\begin{center}
\includegraphics[clip, width=0.60\hsize]{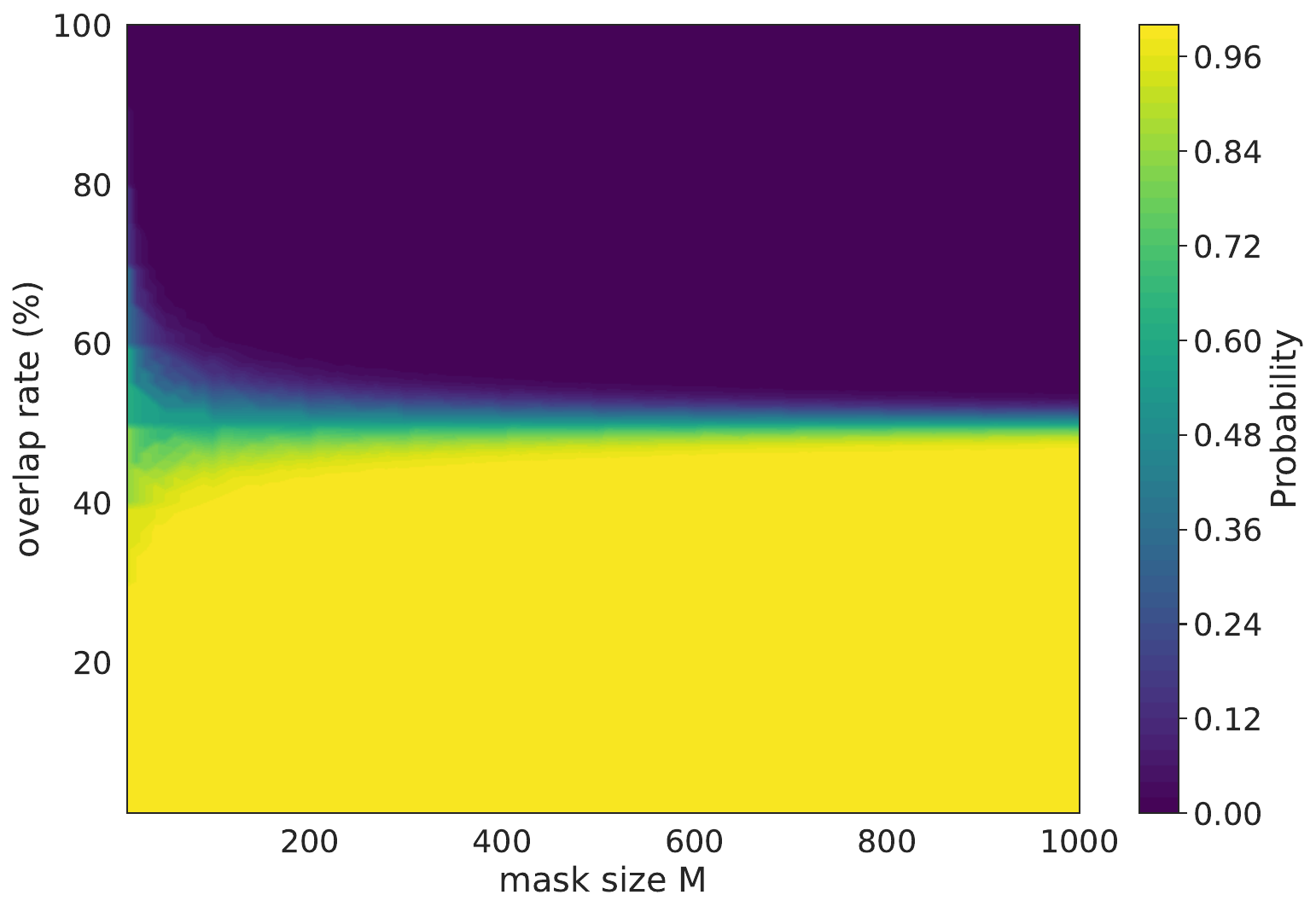}
\end{center}
\caption{The probability of the overlap rate $m/M$ with \( p = 0.5 \) for various values of \( M \).}
\label{fig:prob_overaprate_p05}
\end{figure}

\clearpage
\section{Practical implementation of PIToD for Section~\ref{sec:experiments} and Section~\ref{sec:application}}\label{sec:practical_implementation}
In this section, we describe the practical implementation of PIToD. 
Specifically, we explain (i) the soft actor-critic (SAC)~\citep{haarnoja2018soft} version of PI with a mask, (ii) group mask, and (iii) key implementation decisions to improve learning. 
This practical implementation is used in our experiments (Section~\ref{sec:experiments} and Section~\ref{sec:application}). 

\textbf{(i) SAC version of PI with a mask.} 
The SAC version of PI with a mask is presented in Algorithm~\ref{alg1:SACToD}. 
The mask is applied to the policy and Q-functions during policy evaluation (lines 5--8) and policy improvement (line 9). 
For the policy evaluation, two Q-functions $Q_{\phi_j}$, where $j \in \{1, 2\}$, are updated as: 
\begin{dmath}
    \phi_j \leftarrow \phi_j - \nabla_{\phi_j} \mathbb{E}_{e_i = (s, a, r, s', i) \sim \mathcal{B},~ a' \sim \pi_{\theta, \mathbf{m}_i}(\cdot | s' )} \left[ \left( r + \gamma \left( \min_{j' = 1, 2} Q_{\bar{\phi}_{j'}, \mathbf{m}_i}(s', a') - \alpha \log \pi_{\theta, \mathbf{m}_i}(a' | s') \right) \nonumber\\ ~~~~~~~~~~~~~~~~~~~~~~~~~~~~~~~~~~~~~~~~~~~~~~~~~~~~~~~~~ - Q_{\phi_j, \mathbf{m}_i}(s, a) \right)^2 \right]. \label{eq:Obj_Qfuncs_SACToD} 
\end{dmath}
This is a variant of Eq.~\ref{eq:Obj_Qfuncs} that uses clipped double Q-learning with two target Q-functions $Q_{\bar{\phi}_{j'}, \mathbf{m}_i}$ and entropy bonus $\alpha \log \pi_{\theta, \mathbf{m}_i}(a' | s')$. 
Additionally, for policy improvement, policy $\pi_\theta$ is updated as 
\begin{dmath}
    \theta \leftarrow \theta + \nabla_{\theta} \mathbb{E}_{e_i = (s, i) \sim \mathcal{B}, ~ a_{\theta, \mathbf{m}_i}, a \sim \pi_{\theta, \mathbf{m}_i}(\cdot | s)} \left[ \left( \frac{1}{2} \sum_{j=1}^{2} Q_{\phi_j, \mathbf{m}_i}(s, a_{\theta, \mathbf{m}_i}) - \alpha \log \pi_{\theta, \mathbf{m}_i}(a | s) \right) \right]. \label{eq:Obj_Policy_SACToD} 
\end{dmath}
This is a variant of Eq.~\ref{eq:Obj_Policy} that uses the entropy bonus. 
%
\begin{algorithm*}[t!]
\caption{SAC version of PI with \textcolor{green}{group mask} in PIToD}
\label{alg1:SACToD}
\begin{algorithmic}[1]
\STATE Initialize policy parameters $\theta$, Q-function parameters $\phi_1$, $\phi_2$, and an empty replay buffer $\mathcal{B}$.
\FOR{$i'=0,..., I$}
    \STATE Take action $a \sim \pi_\theta(\cdot | s)$; Observe reward $r$ and next state $s'$; Define an experience \textcolor{green}{using the group identifier $i'' \leftarrow \lfloor i' / 5000 \rfloor$ as $e_{i''} = (s, a, r, s', i'')$}; $\mathcal{B} \leftarrow \mathcal{B} \bigcup \{ e_{i''}\}$.  
    \STATE Sample experiences $\{ (s, a, r, s', i), ... \}$ from $\mathcal{B}$ (Here, $e_i = (s, a, r, s', i)$). 
    \STATE Compute target $y_i$: 
    \vspace{-0.7\baselineskip}\begin{equation}
        y_i = r + \gamma \left( \min_{j = 1, 2} Q_{\bar{\phi}_j, \mathbf{m}_i}(s', a') - \alpha \log \pi_{\theta, \mathbf{m}_i}(a' | s') \right), ~~ a' \sim \pi_{\theta, \mathbf{m}_i}(\cdot | s' ). \nonumber
    \end{equation}\vspace{-0.7\baselineskip}
    \FOR{ $j=1, 2$ }
        \STATE Update $\phi_j$ with gradient descent using 
        \vspace{-0.7\baselineskip}\begin{equation}
            \nabla_{\phi_j} \sum_{(s, a, r, s', i)} \left( Q_{\phi_j, \mathbf{m}_i}(s, a) - y_i \right)^2. \nonumber
        \end{equation}\vspace{-0.7\baselineskip}
        \STATE Update target networks with $\bar{\phi}_j \leftarrow \rho \bar{\phi}_j + (1-\rho) \phi_j$.
    \ENDFOR
    \STATE Update $\theta$ with gradient ascent using 
    \vspace{-0.7\baselineskip}\begin{equation}
        \nabla_\theta \sum_{(s, a, r, s', i)} \left( \frac{1}{2} \sum_{j=1}^{2} Q_{\phi_j, \mathbf{m}_i}(s, a_{\theta, \mathbf{m}_i}) - \alpha \log \pi_{\theta, \mathbf{m}_i}(a | s) \right), ~~~ a, a_{\theta, \mathbf{m}_i} \sim \pi_{\theta, \mathbf{m}_i}(\cdot | s). \nonumber
    \end{equation}\vspace{-0.7\baselineskip}
\ENDFOR
\end{algorithmic}
\end{algorithm*}

\textcolor{green}{\textbf{(ii) Group Mask.}} 
In our preliminary experiments, we found that the influence of a single experience on performance was negligibly small. 
To examine more significant influences, we shifted our focus from the influence of individual experiences to grouped experiences. 
To estimate the influence of grouped experiences, we organize experiences into groups and assign a mask to each group. 
Specifically, we treated 5000 experiences as a single group. 
This grouping process was implemented by assigning a group identifier to each experience, calculated as \(i'' \leftarrow \lfloor i' / 5000 \rfloor\) (line 3 of Algorithm~\ref{alg1:SACToD}). 

\textbf{(iii) Key implementation decisions to improve learning.} 
In our preliminary experiments, we found that directly applying masks and flipped masks to dropping out the parameters of the policy and Q-function degrades learning performance. 
To address this issue, we implemented macro-block dropout and layer normalization (Figure~\ref{fig:NetworkArchitecture}). 
\textbf{Macro-block dropout.} 
Instead of applying dropout to individual parameters, we apply dropout at the block level. 
Specifically, we group several parameters into a ``block'' and apply dropout to these blocks. 
In our experiment, we used an ensemble of 20 multi-layer perceptrons (MLPs) for the policy and Q-function, and treated each MLP's parameters as a single block. 
We implement dropout by multiplying each MLP's output by the corresponding element of the binary mask $\mathbf{m}_i$ (or the flipped mask $\mathbf{w}_i$).
\textbf{Layer normalization.} 
We applied layer normalization~\citep{ba2016layer} after each activation (ReLU) layer. 
Recent works show that layer normalization improves learning in a wide range of RL settings (e.g., \citet{hiraoka2022dropout,ball2023efficient,nauman24over}). 
\begin{figure}[t!]
\begin{center}
\includegraphics[clip, width=0.99\hsize]{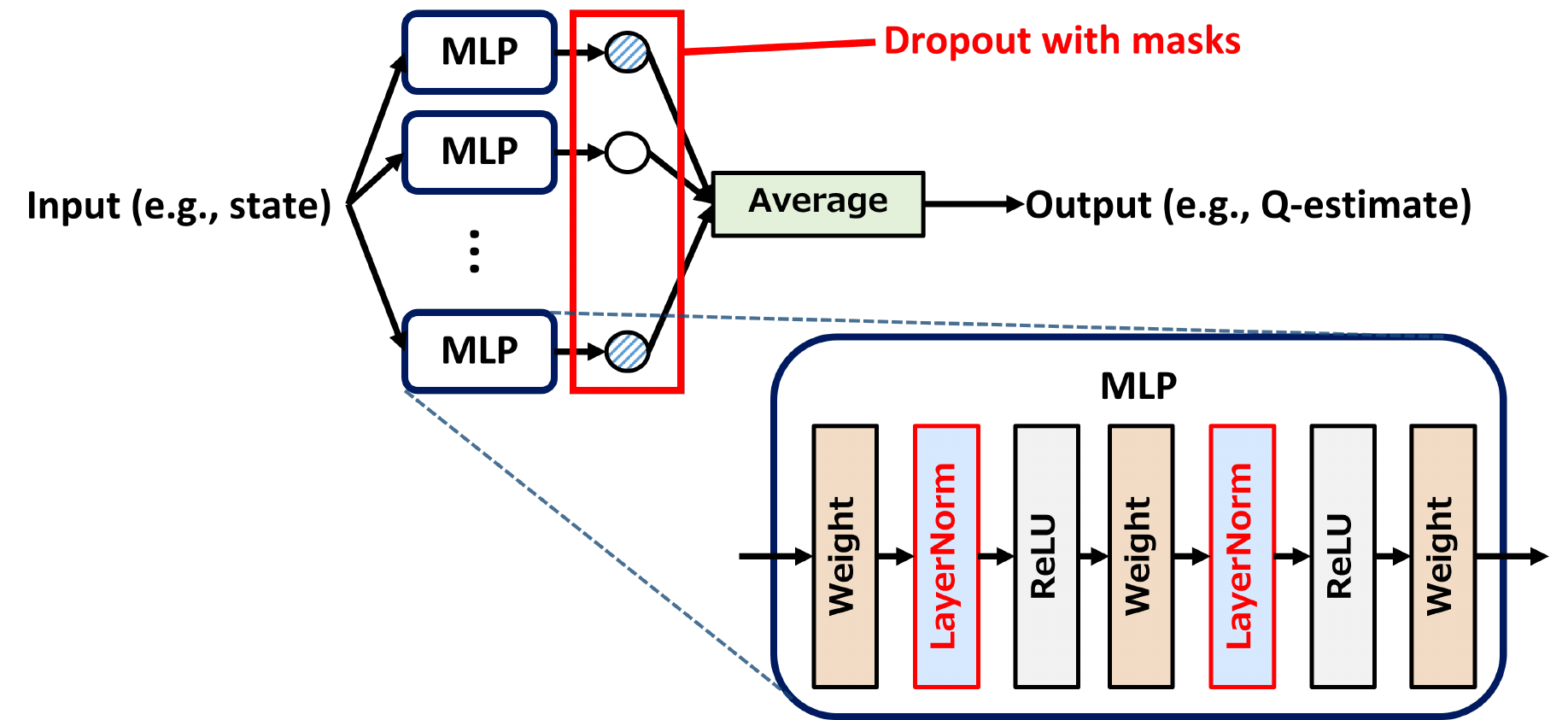}
\end{center}
\caption{
Network architectures for policy and Q-function. 
The policy network takes states as inputs and outputs the parameters of the policy distribution (mean and variance for a Gaussian distribution). 
The Q-function network takes state-action pairs as inputs and outputs Q-estimates. 
These networks incorporate macro-block dropout and layer normalization. 
\textbf{Macro-block dropout.} Our architecture utilizes an ensemble of 20 multi-layer perceptrons (MLPs), applying dropout with masks (or flipped masks) to each MLP's output. 
\textbf{Layer normalization.} 
Layer normalization is applied after every activation (ReLU) layer in each MLP.
}
\label{fig:NetworkArchitecture}
\end{figure}

To evaluate the effect of our key implementation decisions, we compare four implementations of Algorithm~\ref{alg1:SACToD}: \\
 \textbf{1. PIToD} applies vanilla dropout with masks to each parameter of the policy and Q-function.\\ 
 \textbf{2. PIToD+LN} applies layer normalization to the policy and Q-function.\\
 \textbf{3. PIToD+MD} applies macro-block dropout to the policy and Q-function.\\
 \textbf{4. PIToD+LN+MD} applies layer normalization and macro-block dropout to the policy and Q-function.\\
These implementations are compared based on the empirical returns obtained in test episodes. 

The comparison results (Figure~\ref{fig:ablation_study_result}) indicate that the implementation with our key decisions (PIToD+LN+MD) achieves the highest returns in each environment. 
\begin{figure*}[h!]
\begin{minipage}{1.0\hsize}
\includegraphics[clip, width=0.49\hsize]{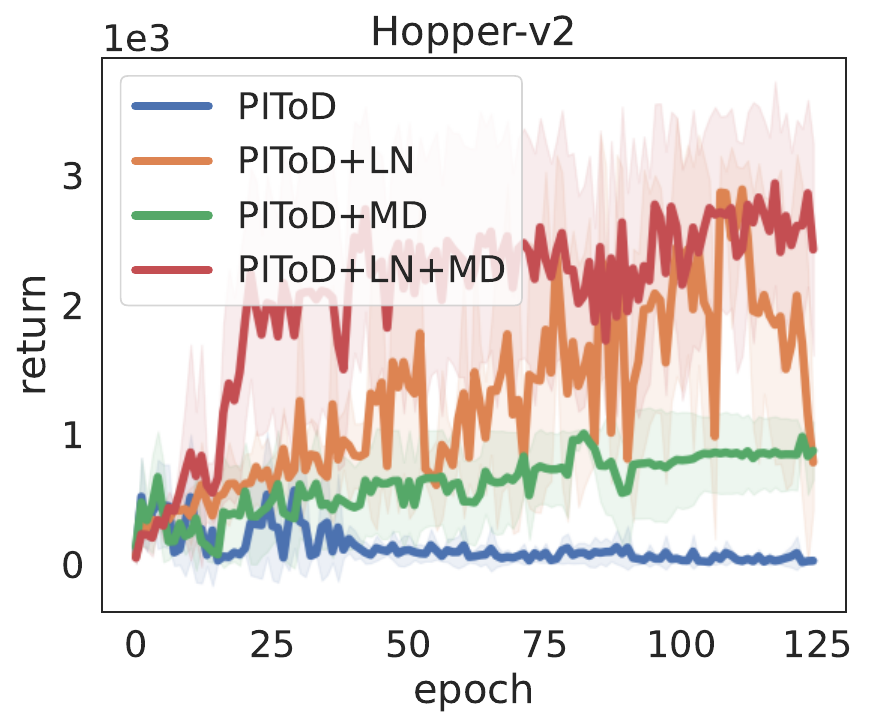}
\includegraphics[clip, width=0.49\hsize]{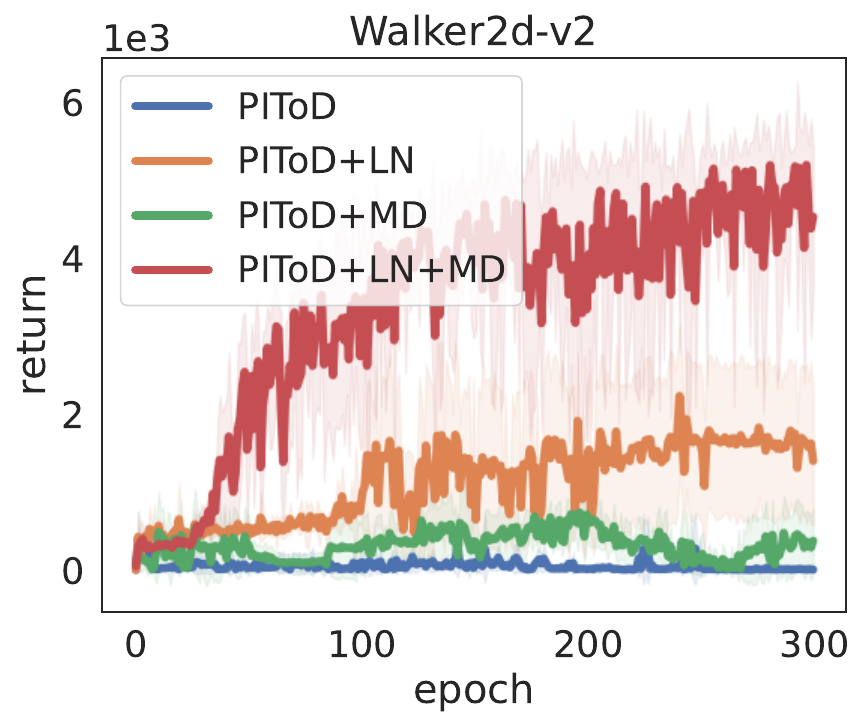}
\end{minipage}
\begin{minipage}{1.0\hsize}
\includegraphics[clip, width=0.49\hsize]{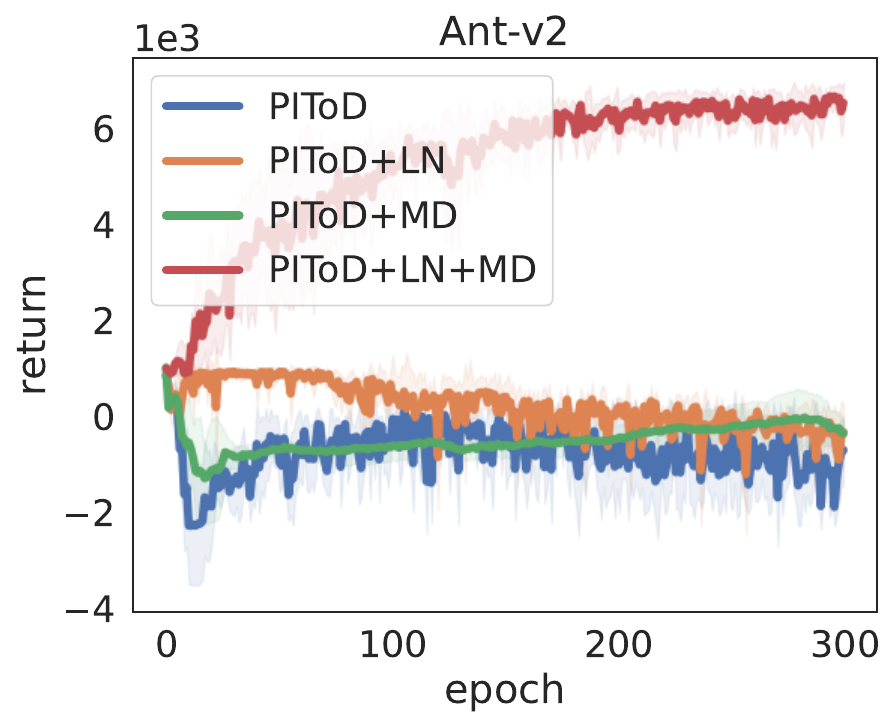}
\includegraphics[clip, width=0.49\hsize]{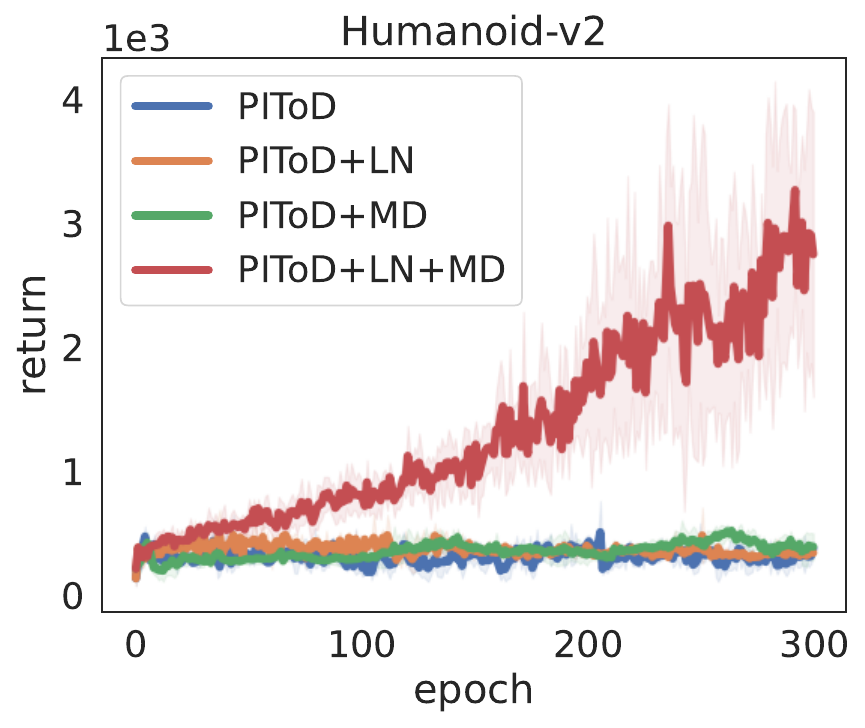}
\end{minipage}
\caption{
Ablation study results. 
The vertical axis represents returns, and the horizontal axis represents epochs. 
Each line represents the mean of ten trials, and the shaded region represents one standard deviation around the mean. 
In each environment, the implementation with our key decisions (PIToD+LN+MD) achieves the highest returns. 
}
\label{fig:ablation_study_result}
\end{figure*}

\textbf{Reference inference times.} 
In this section, we introduced a group mask, treating 5000 experiences as a single group.
Grouping multiple experiences not only improves performance but also reduces the number of targets for influence estimation, thus decreasing computation time.
To quantify the effect of grouping on the computation time, we compared the following three cases:\\
\textbf{1-group:} Time required to estimate the influence of a single group of 5000 experiences.\\
\textbf{5000-experiences:} Time required to sequentially estimate the influence of each of the 5000 experiences.\\
\textbf{5000-experiences-parallel:} Time to estimate the influence of 5000 experiences in parallel using a single GPU forward pass.\\
For comparison, we measure the time required to estimate the self-influence with respect to policy evaluation and policy improvement. 
The results are summarized in Table~\ref{tab:computation_time}.
In all environments, the 5000-experiences-parallel case takes approximately 50--60 times longer, and the 5000-experiences case takes approximately 2400--2500 times longer than the 1-group case. 
\begin{table}[t]
\centering
\caption{Comparison of influence estimation time (in seconds)}
\label{tab:computation_time}
\begin{tabular}{l||c|c|c}\hline
              & 1-group & 5000-experiences-parallel & 5000-experiences \\\hline\hline
Hopper-v2     & 0.0081 & 0.4811 & 20.0540 \\\hline
Walker2d-v2   & 0.0084 & 0.4934 & 20.4301 \\\hline
Ant-v2        & 0.0084 & 0.4840 & 20.4246 \\\hline
Humanoid-v2   & 0.0080 & 0.4930 & 19.7965 \\\hline
\end{tabular}
\end{table}

\textbf{Memory-efficient alternative implementation of masks.}  
In our implementation, a unique mask is explicitly stored in memory for each experience group. 
However, this implementation can become memory-intensive when scaling to a large number of parameters or groups.  
For memory-constrained settings, we recommend a memory-efficient alternative:  
instead of storing full binary mask vectors, only a unique scalar value (i.e., $i$ in Algorithm~\ref{alg1:SACToD}) is stored for each group, which is used as a random seed. 
The corresponding mask can then be generated on demand using a pseudo-random number generator initialized with that value. 
This approach eliminates the need to store full masks in memory and could enable the scalable application of PIToD to large-scale settings. 

\clearpage
\section{Algorithm for amending policy and Q-function used in Section~\ref{sec:application}}\label{app:alg_amend}
\begin{algorithm*}[h!]
\caption{Amendment of policy and Q-function using influence estimates. Lines 5--7 are for \textcolor{scarlet}{policy amendment.} Lines 8--10 are for \textcolor{violet}{Q-function amendment.}}
\label{alg4:Amendment}
\begin{algorithmic}[1]
\STATE Initialize policy parameters $\theta$, Q-function parameters $\phi$, and an empty replay buffer $\mathcal{B}$. Set the influence estimation interval $I_{\text{ie}}$. 
\FOR{$i' = 0, ..., I$ iterations} 
    \STATE Execute environment interaction, store experiences, and perform policy iteration as per lines 3--6 of Algorithm~\ref{alg2:PI_PIToD}.
    \IF{$ i' \% I_{\text{ie}} = 0 $}
        \color{scarlet}
        \STATE Identify $\mathbf{w}_*$ for policy as follows:
        \begin{dmath*}
             \mathbf{w}_{*} = \argmax_{\mathbf{w}_i} L_{\text{ret}}\left(\pi_{\theta, \mathbf{w}_i} \right) - L_{\text{ret}}\left(\pi_{\theta} \right). \nonumber
        \end{dmath*}
        \IF{$L_{\text{ret}}\left(\pi_{\theta, \mathbf{w}_*} \right) - L_{\text{ret}}\left(\pi_{\theta} \right) > 0$}
            \STATE Evaluate the return of the amended policy $L_{\text{ret}}\left(\pi_{\theta, \mathbf{w}_*} \right)$. 
        \ENDIF

        \color{violet}
        \STATE Identify $\mathbf{w}_*$ for Q-function as follows: 
        \begin{dmath*}
             \mathbf{w}_{*} = \argmin_{\mathbf{w}_i} L_{\text{bias}}\left(Q_{\phi, \mathbf{w}_i}\right) - L_{\text{bias}}\left(Q_{\phi}\right). \nonumber
        \end{dmath*}
        \IF{$L_{\text{bias}}\left(Q_{\phi, \mathbf{w}_*}\right) - L_{\text{bias}}\left(Q_{\phi}\right) < 0$}
            \STATE Evaluate the Q-estimation bias of the amended Q-function $L_{\text{bias}}\left(Q_{\phi, \mathbf{w}_*}\right)$. 
        \ENDIF
    \ENDIF
\ENDFOR
\end{algorithmic}
\end{algorithm*}

\clearpage
\section{Supplementary experimental results for Section~\ref{sec:application}}\label{app:additional_results}
\begin{figure*}[h!]
\begin{minipage}{1.0\hsize}
\includegraphics[clip, width=0.49\hsize]{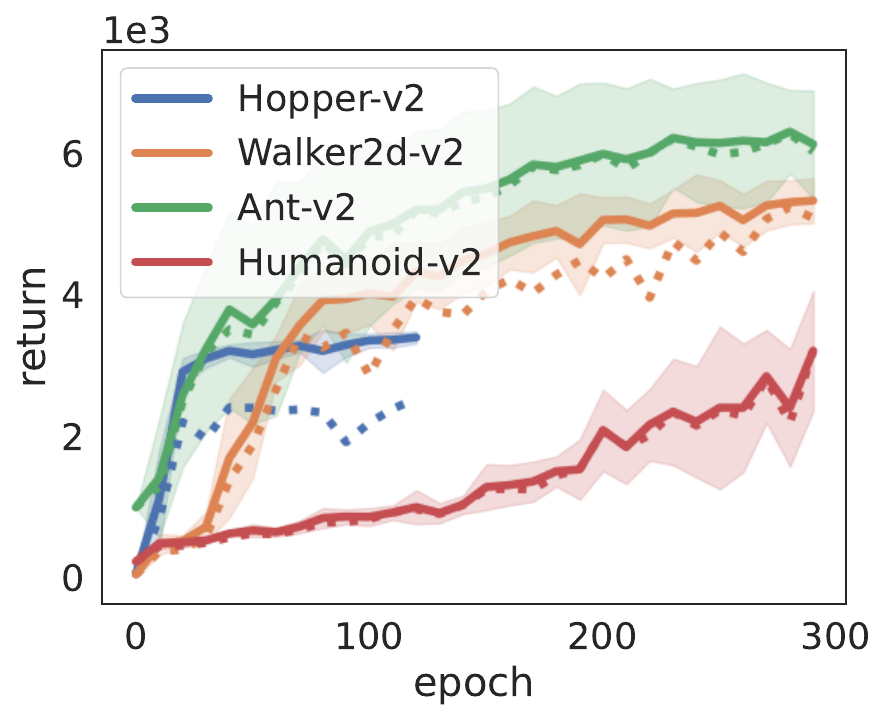}
\includegraphics[clip, width=0.49\hsize]{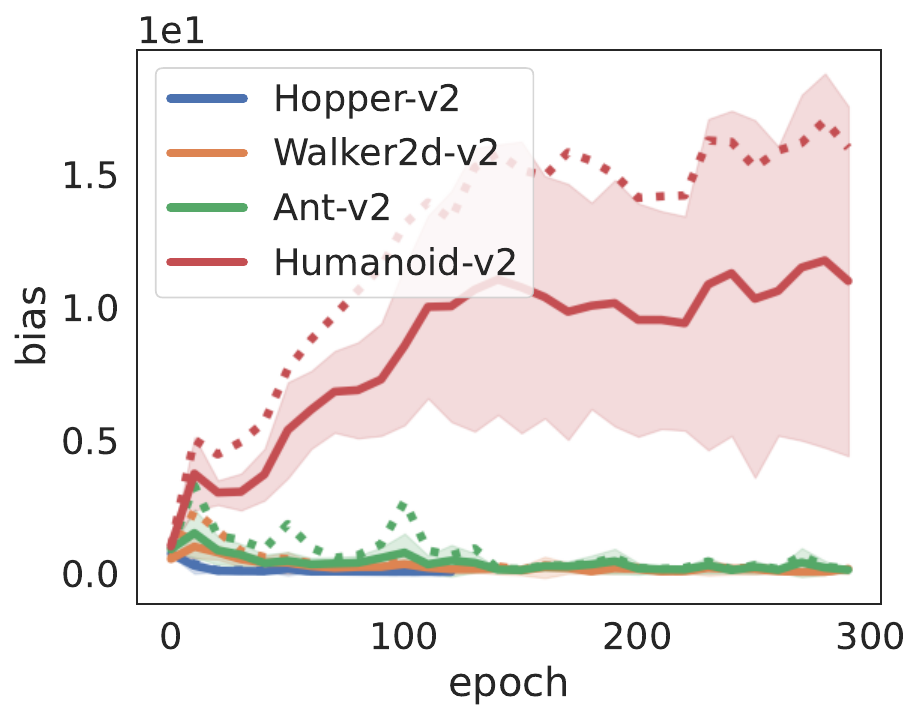}
\end{minipage}
\caption{
Results of policy amendments (left) and Q-function amendments (right) for all ten trials. 
The solid lines represent the post-amendment performance: return for the policy (left; i.e., $L_{\text{ret}}(\pi_{\theta, \mathbf{w}_*})$) and bias for the Q-function (right; i.e., $L_{\text{bias}}(Q_{\phi, \mathbf{w}_*})$). 
The dashed lines show the pre-amendment performance: return (left; i.e., $L_{\text{ret}}(\pi_{\theta})$) and bias (right; i.e., $L_{\text{bias}}(Q_{\phi})$). 
}
\label{fig:cleansing_results_average_case}
\end{figure*}
\begin{figure*}[h!]
\begin{minipage}{1.0\hsize}
\includegraphics[clip, width=0.245\hsize]{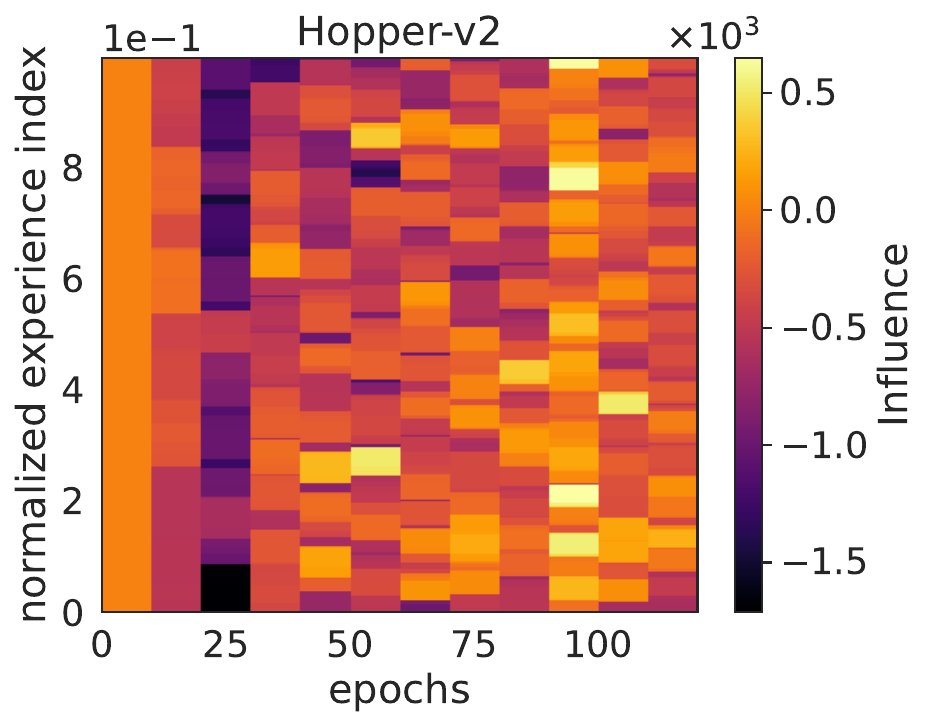}
\includegraphics[clip, width=0.245\hsize]{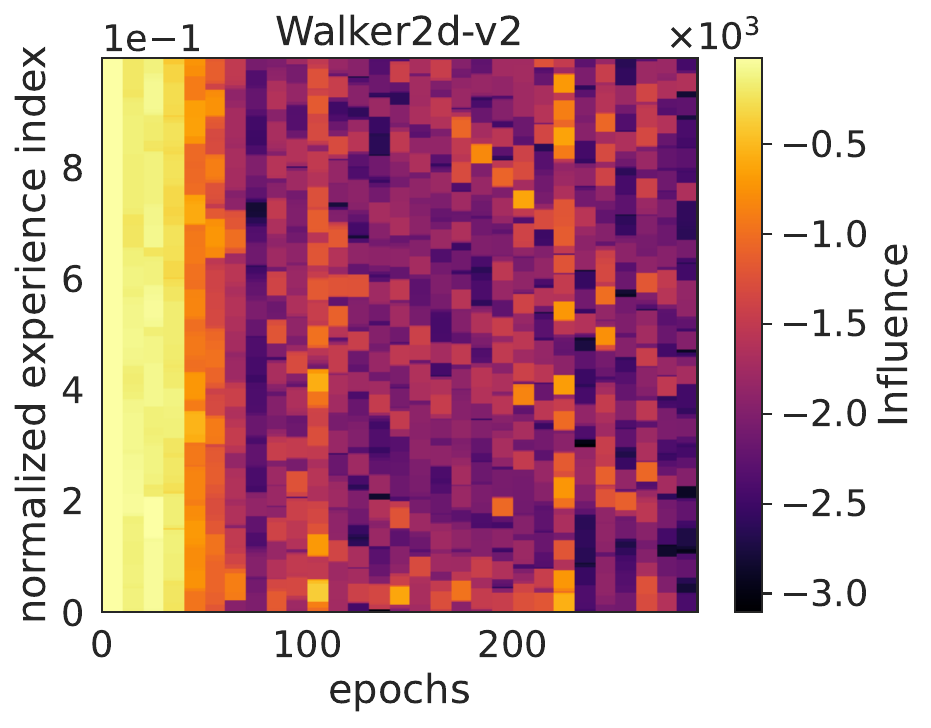}
\includegraphics[clip, width=0.245\hsize]{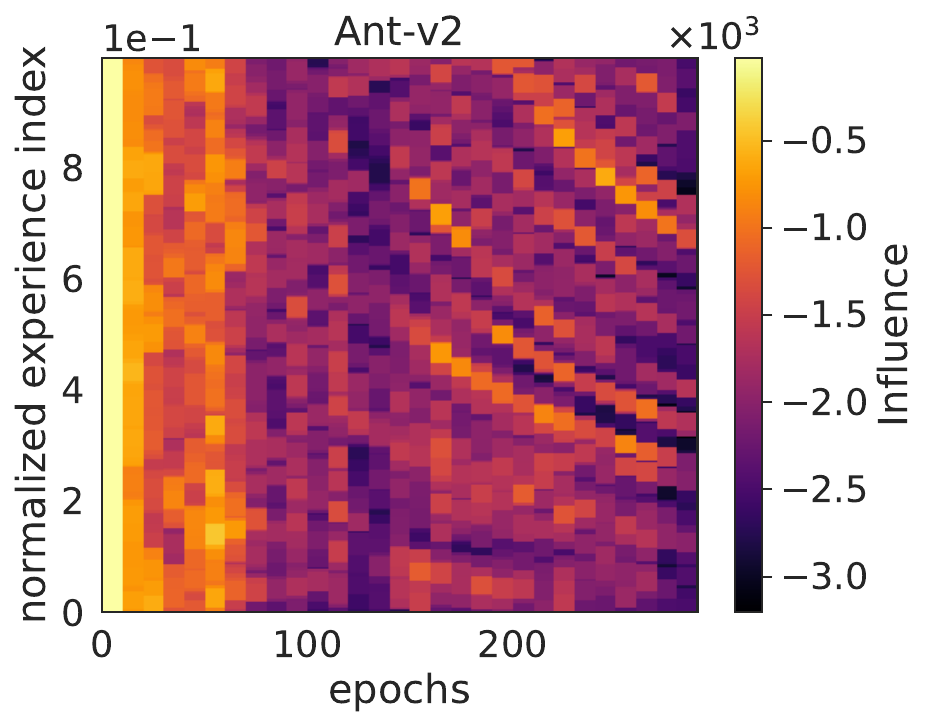}
\includegraphics[clip, width=0.245\hsize]{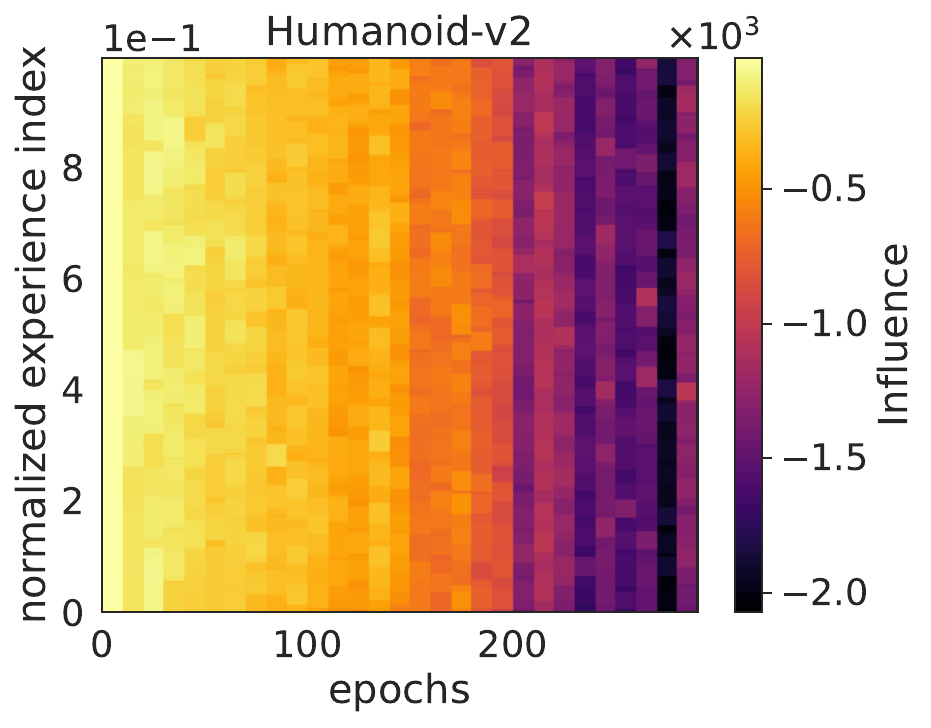}
\subcaption{Distribution of influence on return (Eq.~\ref{eq:influence_return}).}
\end{minipage}
\begin{minipage}{1.0\hsize}
\includegraphics[clip, width=0.245\hsize]{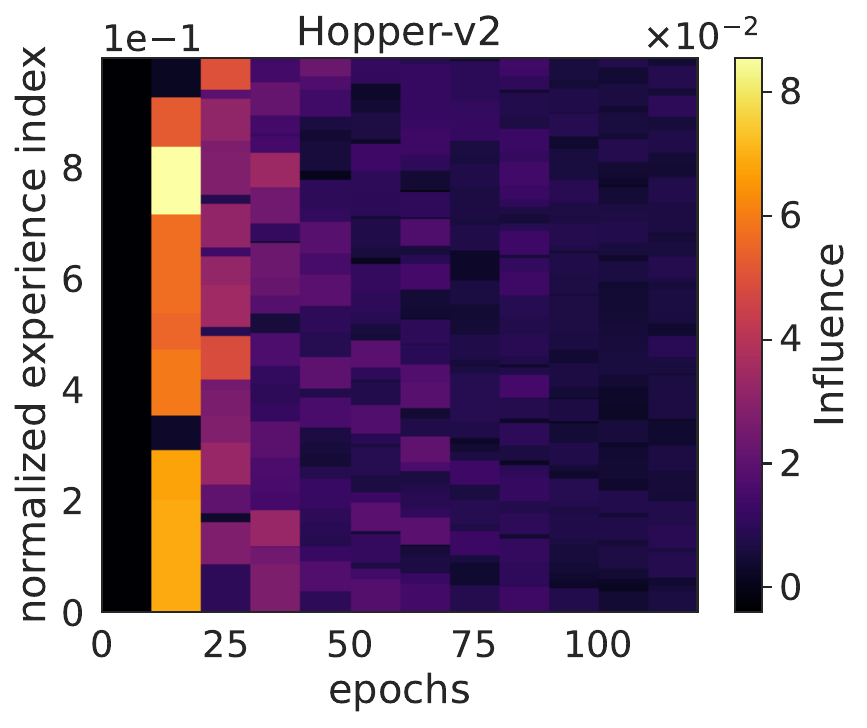}
\includegraphics[clip, width=0.245\hsize]{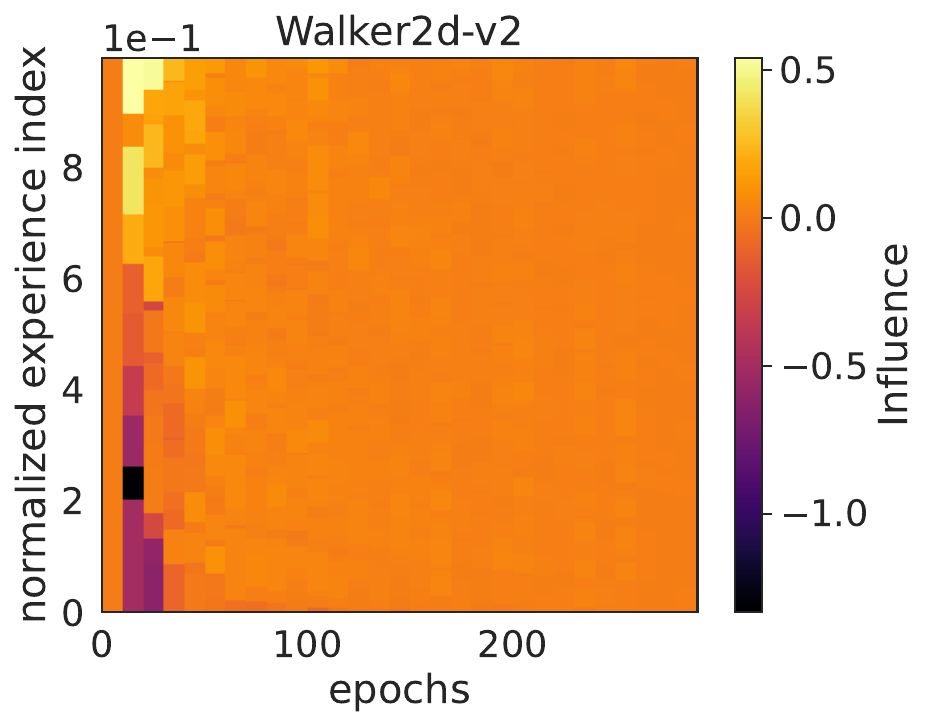}
\includegraphics[clip, width=0.245\hsize]{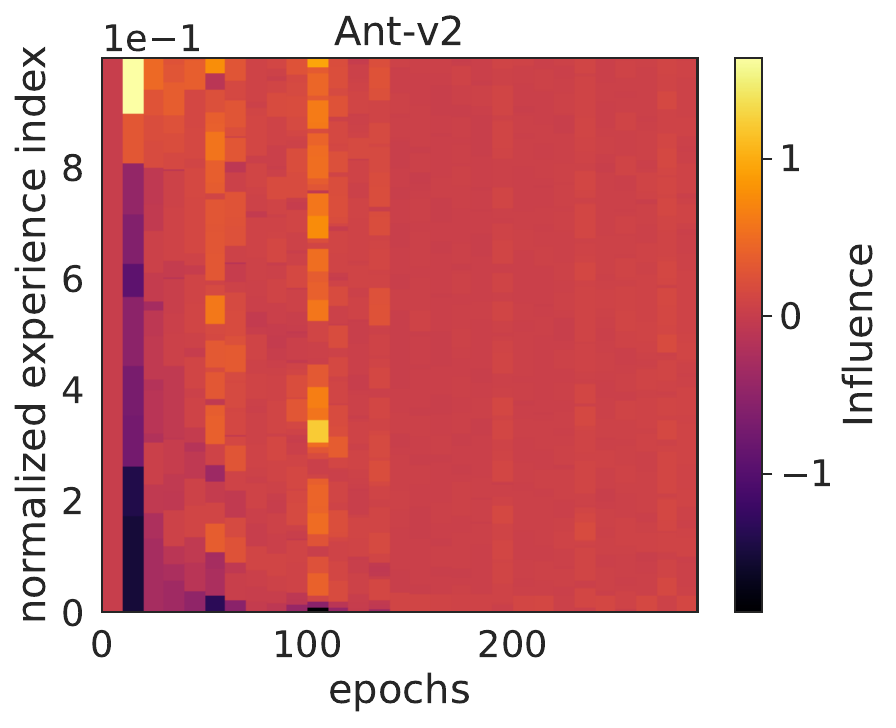}
\includegraphics[clip, width=0.245\hsize]{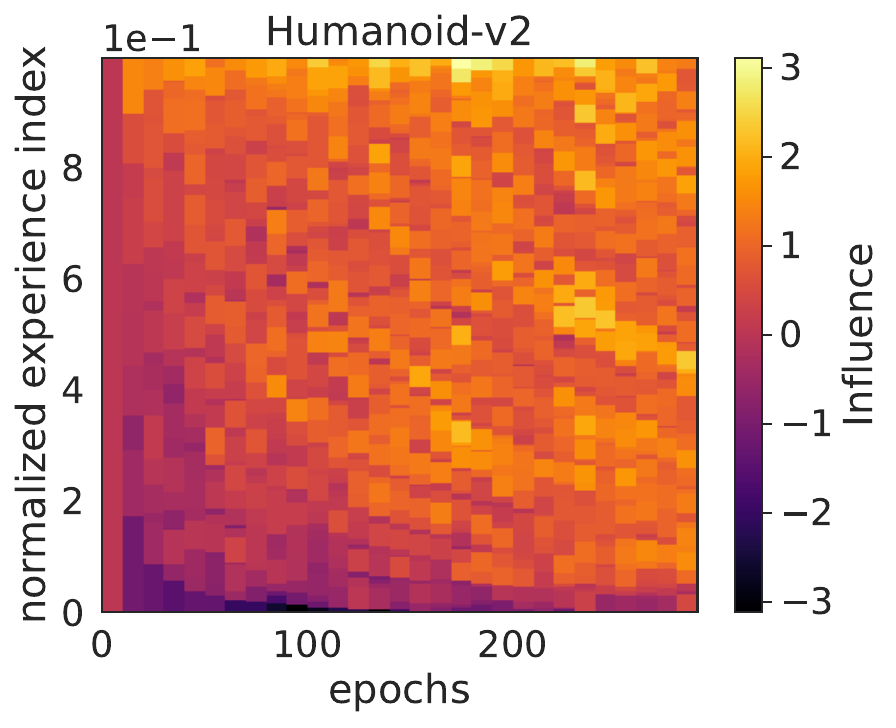}
\subcaption{Distribution of influence on Q-estimation bias (Eq.~\ref{eq:influence_bias}).}
\end{minipage}
%
\caption{
Distribution of the influence on return and on Q-estimation bias over all ten trials. 
The vertical axis represents the normalized experience index, ranging from 0.0 for the oldest experiences to 1.0 for the most recent experiences. 
The horizontal axis represents the number of epochs. 
The color bar represents the value of influence. 
}
\label{fig:distribution_of_bias_return}
\end{figure*}
\begin{figure*}[h!]
\begin{minipage}{1.0\hsize}
\centering
\includegraphics[clip, width=0.49\hsize]{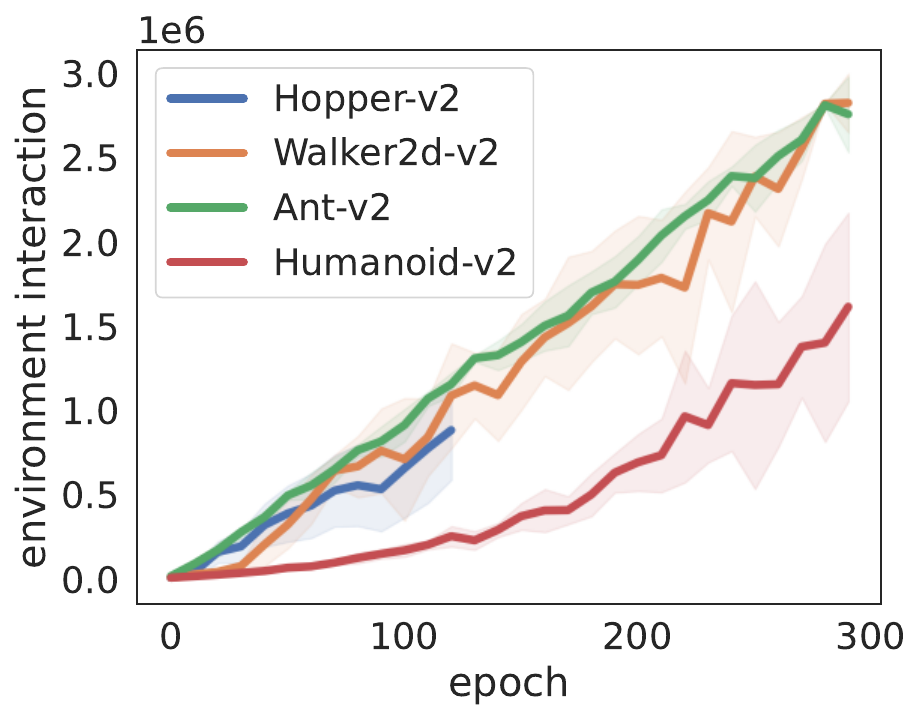}
\end{minipage}
\caption{
The number of environment interactions required for policy amendments in Section~\ref{sec:application}. 
Each line represents the mean of ten trials, and the shaded region represents one standard deviation around the mean. 
}
\label{fig:additional_environment_interaction_for_amendment}
\end{figure*}

\clearpage
\section{Analysis of the correlation between the influences of experiences}\label{app:correlation_of_influence}
In Sections~\ref{sec:experiments} and \ref{sec:application}, we estimated the influences of experiences on performance (e.g., return or Q-estimation bias). 
In Appendix~\ref{sec:analyzing_overlap_in_mask}, we discussed how the dropout rate of mask elements relates to the overlap between masks. 
The overlap between masks may affect the extent to which the influence of each experience can be isolated. 
In this section, we investigate the overlap among the estimated influences of experiences by analyzing two aspects: 
(i) the pairwise correlation of experience influences within each performance metric, and 
(ii) how the mask dropout rate affects this correlation\footnote{Note that we analyze correlations within, rather than across, performance metrics.}. 

We calculate the correlation between the experience influences for each performance metric used in Sections~\ref{sec:experiments} and \ref{sec:application}. 
In these sections, we estimated the influences of experiences on policy evaluation ($L_{\text{pe}, i}$), policy improvement ($L_{\text{pi}, i}$), return ($L_{\text{ret}}$), and Q-estimation bias ($L_{\text{bias}}$). 
We treat the influences of experiences on each metric at each epoch as a vector of random variables, where each element represents the influence of a single experience. 
We calculate the Pearson correlation between these elements. 
The influence values observed in the ten learning trials are used as samples. 
In the following discussion, we focus on the average value of the correlations between the pairs of vector elements. 

\textbf{(i) The correlation between the influences of experiences.} 
The correlation between the influences of experiences is shown in Figure~\ref{fig:correlation_in_environments}. 
The figure shows that the correlation tends to approach zero as the number of epochs increases. 
For return and bias, the correlation converges to zero early in the learning process, regardless of the environments. 
For policy evaluation and improvement, the degree of correlation convergence varies significantly across environments. 
\begin{figure*}[h!]
\begin{minipage}{1.0\hsize}
\includegraphics[clip, width=0.245\hsize]{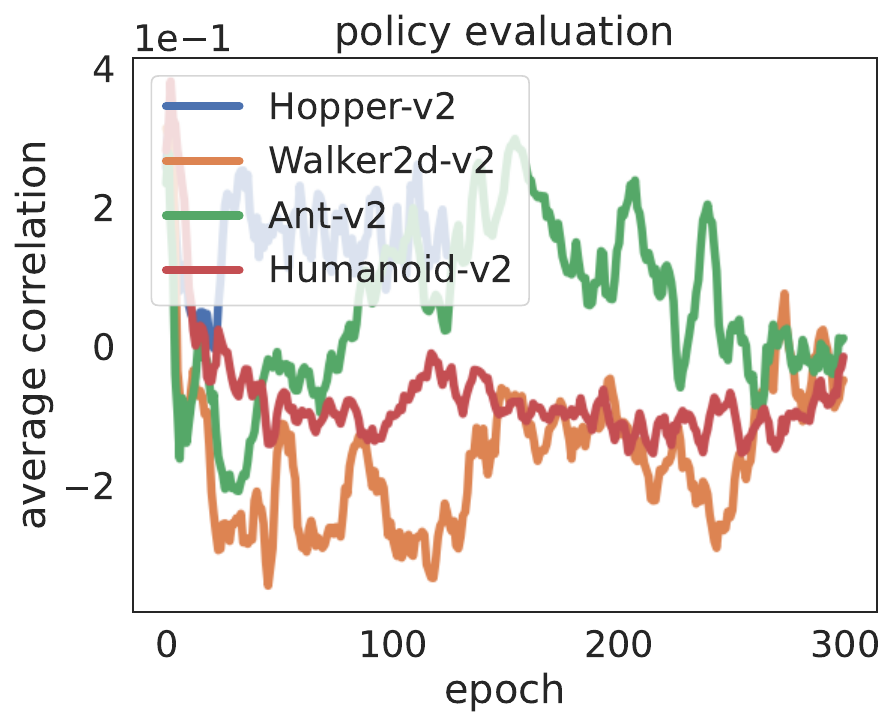}
\includegraphics[clip, width=0.245\hsize]{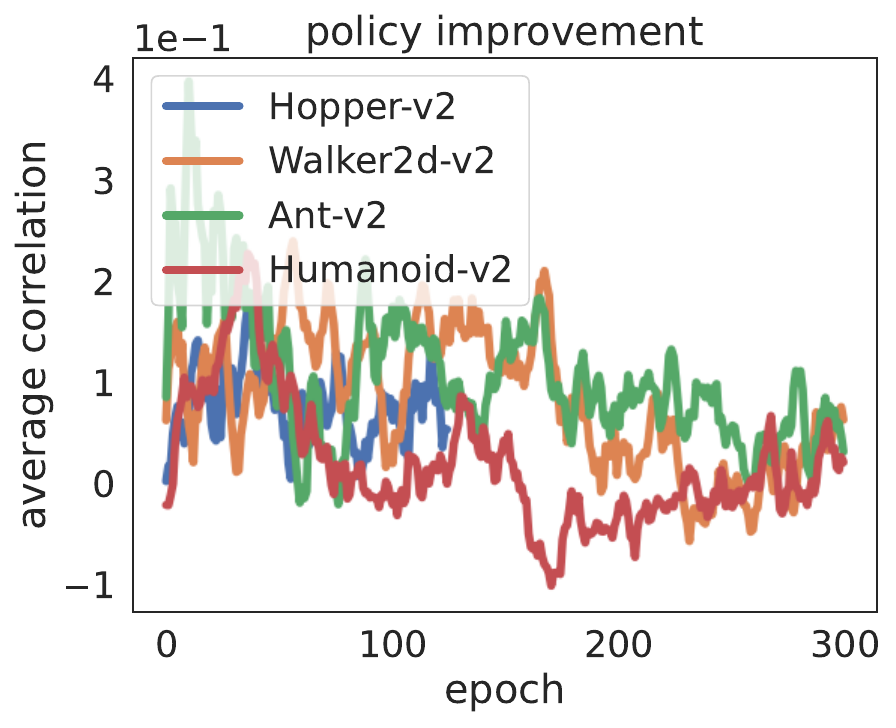}
\includegraphics[clip, width=0.245\hsize]{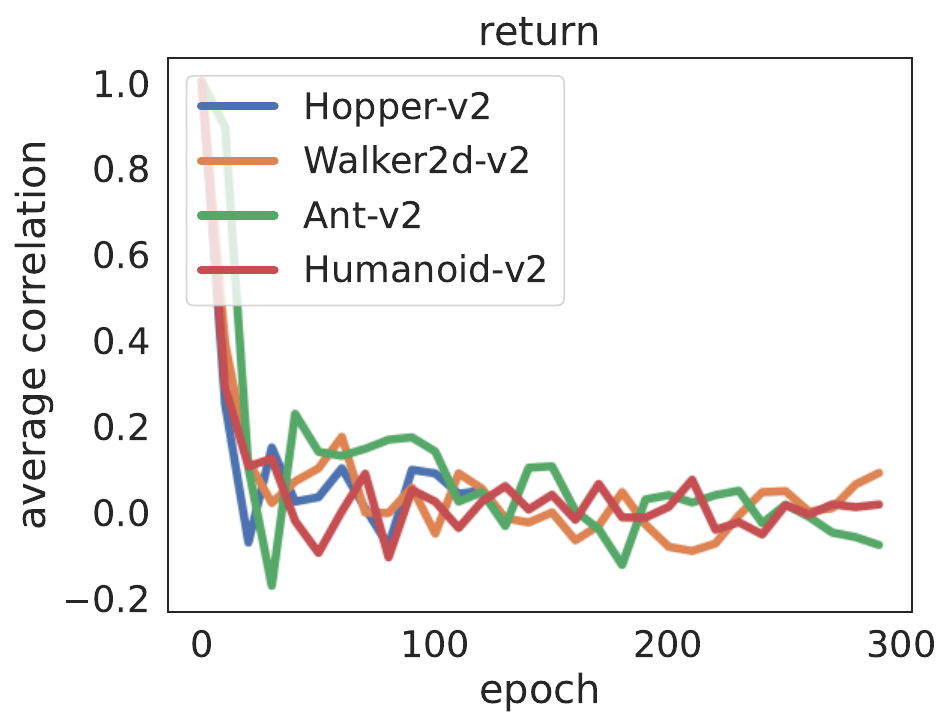}
\includegraphics[clip, width=0.245\hsize]{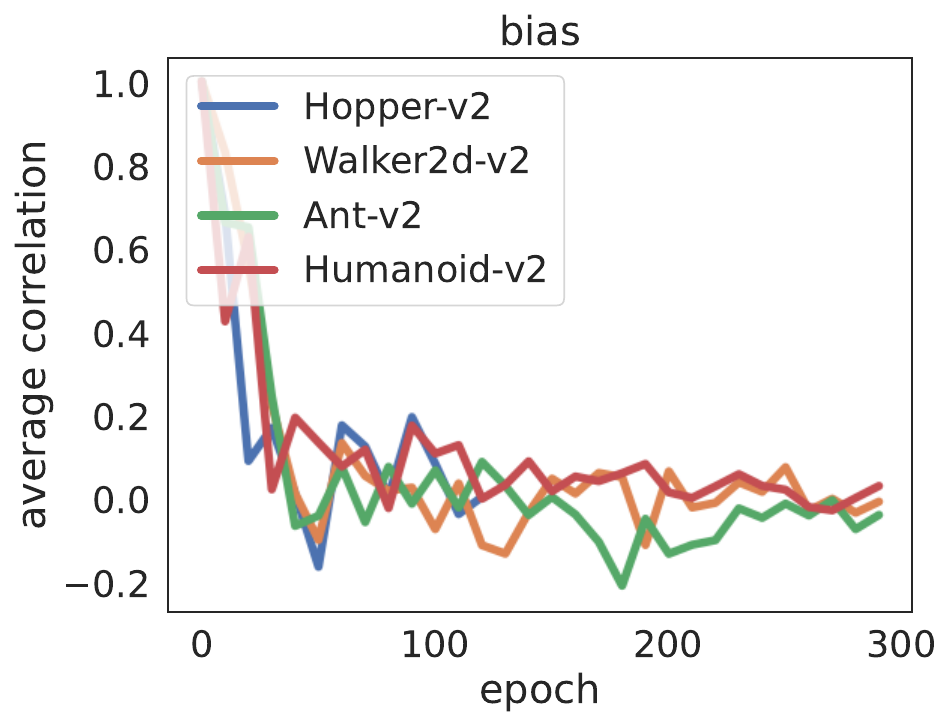}
\end{minipage}
\caption{
Correlation between the influences of experiences on policy evaluation ($L_{\text{pe}, i}$), policy improvement ($L_{\text{pi}, i}$), return ($L_{\text{ret}}$), and Q-estimation bias ($L_{\text{bias}}$) for each epoch in each environment. 
The vertical axis represents the average correlation of experience influences, ranging from -1.0 to 1.0. 
The horizontal axis represents the number of epochs. 
}
\label{fig:correlation_in_environments}
\end{figure*}

\textbf{(ii) The relationship between the correlation and the dropout rate.} 
We evaluate how varying the dropout rate of masks affects the correlation between experience influences. 
Specifically, we evaluated the correlations using PIToD with four different dropout rates:\\
\textbf{DR0.5:} PIToD with a dropout rate of 0.5, which is the setting used in the main experiments of this paper.\\ 
\textbf{DR0.25:} PIToD with a dropout rate of 0.25.\\ 
\textbf{DR0.1:} PIToD with a dropout rate of 0.1.\\ 
\textbf{DR0.05:} PIToD with a dropout rate of 0.05.\\ 
%
The correlations for these cases in the Hopper environment are shown in Figure~\ref{fig:correlation_in_method}. 
The results imply that the impact of the dropout rate on the correlation depends significantly on the specific performance metric. 
For instance, we do not observe a significant impact of the dropout rate in policy evaluation or policy improvement. 
In contrast, for return, we observe that the correlation increases as the dropout rate decreases. 
\begin{figure*}[h!]
\begin{minipage}{1.0\hsize}
\includegraphics[clip, width=0.245\hsize]{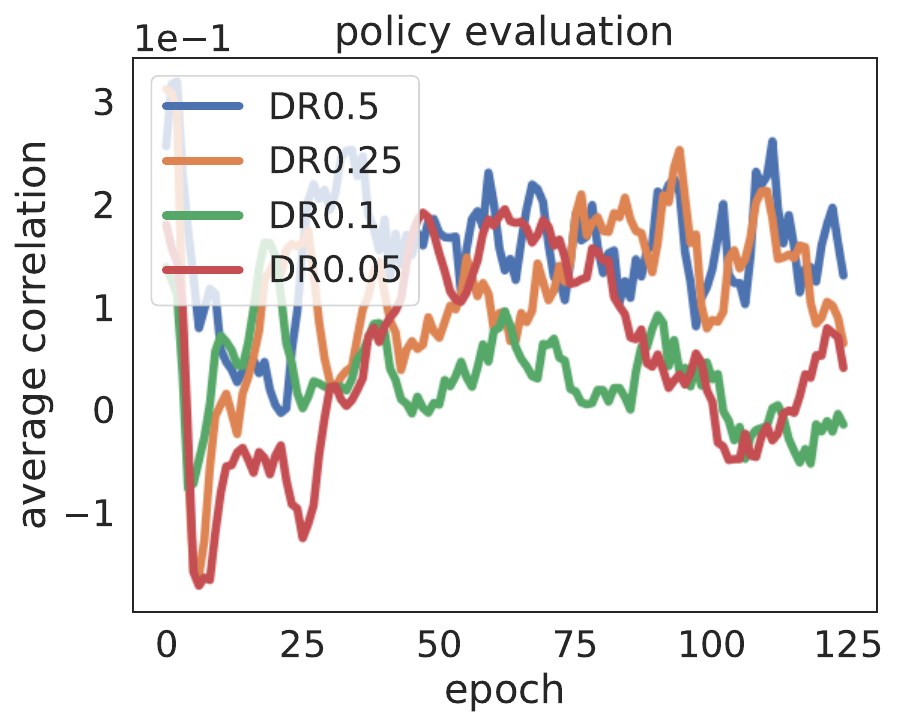}
\includegraphics[clip, width=0.245\hsize]{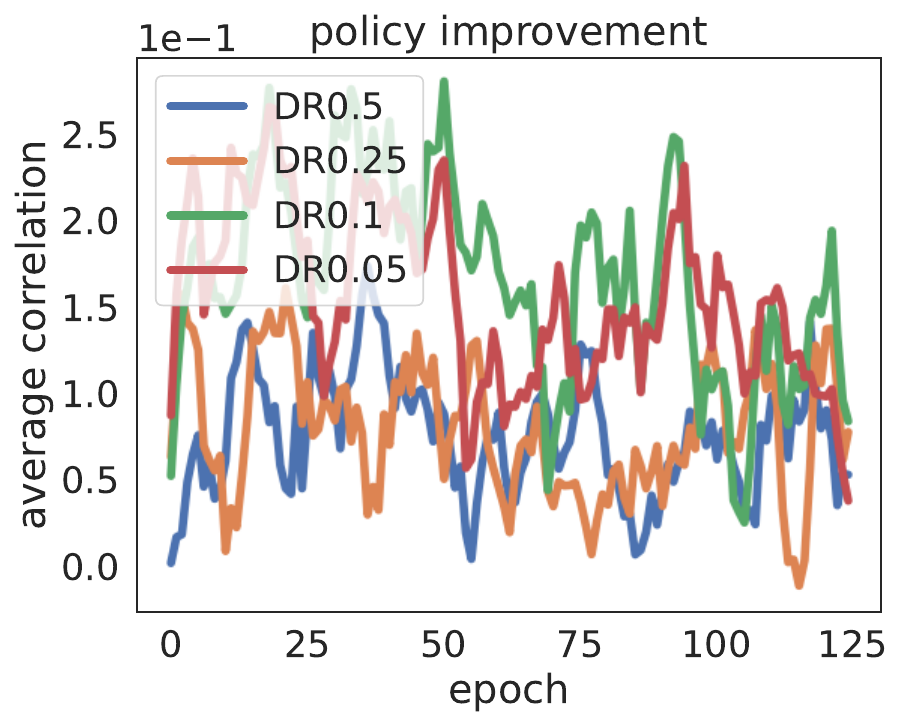}
\includegraphics[clip, width=0.245\hsize]{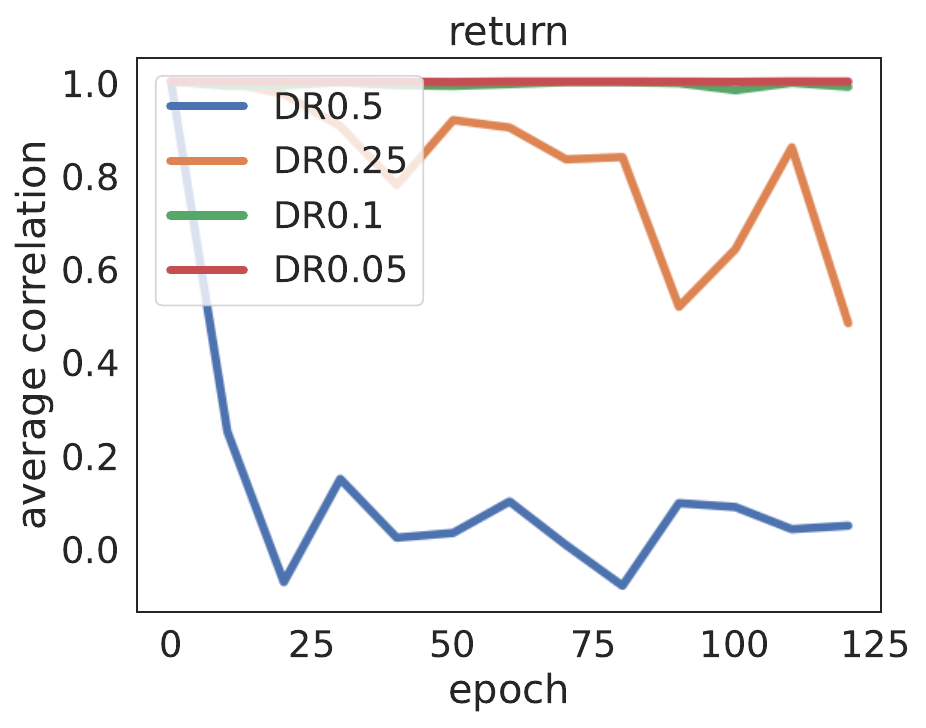}
\includegraphics[clip, width=0.245\hsize]{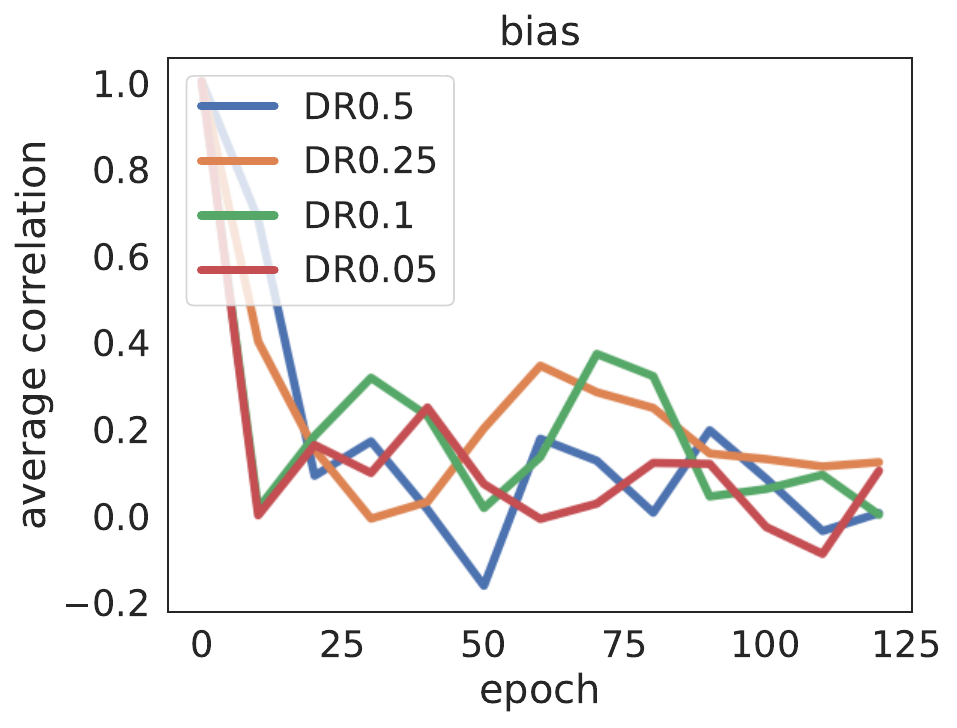}
\end{minipage}
\caption{
Correlation between the influences of experiences at each epoch in the Hopper environment. 
The vertical axis represents the average correlation of experience influences. 
The horizontal axis represents the number of learning epochs. 
Each label in the legend corresponds to a dropout rate for masks. For example, ``DR0.5'' means a dropout rate of $0.5$ (half of the elements in each mask are set to zero), and ``DR0.1'' means a dropout rate of $0.1$ ($10\%$ of the elements in each mask are set to zero). 
}
\label{fig:correlation_in_method}
\end{figure*}

\clearpage
\section{Application of LOO: amending policies and Q-functions by deleting negatively influential experiences}\label{sec:loo_prelim}
In Section~\ref{sec:application}, we amended the agent’s policies and Q-functions using PIToD. 
In this section, we amend policies and Q-functions using LOO introduced in Section~\ref{sec:problem_of_loo}.

To complete the amendment using LOO within a practical time frame, we propose a simplified implementation of LOO. 
This implementation has the following characteristics:\\
1. As in the practical implementation of PIToD (Section~\ref{sec:practical_implementation}), we do not estimate the influence of each individual experience. Instead, we group 5000 experiences together and estimate the influence of each group.\\
2. The number of policy iteration steps during retraining is the same for all groups. 
We varied the number of steps from 5000 to 75000 in a pilot run and selected 75000 because it yielded the best overall performance. 

To keep runtime reasonable, we applied the amendments once at epoch 50. 
We ran ten trials and collected the results in the Hopper environment. 

Figure~\ref{fig:cleansing_results_policy_loo} shows the results of the amendments. 
The amendments using LOO achieve significantly higher returns than the NoAmendment baseline (i.e., no amendments are applied). 
The return with the LOO amendments is comparable to the return with PIToD amendments.
\begin{figure*}[h!]
\begin{minipage}{1.0\hsize}
\centering
\includegraphics[clip, width=0.49\hsize]{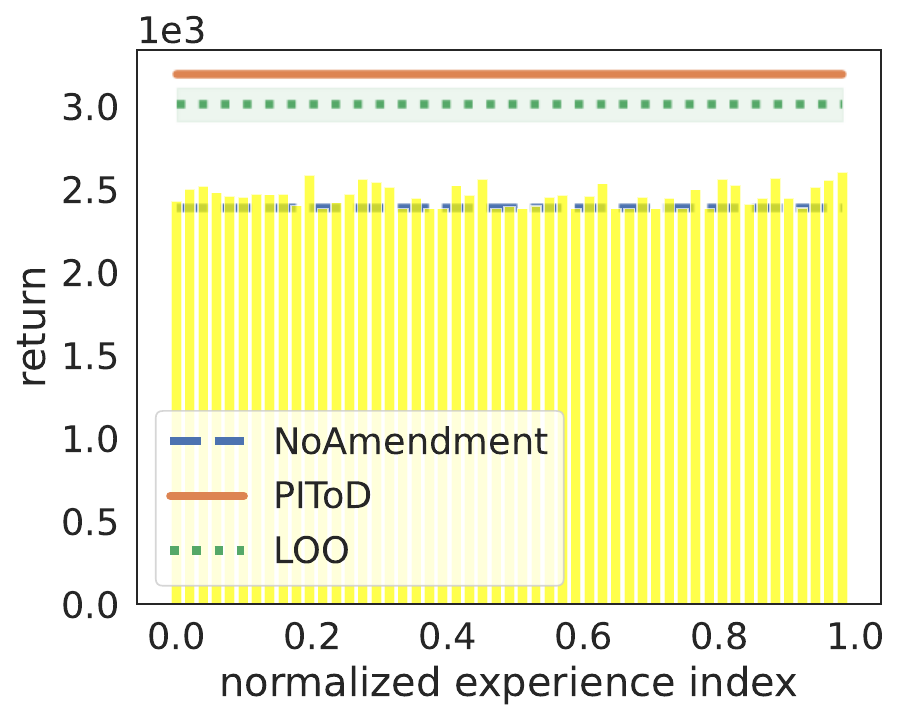}
\end{minipage}
\caption{
Results of policy amendments applied at epoch 50 in the Hopper environment. 
The vertical axis shows the empirical return, and the horizontal axis shows the normalized experience index.
The dashed line (NoAmendment) represents the return before any amendments. 
The solid line (PIToD) represents the return after applying the PIToD amendments,
and the dotted line (LOO) represents the return after applying the LOO amendments. 
Each yellow bar represents the return when LOO deletes the experience group corresponding to that normalized experience index. 
Results are averaged over the ten trials. The shaded region around the dashed line denotes one standard deviation.
}
\label{fig:cleansing_results_policy_loo}
\end{figure*}

\clearpage
\section{Amending policies and Q-functions in DM control environments with adversarial experiences}\label{app:adversarial_dmc}
In Section~\ref{sec:application}, we applied PIToD to amend policies and Q-functions in the MuJoCo~\citep{todorov2012mujoco} environments. 

In this section, we apply PIToD to amend policies and Q-functions in DM control~\citep{tunyasuvunakool2020} environments with adversarial experiences. 
We focus on the DM control environments: finger-turn\_hard, hopper-stand, hopper-hop, fish-swim, cheetah-run, quadruped-run, humanoid-run, and humanoid-stand. 
In these environments, we introduce adversarial experiences. 
An adversarial experience contains an adversarial reward $r'$, which is a reversed and magnified version of the original reward $r$: $r' = -100 \cdot r$. 
These adversarial experiences are designed to (i) prevent the agent from maximizing the original reward and (ii) have greater influence than other non-adversarial experiences stored in the replay buffer~\footnote{Learning with these adversarial experiences can be considered learning under a data-poisoning attack~\citep{gong2024baffle,cui2024badrl}.}. 
At 150 epochs (i.e., in the middle of training), the RL agent encounters 5000 adversarial experiences. 
In these environments, we amend policies and Q-functions as in Section~\ref{sec:application}. 

The results of the policy and Q-function amendments (Figures~\ref{fig:cleansing_results_policy_dmc} and \ref{fig:cleansing_results_qfunction_dmc}) show that performance is improved by the amendments. 
The policy amendment results (Figure~\ref{fig:cleansing_results_policy_dmc}) show that returns are improved, particularly in fish-swim. 
Additionally, the Q-function amendment results (Figure~\ref{fig:cleansing_results_qfunction_dmc}) show that the Q-estimation bias is significantly reduced in finger-turn\_hard, hopper-stand, hopper-hop, fish-swim, cheetah-run, and quadruped-run. 
\begin{figure*}[h!]
\begin{minipage}{1.0\hsize}
\includegraphics[clip, width=0.49\hsize]{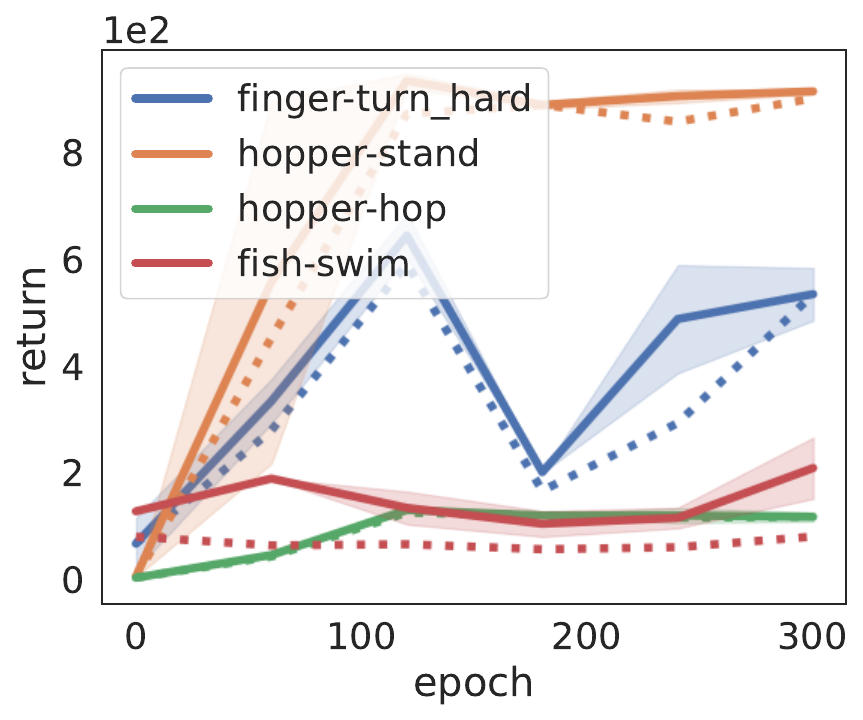}
\includegraphics[clip, width=0.49\hsize]{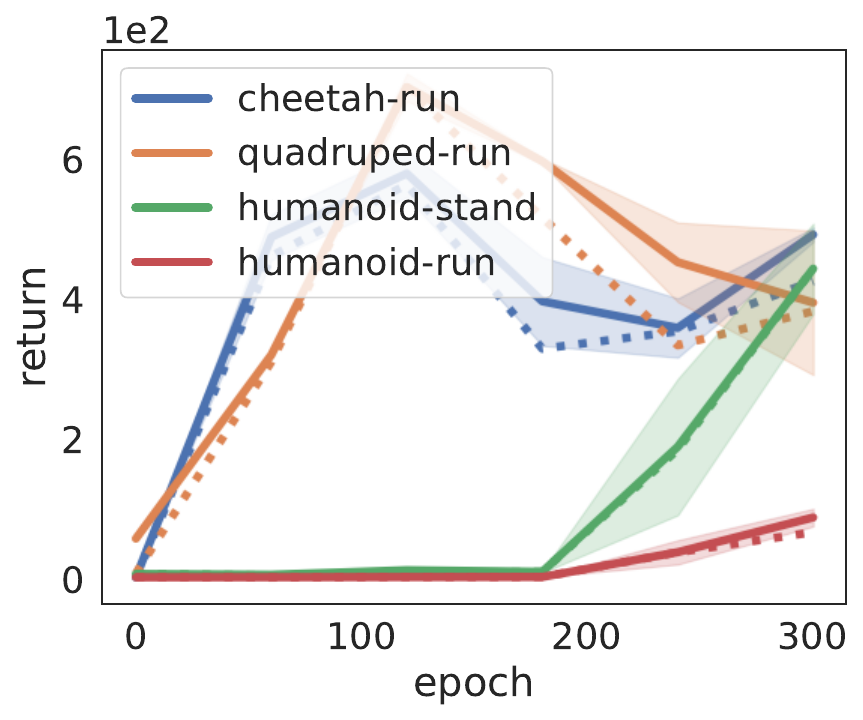}
\end{minipage}
\caption{
Results of policy amendments in DM control environments with adversarial experiences. 
The solid lines represent the post-amendment return for the policy (i.e., $L_{\text{ret}}(\pi_{\theta, \mathbf{w}_*})$). 
The dashed lines show the pre-amendment return (i.e., $L_{\text{ret}}(\pi_{\theta})$). 
These results are averaged over ten trials. 
The shaded regions around the solid lines represent one standard deviation around the mean. 
}
\label{fig:cleansing_results_policy_dmc}
\end{figure*}
\begin{figure*}[h!]
\begin{minipage}{1.0\hsize}
\includegraphics[clip, width=0.49\hsize]{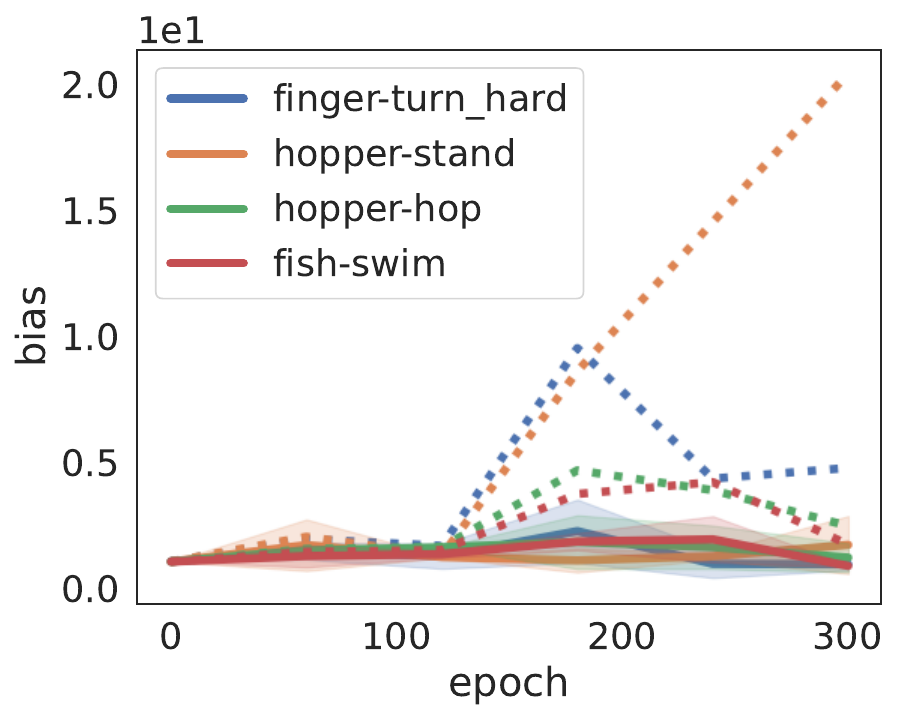}
\includegraphics[clip, width=0.49\hsize]{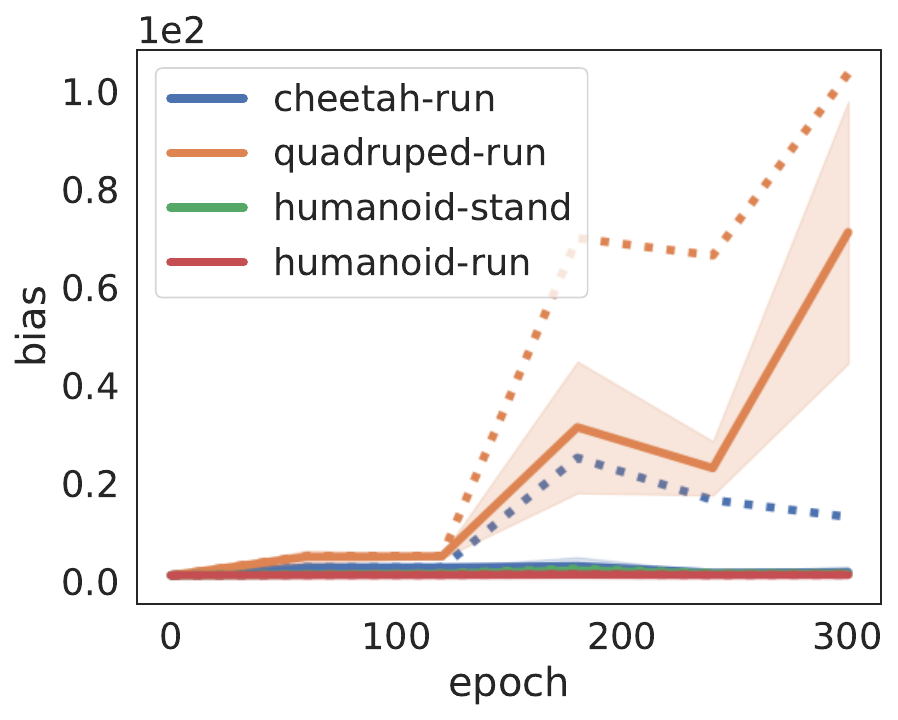}
\end{minipage}
\caption{
Results of Q-function amendments in DM control environments with adversarial experiences. 
The solid lines represent the post-amendment bias for the Q-function (i.e., $L_{\text{bias}}(Q_{\phi, \mathbf{w}_*})$). 
The dashed lines show the pre-amendment bias (i.e., $L_{\text{bias}}(Q_{\phi})$). 
These results are averaged over ten trials. 
The shaded regions around the solid lines represent one standard deviation around the mean. 
}
\label{fig:cleansing_results_qfunction_dmc}
\end{figure*}

\textbf{Can PIToD identify adversarial experiences?} 
PIToD identifies adversarial experiences as (i) strongly influential experiences for policy evaluation and (ii) positively influential experiences for Q-estimation bias. 
\textbf{Policy evaluation:} 
Figure~\ref{fig:distribution_of_self_influences_pe_dmc} shows the distribution of influences on policy evaluation. 
We observe that adversarial experiences have a strong influence (highlighted in lighter colors), except in humanoid-run. 
\textbf{Q-estimation bias:} 
Figure~\ref{fig:distribution_of_bias_dmc} shows the distribution of influences on Q-estimation bias. 
Interestingly, we observe that adversarial experiences have a strong positive influence (highlighted in lighter colors). 
Namely, these adversarial experiences contribute to reducing Q-estimation bias. 
However, after introducing adversarial experiences (i.e., after epoch 150), we also observe experiences with a negative influence. 
We hypothesize that adversarial experiences hinder the learning from other experiences. 
\begin{figure*}[h!]
\begin{minipage}{1.0\hsize}
\includegraphics[clip, width=0.245\hsize]{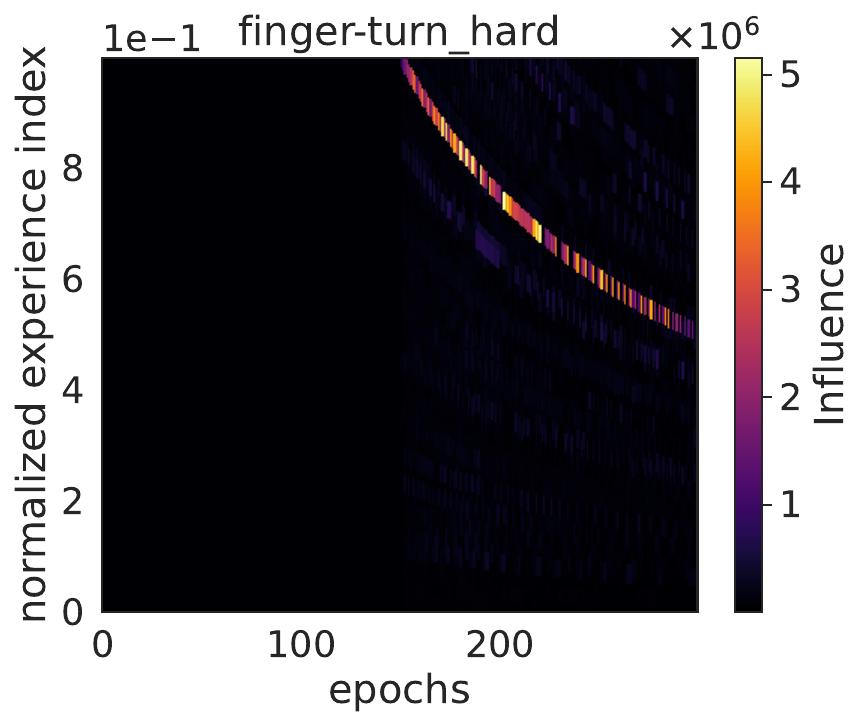}
\includegraphics[clip, width=0.245\hsize]{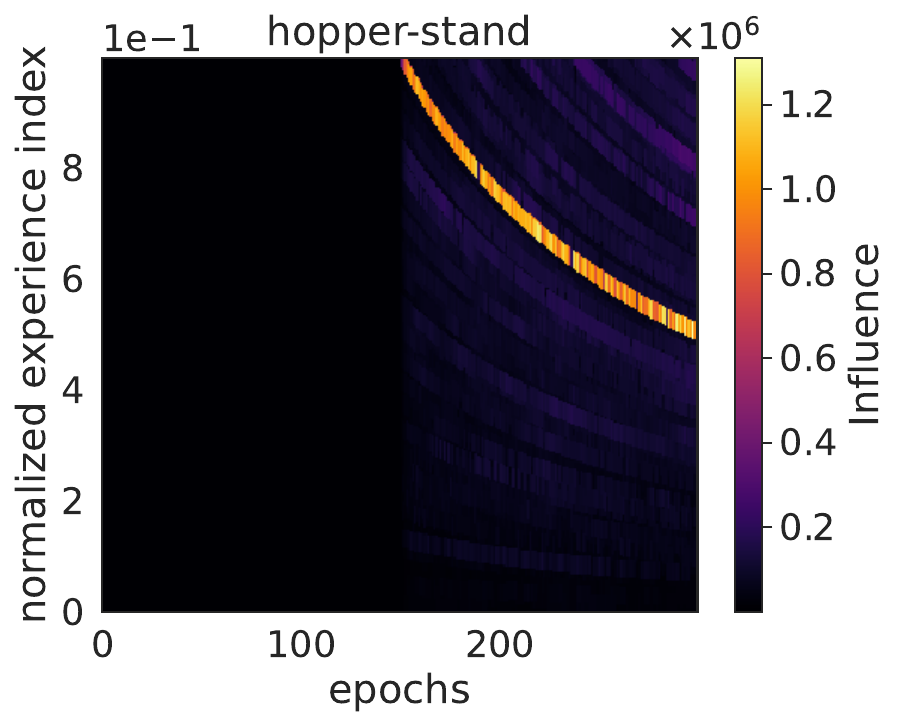}
\includegraphics[clip, width=0.245\hsize]{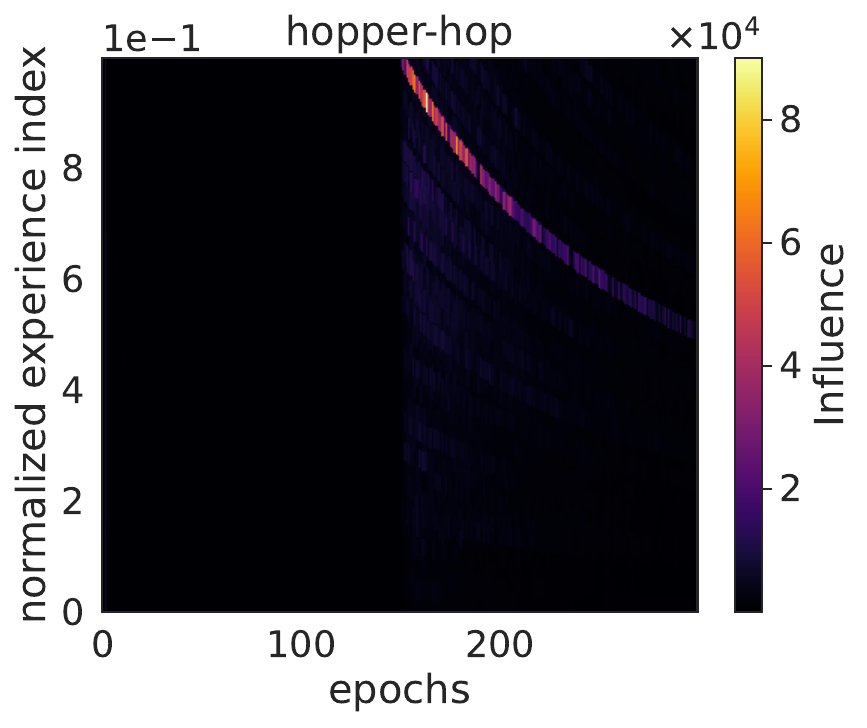}
\includegraphics[clip, width=0.245\hsize]{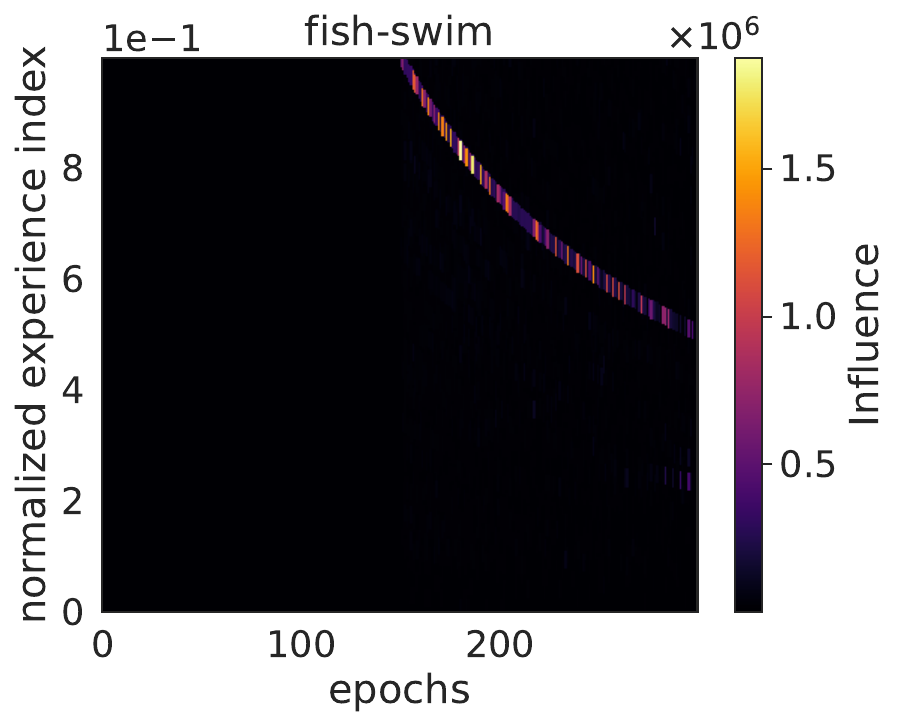}
\includegraphics[clip, width=0.245\hsize]{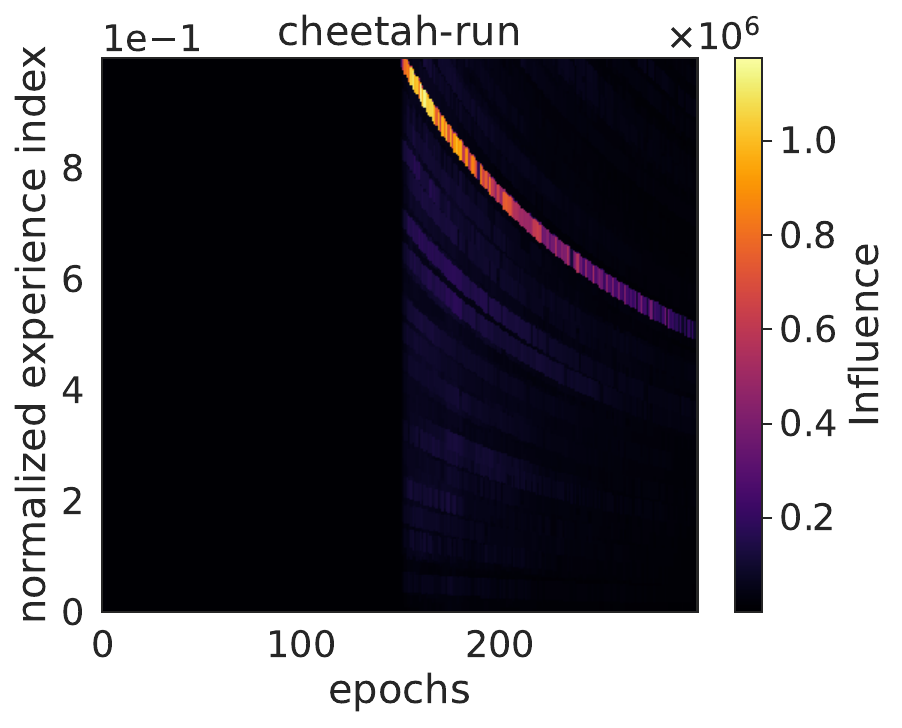}
\includegraphics[clip, width=0.245\hsize]{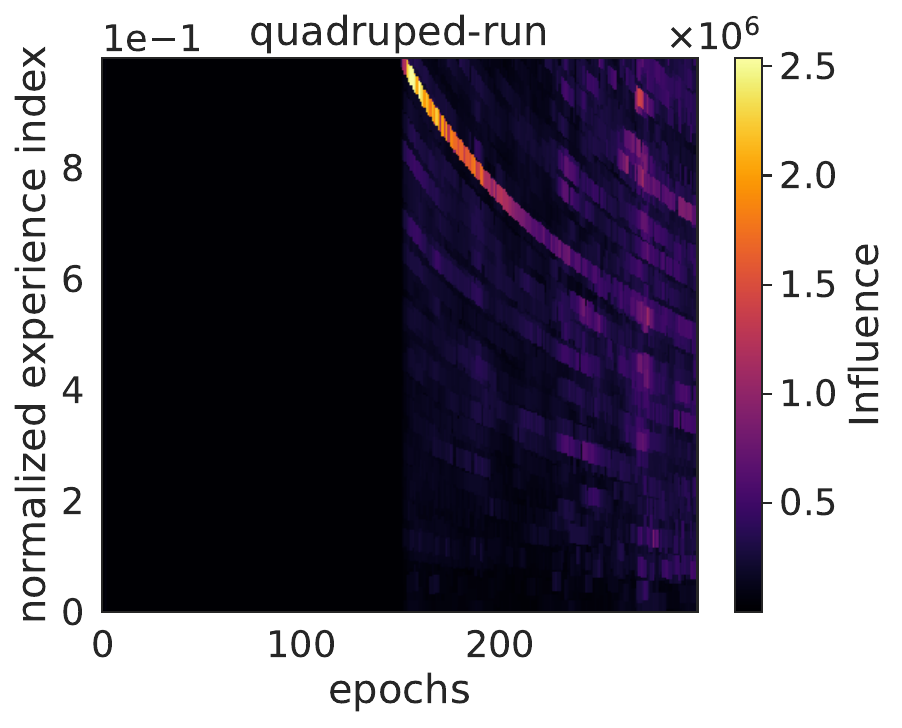}
\includegraphics[clip, width=0.245\hsize]{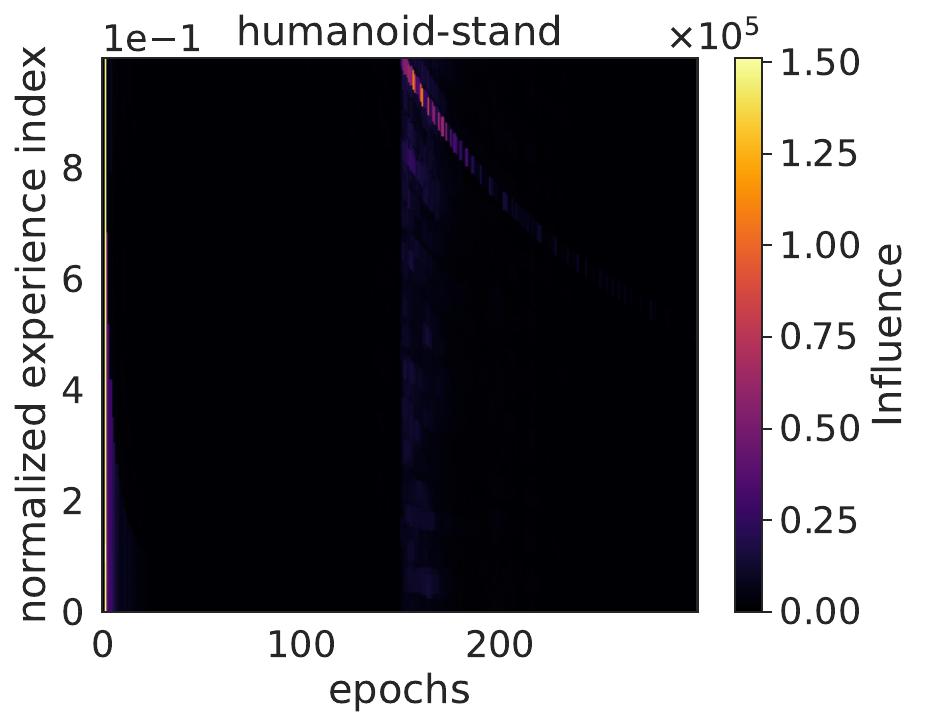}
\includegraphics[clip, width=0.245\hsize]{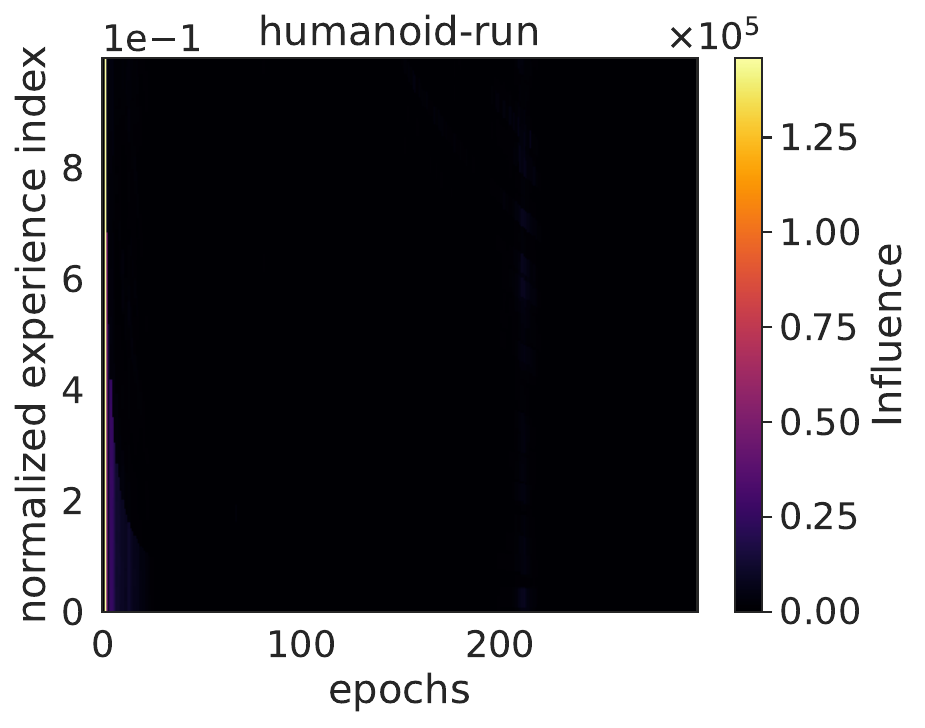}
\end{minipage}
\caption{
Distribution of experience influence on policy evaluation (Eq.~\ref{eq:self_infl_q_func}) in DM control environments with adversarial experiences.
}
\label{fig:distribution_of_self_influences_pe_dmc}
\end{figure*}
\begin{figure*}[h!]
\begin{minipage}{1.0\hsize}
\includegraphics[clip, width=0.245\hsize]{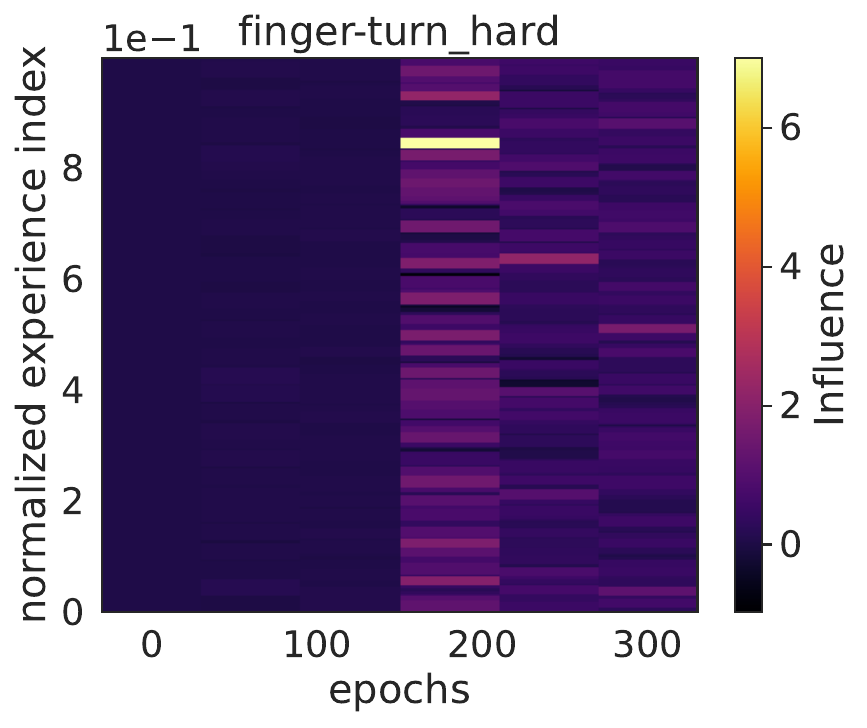}
\includegraphics[clip, width=0.245\hsize]{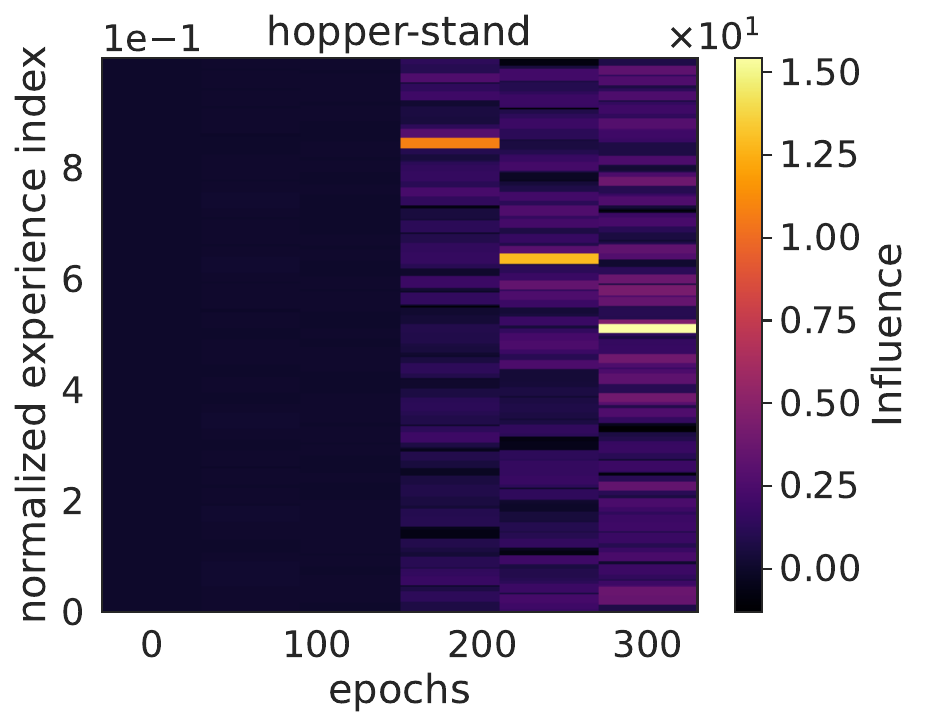}
\includegraphics[clip, width=0.245\hsize]{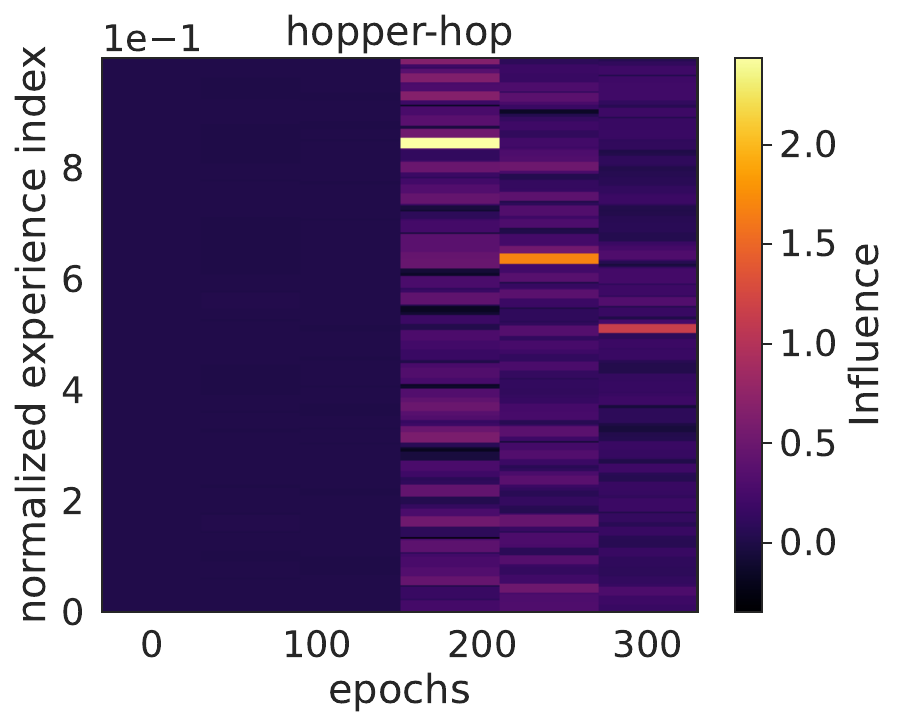}
\includegraphics[clip, width=0.245\hsize]{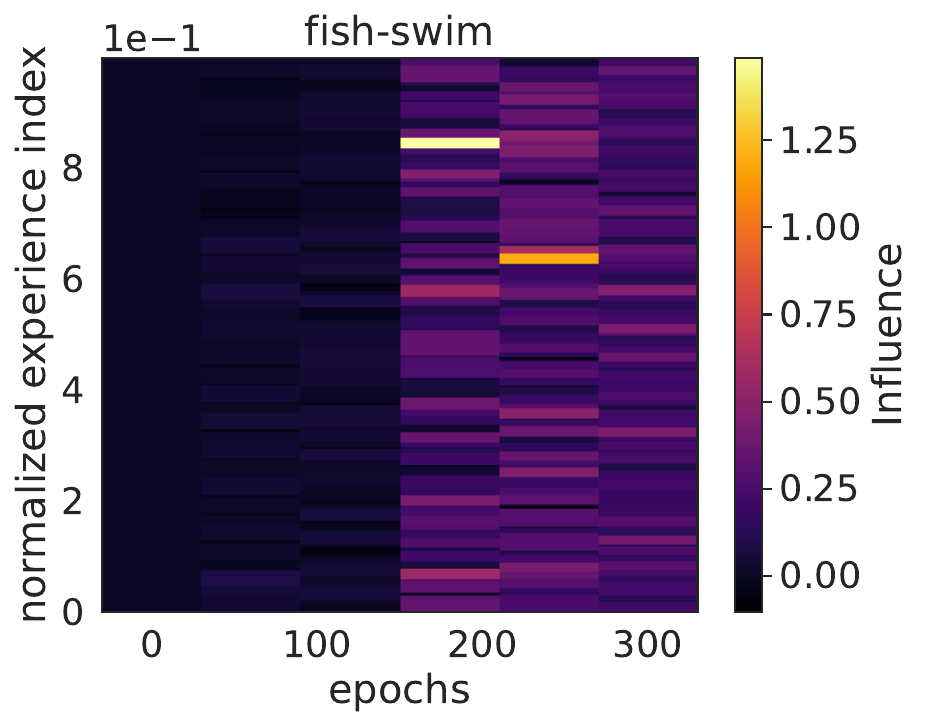}
\includegraphics[clip, width=0.245\hsize]{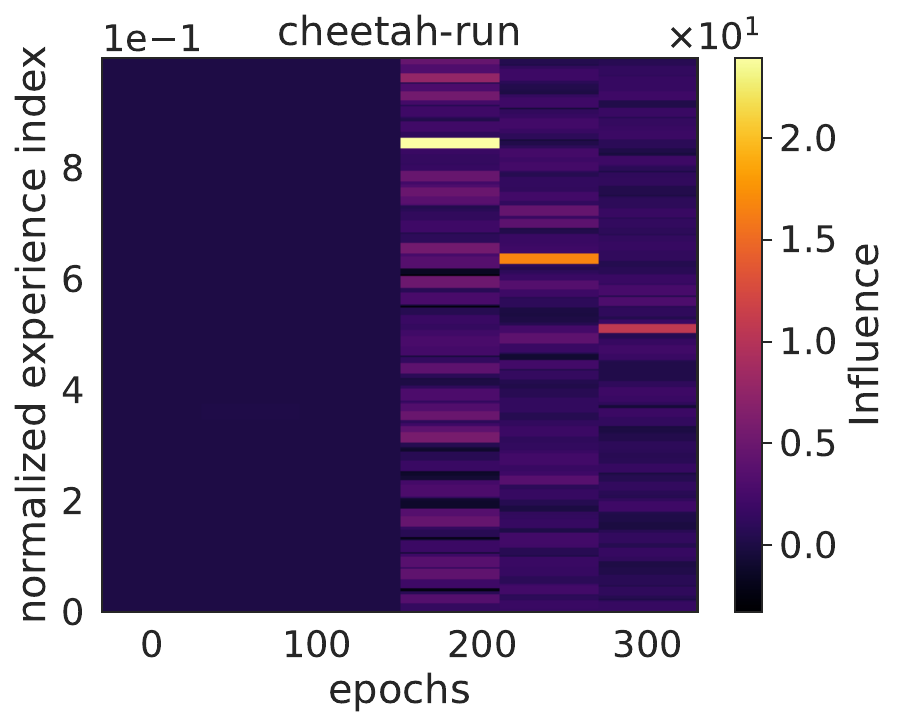}
\includegraphics[clip, width=0.245\hsize]{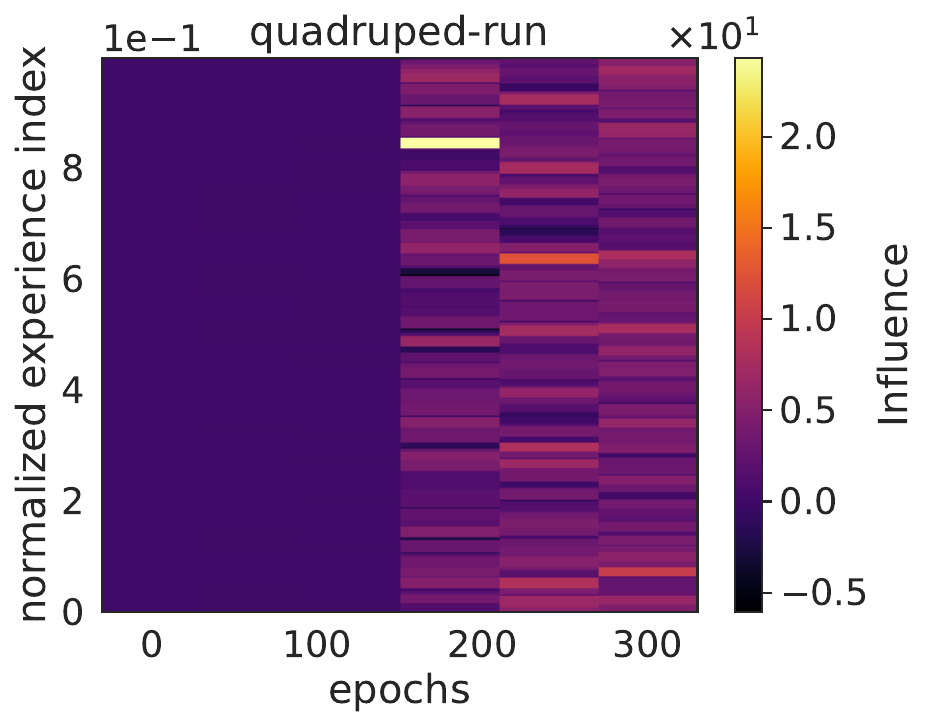}
\includegraphics[clip, width=0.245\hsize]{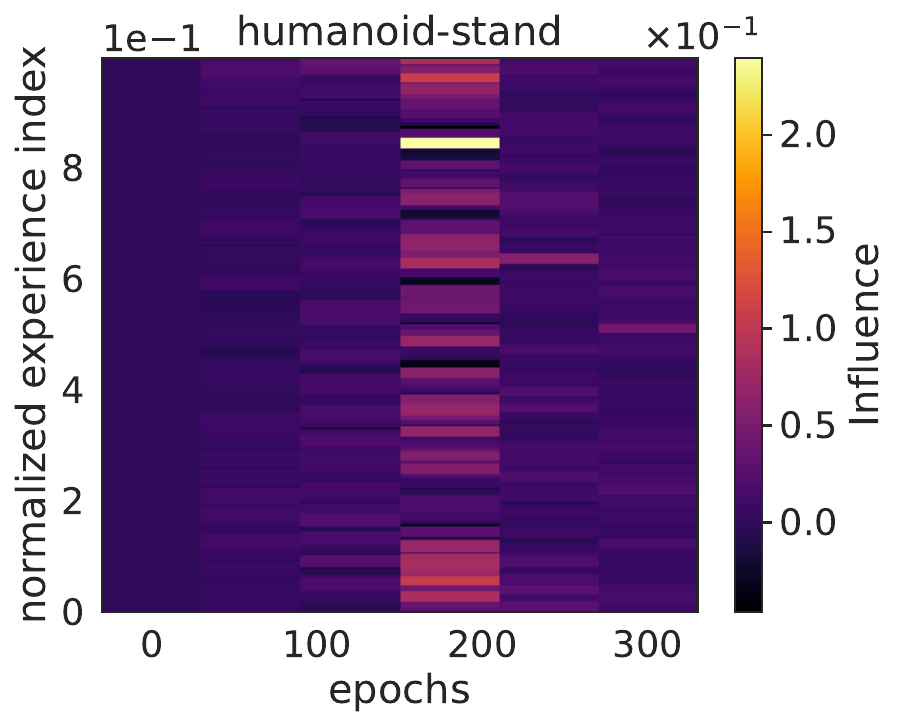}
\includegraphics[clip, width=0.245\hsize]{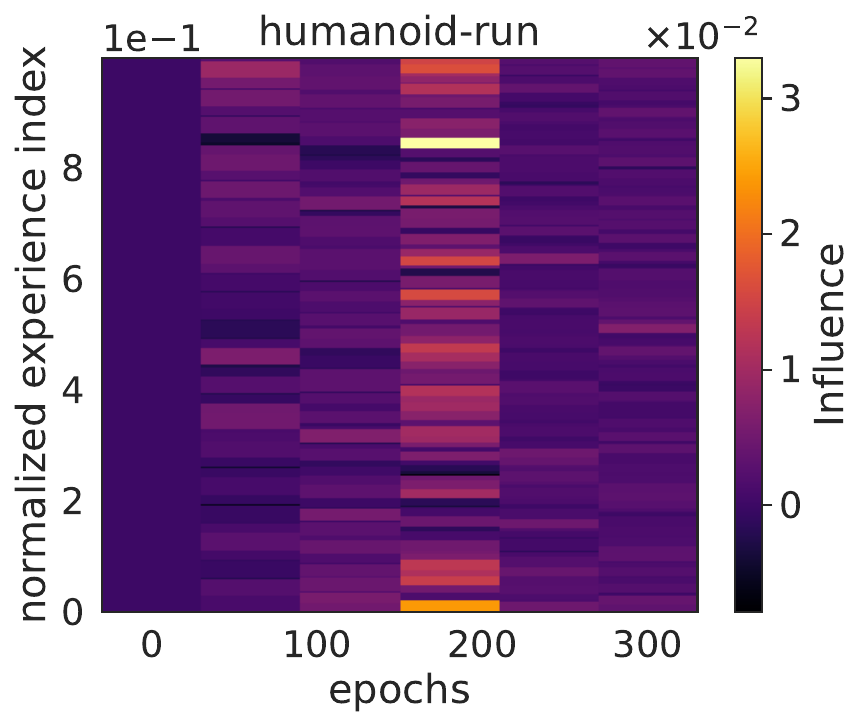}
\end{minipage}
%
\caption{
Distribution of experience influence on Q-estimation bias (Eq.~\ref{eq:influence_bias}) in DM control environments with adversarial experiences. 
}
\label{fig:distribution_of_bias_dmc}
\end{figure*}

\clearpage
\subsection{Additional experiments in DM control environments with adversarial experiences}\label{app:adversarial_dmc_additional}
\begin{figure*}[h!]
\begin{minipage}{1.0\hsize}
\includegraphics[clip, width=0.245\hsize]{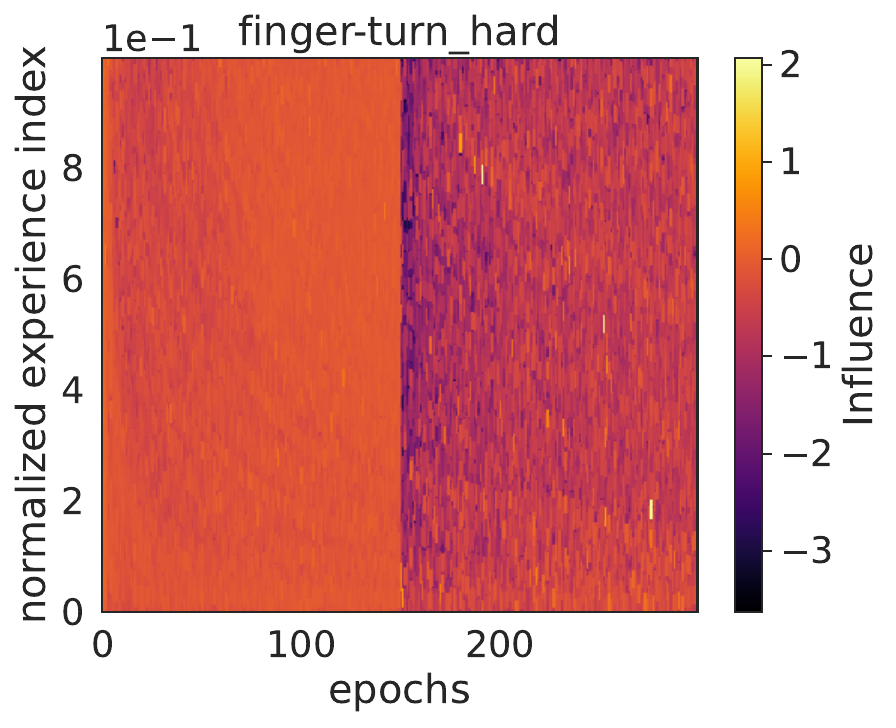}
\includegraphics[clip, width=0.245\hsize]{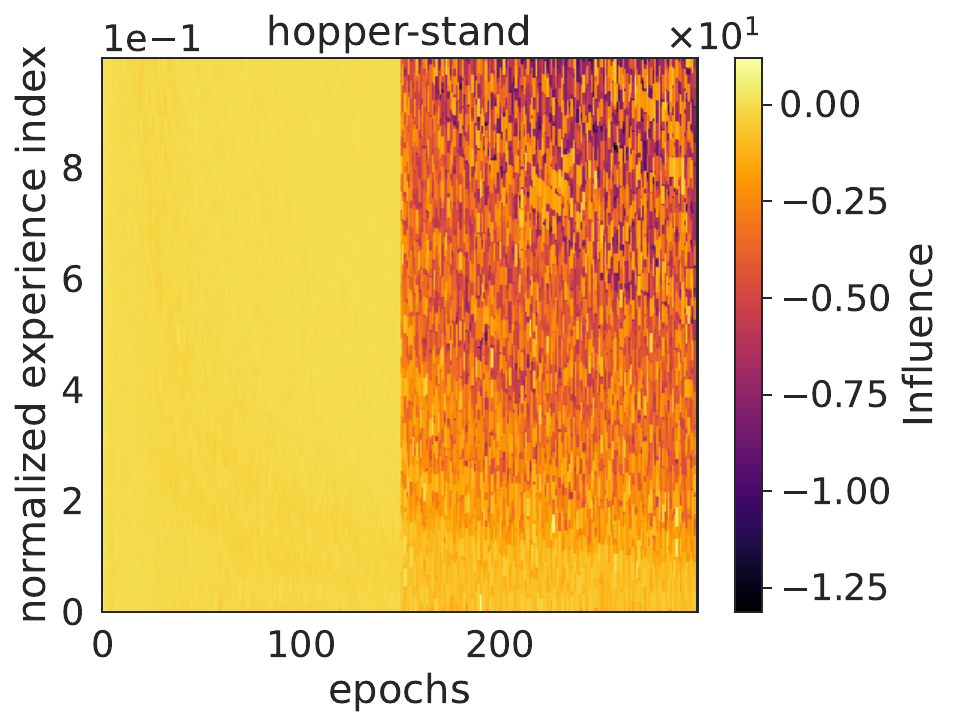}
\includegraphics[clip, width=0.245\hsize]{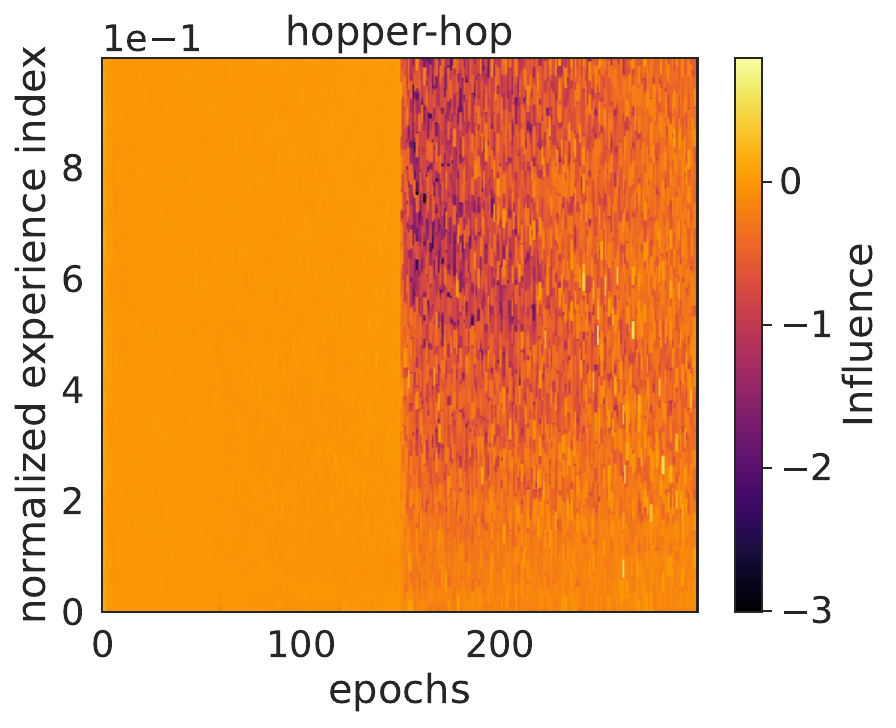}
\includegraphics[clip, width=0.245\hsize]{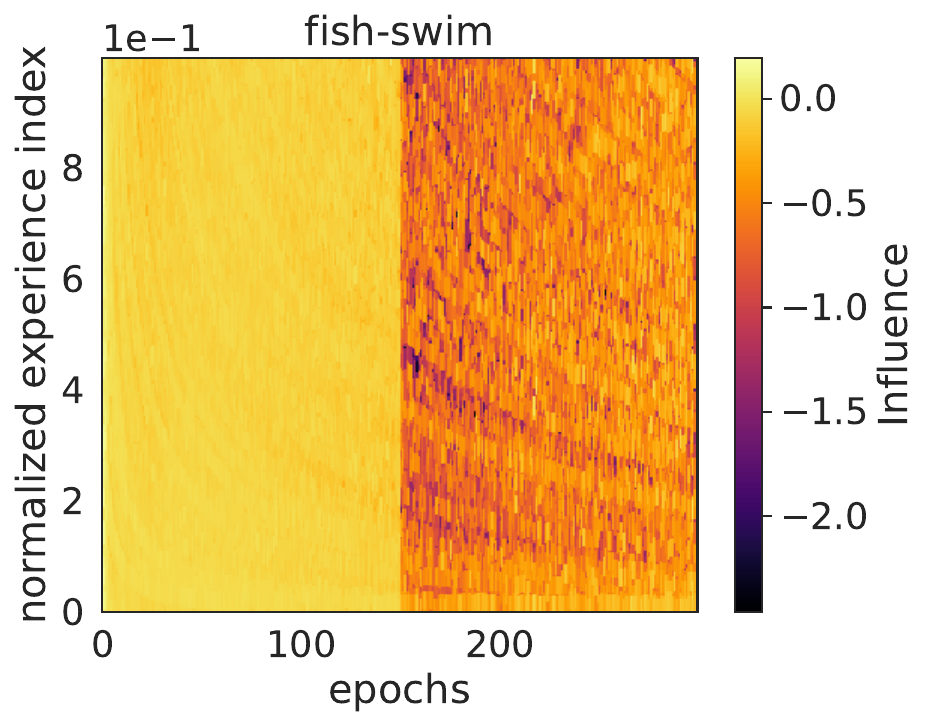}
\includegraphics[clip, width=0.245\hsize]{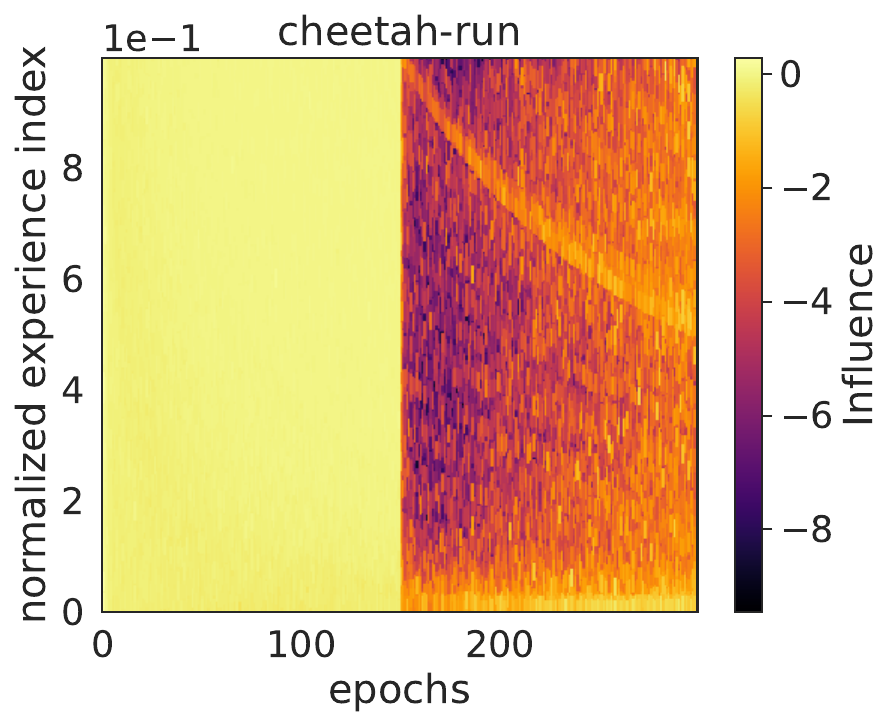}
\includegraphics[clip, width=0.245\hsize]{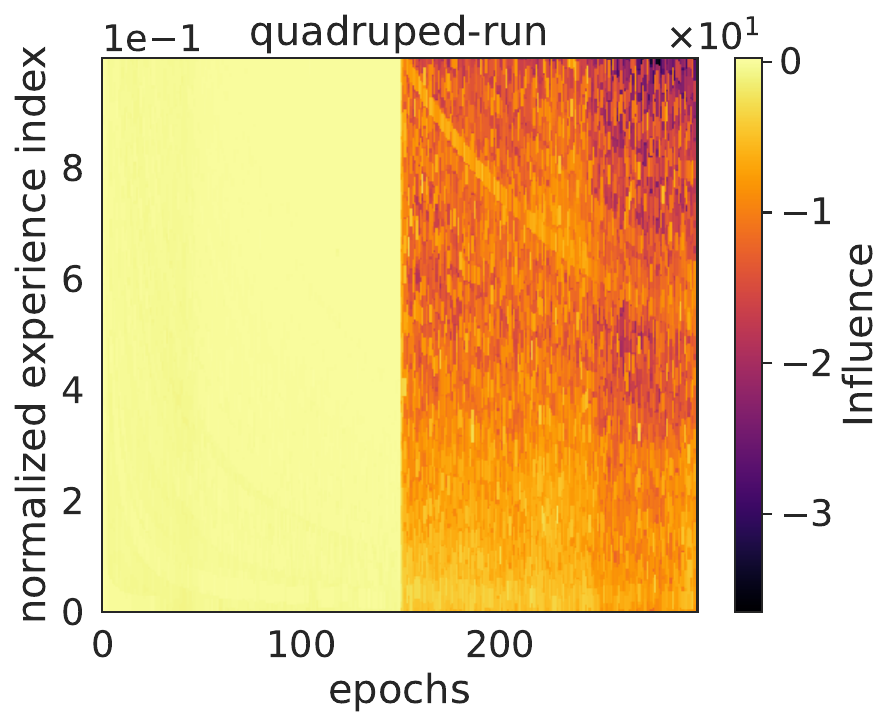}
\includegraphics[clip, width=0.245\hsize]{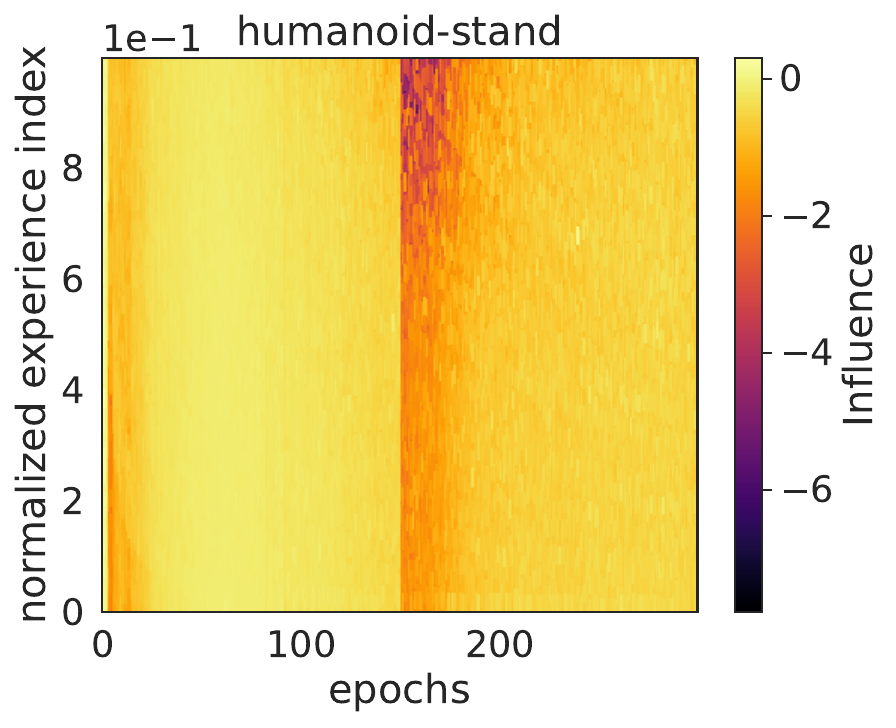}
\includegraphics[clip, width=0.245\hsize]{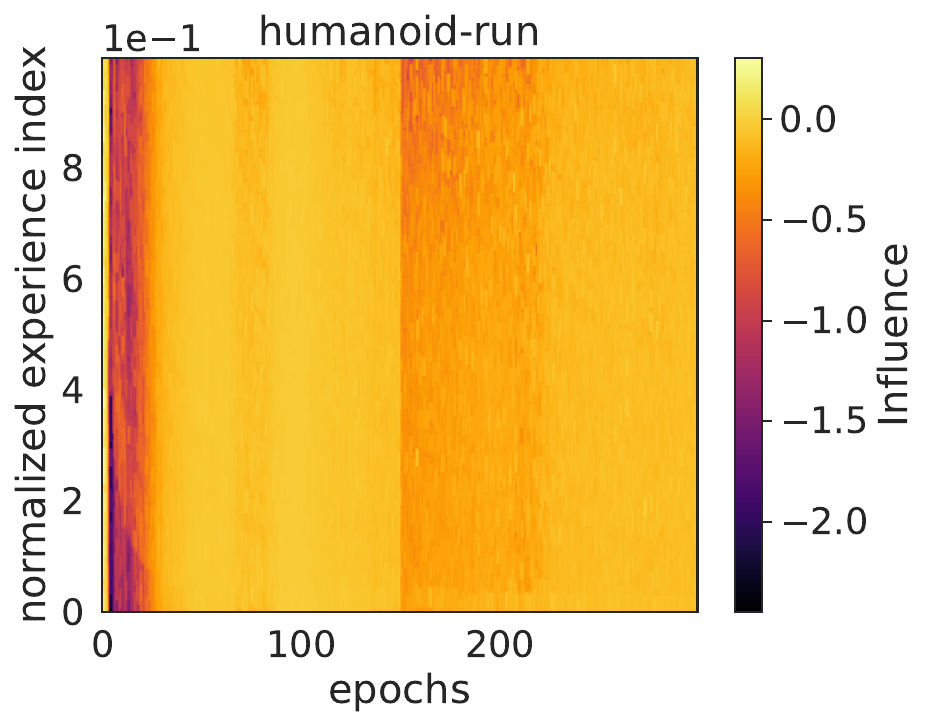}
\end{minipage}
%
\caption{
Distribution of experience influence on policy improvement (Eq.~\ref{eq:self_infl_policy_func}) in DM control environments with adversarial experiences.
}
\label{fig:distribution_of_self_influences_pi_dmc}
\end{figure*}
\begin{figure*}[h!]
\begin{minipage}{1.0\hsize}
\includegraphics[clip, width=0.245\hsize]{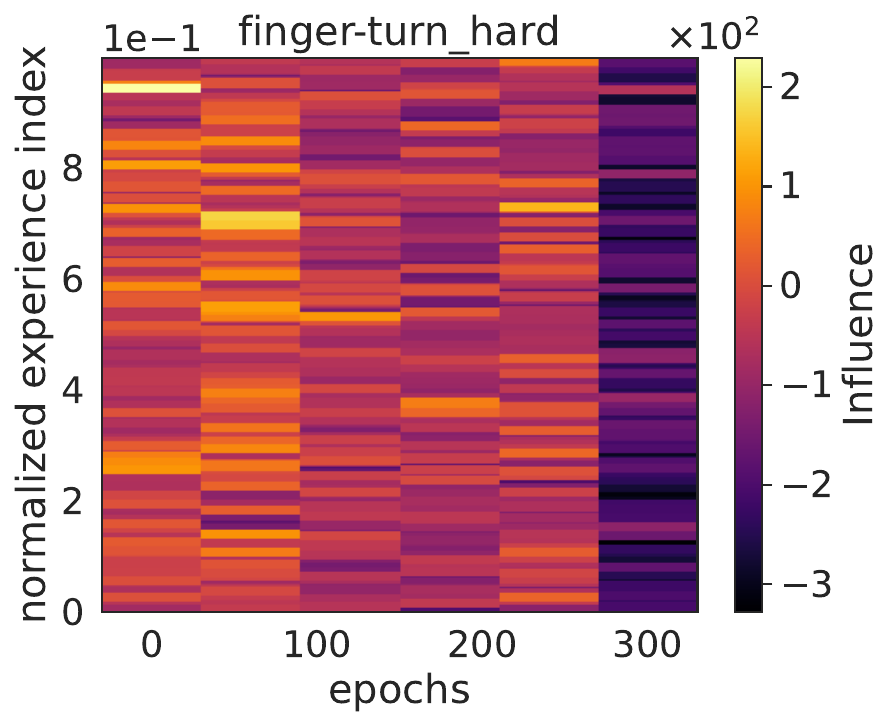}
\includegraphics[clip, width=0.245\hsize]{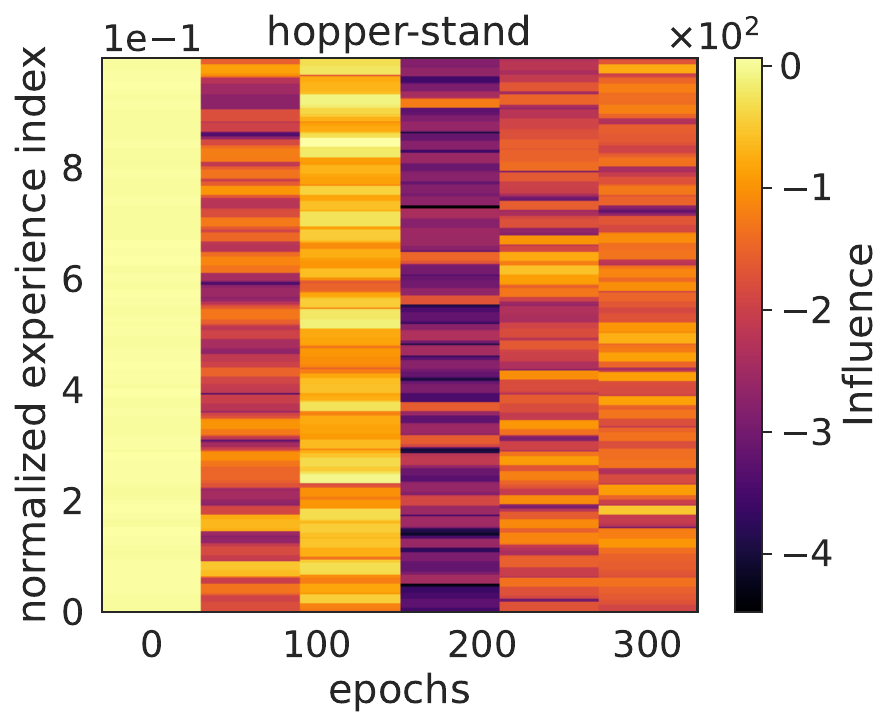}
\includegraphics[clip, width=0.245\hsize]{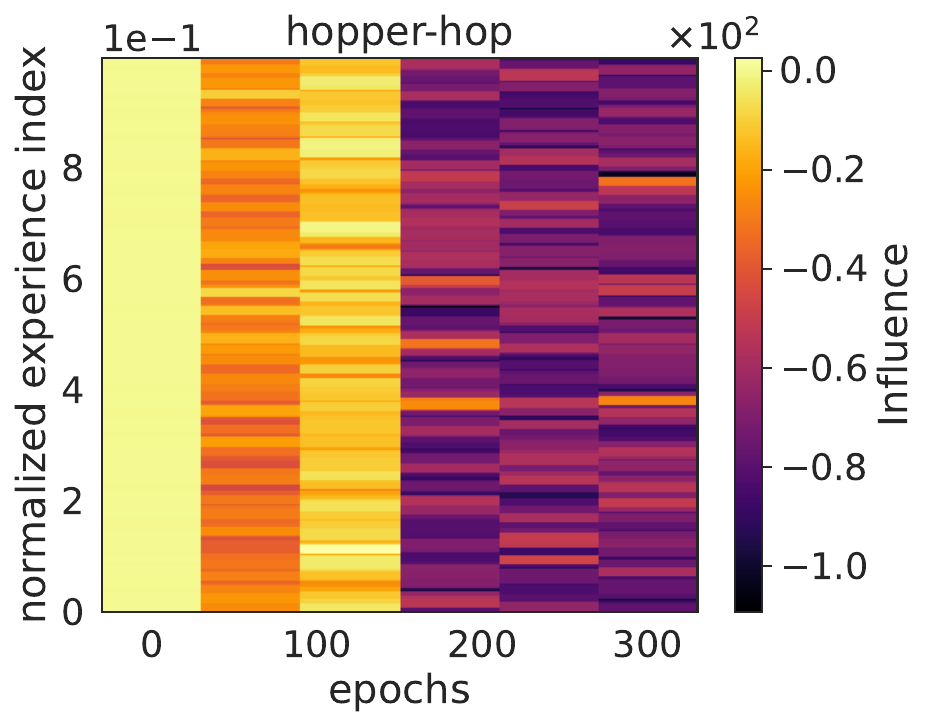}
\includegraphics[clip, width=0.245\hsize]{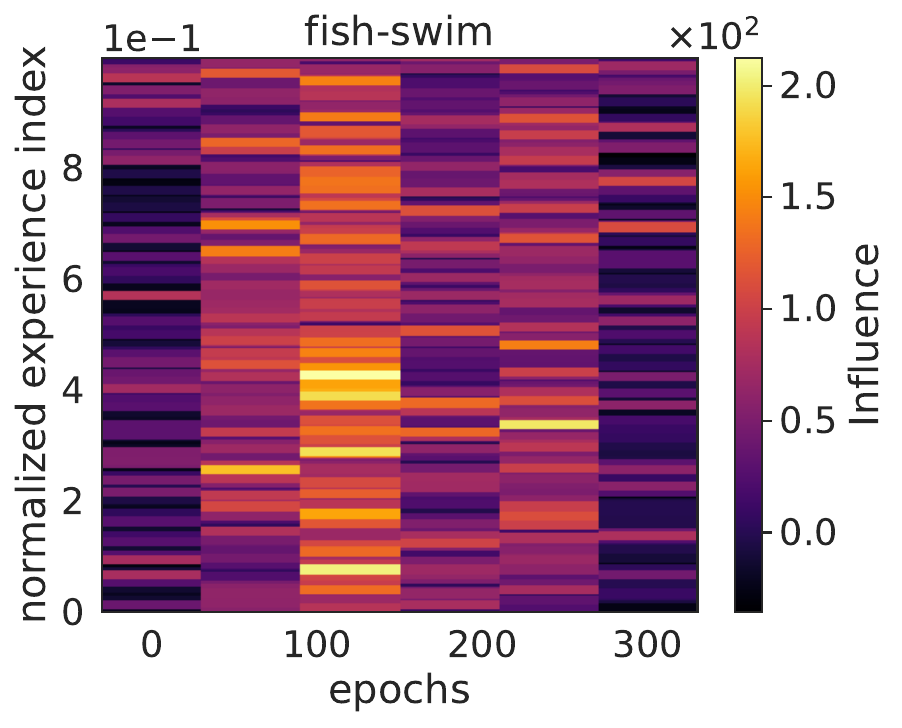}
\includegraphics[clip, width=0.245\hsize]{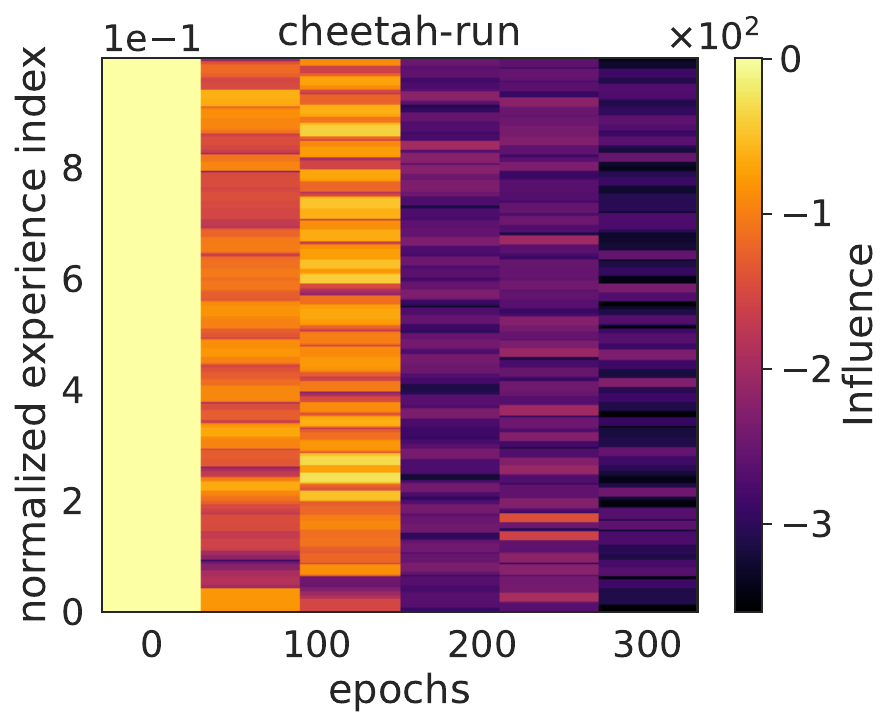}
\includegraphics[clip, width=0.245\hsize]{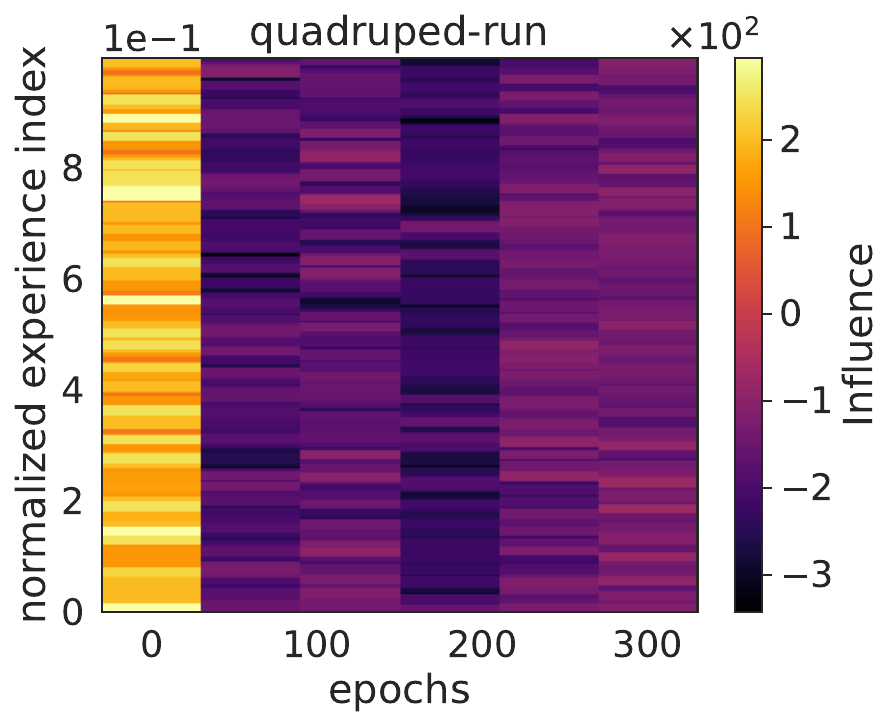}
\includegraphics[clip, width=0.245\hsize]{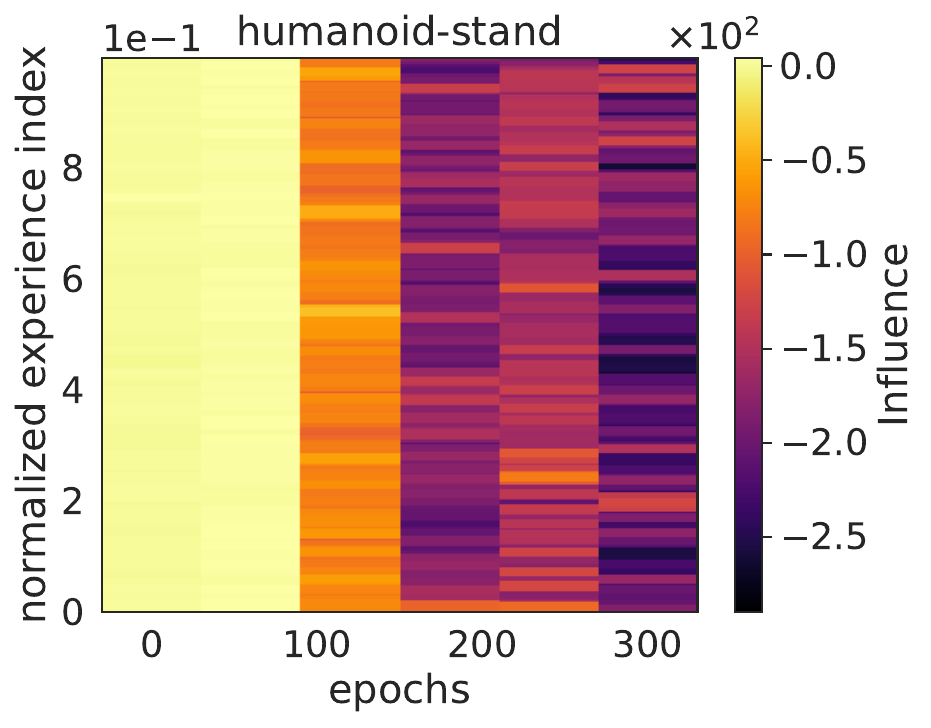}
\includegraphics[clip, width=0.245\hsize]{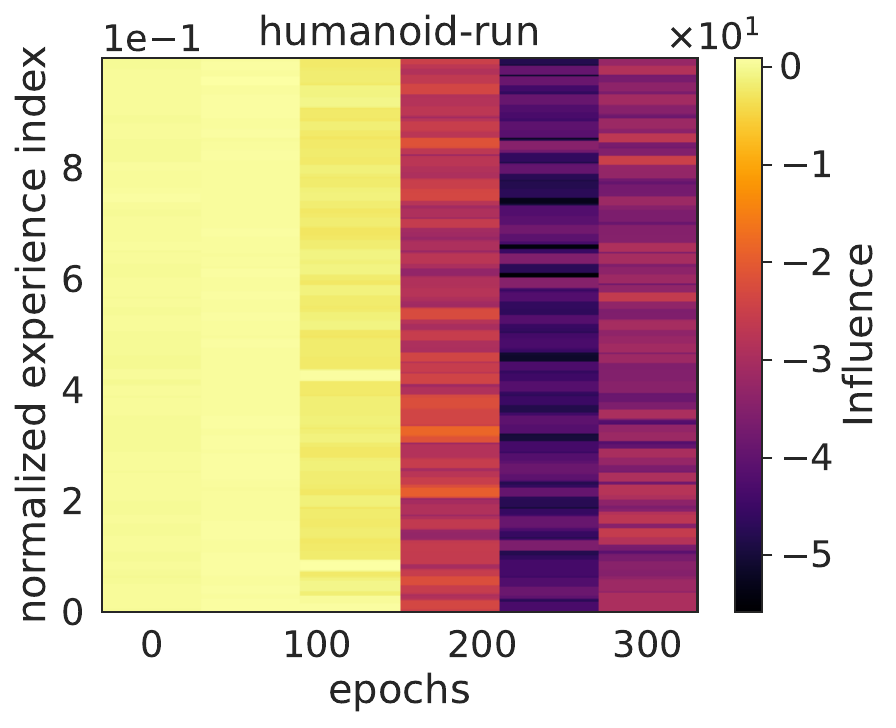}
\end{minipage}
%
\caption{
Distribution of experience influence on return (Eq.~\ref{eq:influence_return}) in DM control environments with adversarial experiences. 
}
\label{fig:distribution_of_return_dmc}
\end{figure*}
\begin{figure*}[h!]
\begin{minipage}{1.0\hsize}
\includegraphics[clip, width=0.49\hsize]{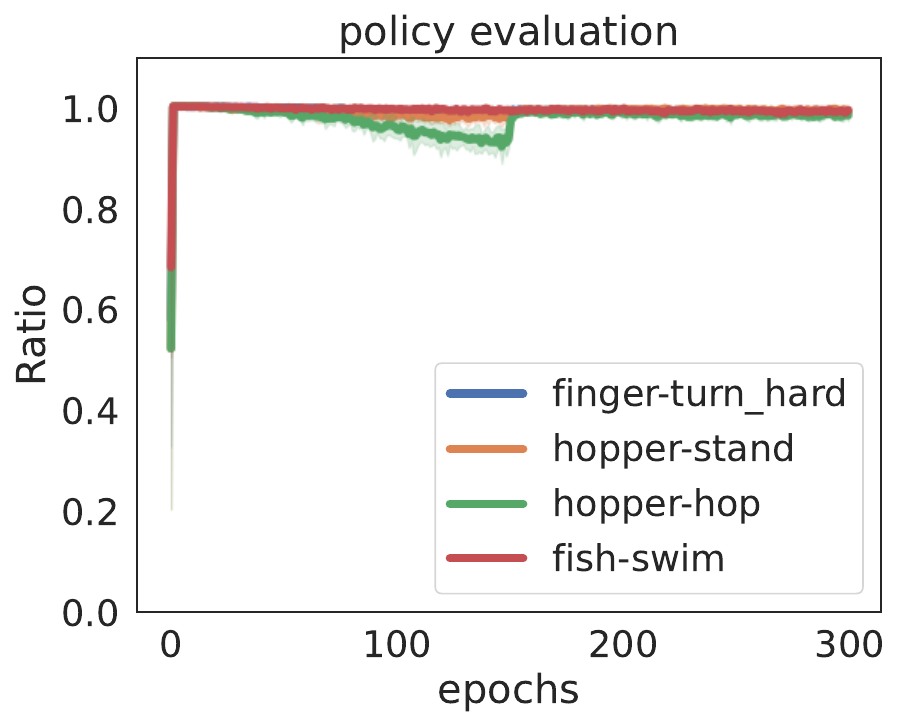}
\includegraphics[clip, width=0.49\hsize]{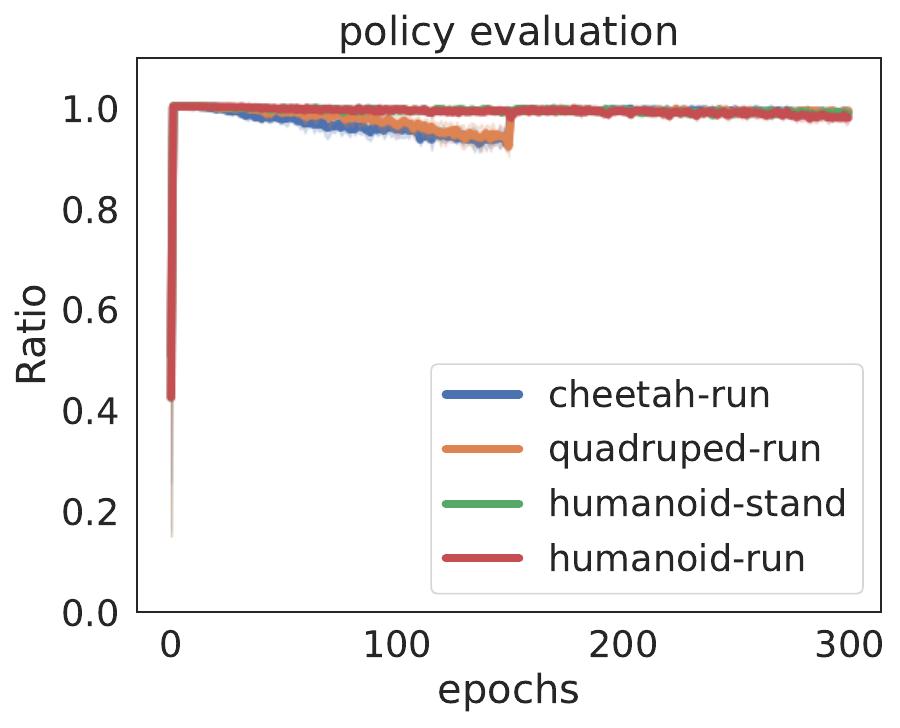}
\end{minipage}
\caption{
The ratio of experiences for which PIToD correctly estimated their influence on policy evaluation (Eq.~\ref{eq:self_infl_q_func}). 
}
\label{fig:raio_of_experiences_with_pos_neg_pe_dmc}
\end{figure*}
\begin{figure*}[h!]
\begin{minipage}{1.0\hsize}
\includegraphics[clip, width=0.49\hsize]{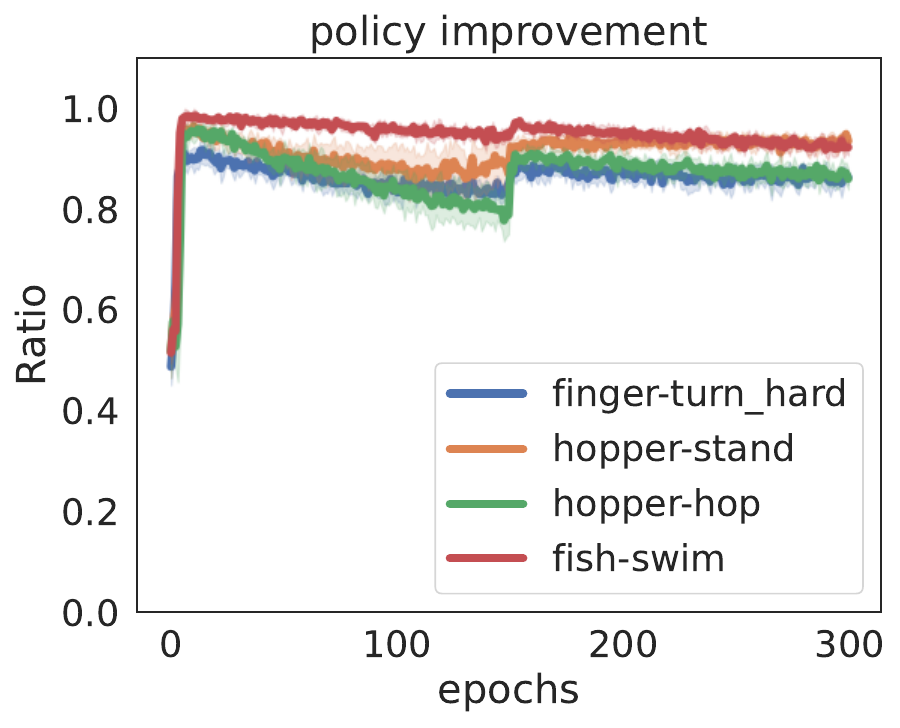}
\includegraphics[clip, width=0.49\hsize]{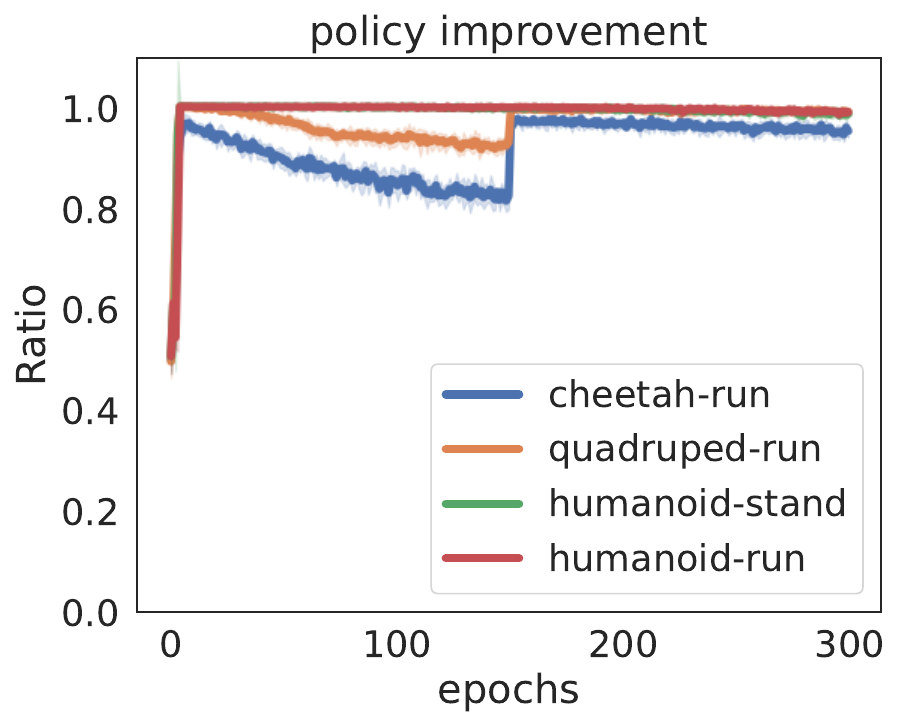}
\end{minipage}
\caption{
The ratio of experiences for which PIToD correctly estimated their influence on policy improvement (Eq.~\ref{eq:self_infl_policy_func}). 
}
\label{fig:raio_of_experiences_with_pos_neg_influence_pi_dmc}
\end{figure*}

\begin{figure*}[h!]
\begin{minipage}{0.99\hsize}
\includegraphics[clip, width=0.49\hsize]{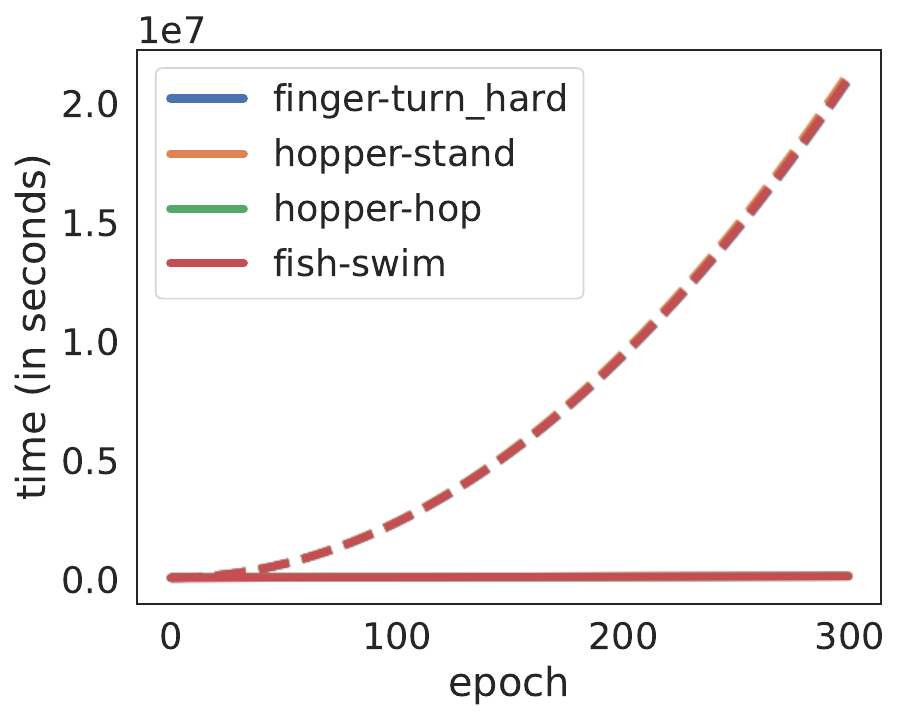}
\includegraphics[clip, width=0.49\hsize]{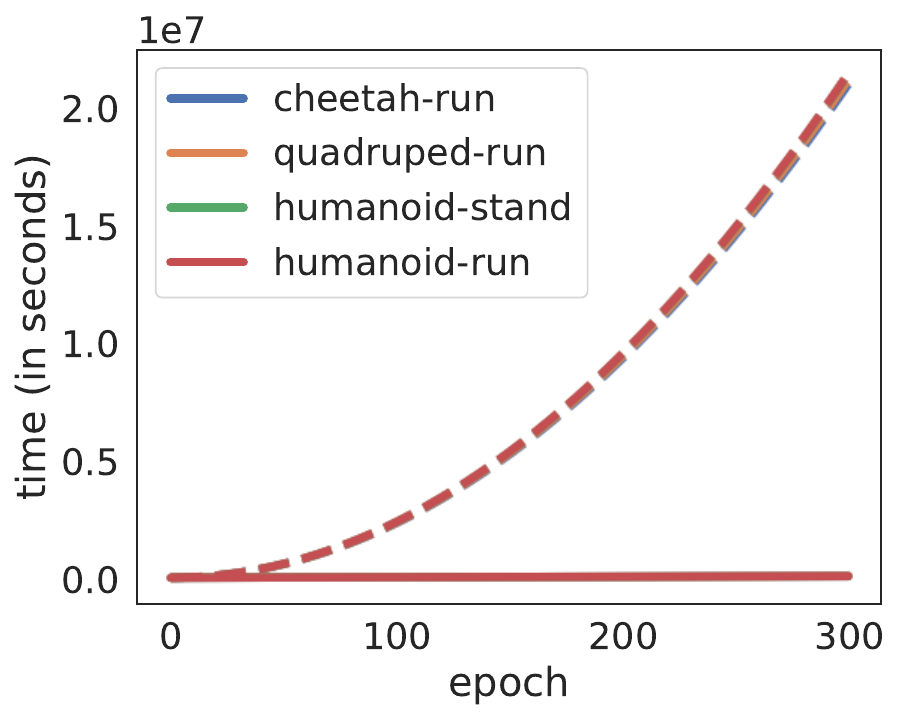}
\end{minipage}
\caption{
Wall-clock time required to estimate influence with PIToD and LOO. 
The solid line represents the time for PIToD, and the dashed line represents the estimated time for LOO. 
}
\label{fig:training_time_dmc}
\end{figure*}

\begin{figure*}[h!]
\begin{minipage}{0.99\hsize}
\includegraphics[clip, width=0.49\hsize]{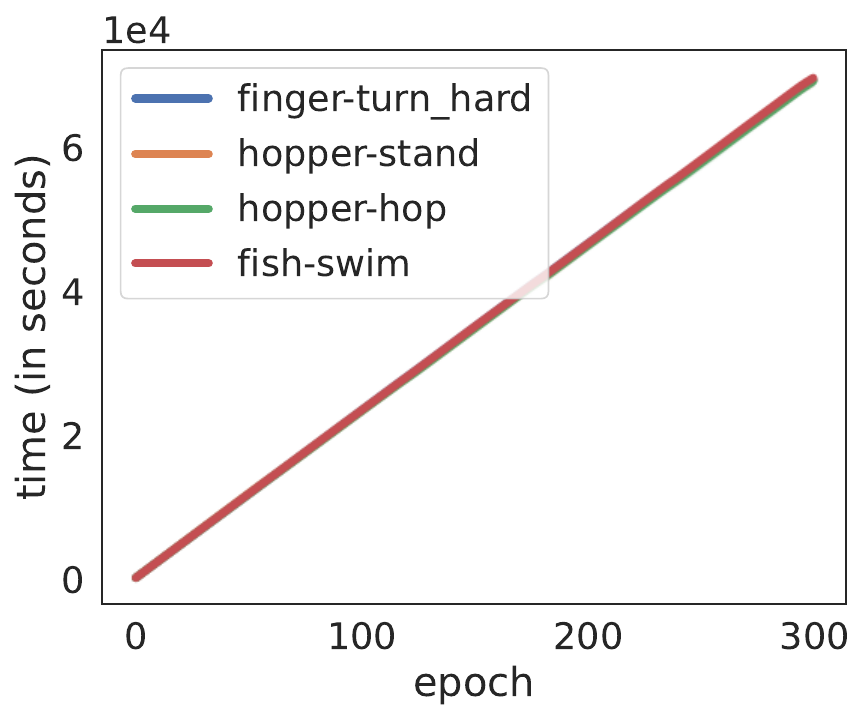}
\includegraphics[clip, width=0.49\hsize]{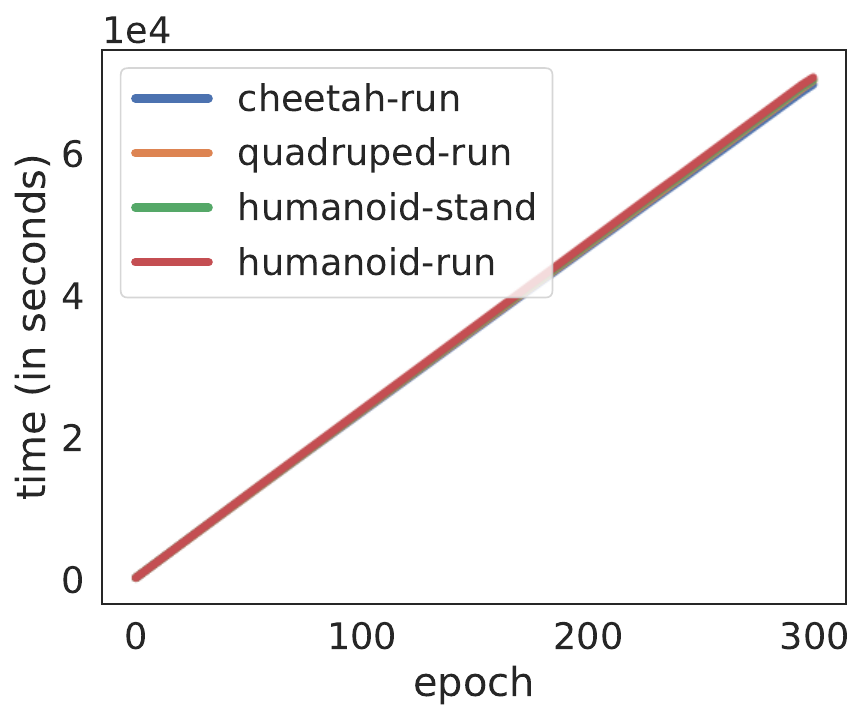}
\end{minipage}
\caption{
Wall-clock time required to estimate influence with PIToD. 
}
\label{fig:training_time_dmc_wo_loo}
\end{figure*}

\clearpage
\section{Amending the policies and Q-functions of DroQ and Reset agents}\label{app:amend-droq-reset}
In Section~\ref{sec:application}, we amended the SAC agent using PIToD. 
In this section, we apply PIToD to amend other RL agents. 

We evaluate two PIToD implementations: DroQToD and ResetToD.\\
\textbf{DroQToD} is a PIToD implementation based on DroQ~\citep{hiraoka2022dropout}. 
DroQ is the SAC variant that applies dropout and layer normalization to the Q-function. 
DroQToD differs from the original PIToD implementation (Appendix~\ref{sec:practical_implementation}) in that it has a dropout layer after each weight layer in the Q-function. 
The dropout rate is set to 0.01 as in \citet{hiraoka2022dropout}. 
Layer normalization is already included in the Q-function of the original PIToD implementation; thus, no additional changes are made to it.\\
\textbf{ResetToD} is a PIToD implementation based on the periodic reset~\citep{pmlr-v162-nikishin22a,doro2023sampleefficient} of the Q-function and policy parameters. 
ResetToD differs from the original PIToD implementation in that it resets the parameters of the Q-function and policy every $10^5$ steps. 

The policies and Q-functions of these implementations are amended as in Section~\ref{sec:application} (i.e., the amendment process follows Algorithm~\ref{alg4:Amendment} in Appendix~\ref{app:alg_amend}). 

The results of the policy and Q-function amendments (Figures~\ref{fig:cleansing_results_DroQ} and \ref{fig:cleansing_results_Reset}) show that the performance of both DroQToD and ResetToD is improved after the amendments. 
\textbf{Return:} 
For DroQToD, the return is improved after amendment, especially in Hopper (the left side of Figure~\ref{fig:cleansing_results_DroQ}). 
For ResetToD, the return is improved across all environments (the left side of Figure~\ref{fig:cleansing_results_Reset}). 
\textbf{Q-estimation bias:} 
For DroQToD, the estimation bias is reduced after amendment, especially in Humanoid (the right side of Figure~\ref{fig:cleansing_results_DroQ}). 
For ResetToD, the estimation bias is reduced in the early stages of training (epochs 0--10) in Ant and Walker2d (the right side of Figure~\ref{fig:cleansing_results_Reset}). 
\begin{figure*}[h!]
\begin{minipage}{1.0\hsize}
\includegraphics[clip, width=0.49\hsize]{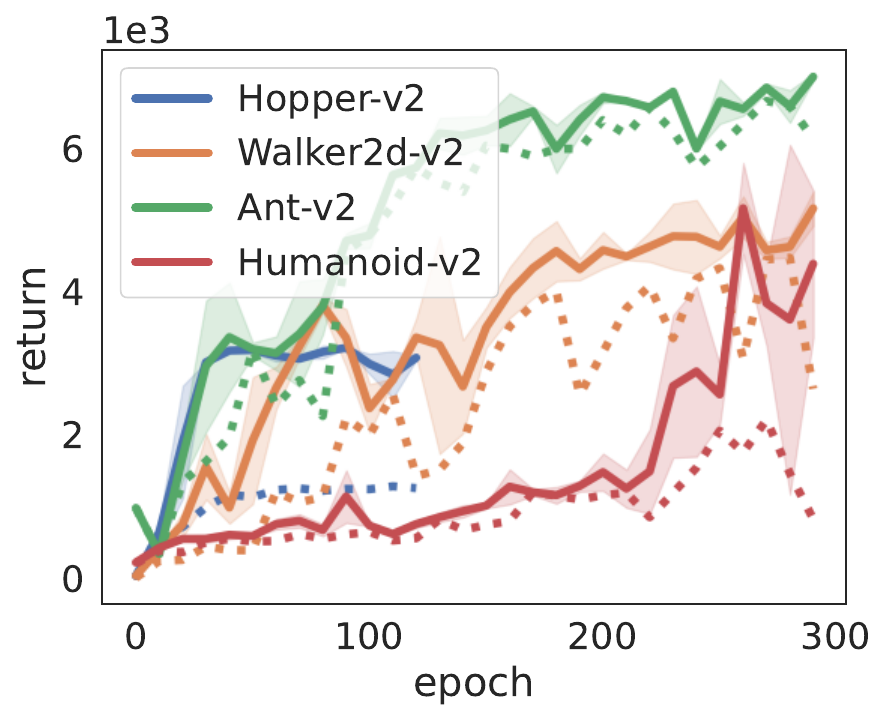}
\includegraphics[clip, width=0.49\hsize]{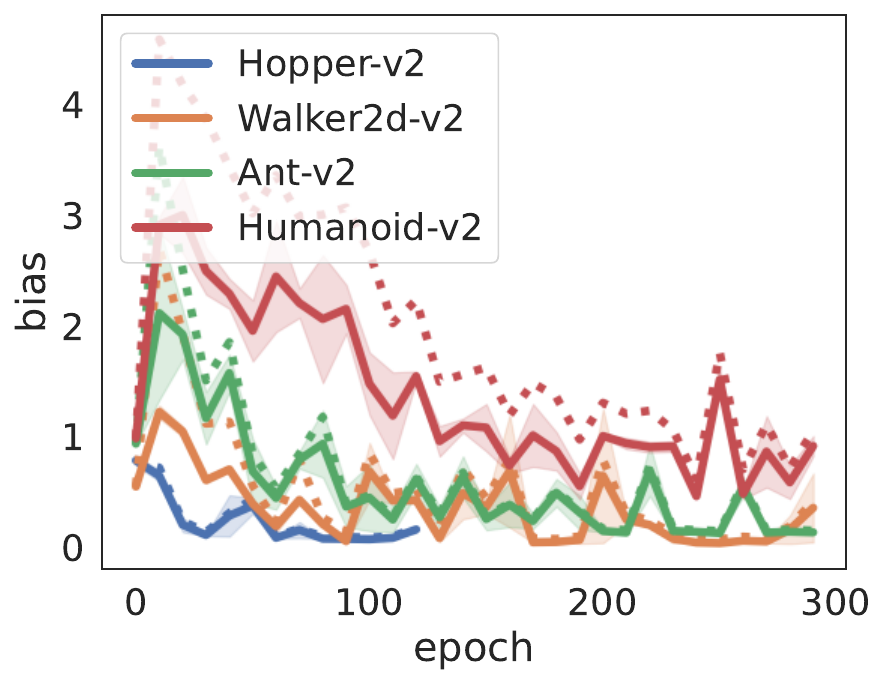}
\end{minipage}
\caption{
Results of the policy amendments (left) and the Q-function amendments (right) for DroQToD in underperforming trials. 
The solid lines represent the post-amendment performances: return for the policy (left; i.e., $L_{\text{ret}}(\pi_{\theta, \mathbf{w}_*})$) and bias for the Q-function (right; i.e., $L_{\text{bias}}(Q_{\phi, \mathbf{w}_*})$). 
The dashed lines show the pre-amendment performances: return (left; i.e., $L_{\text{ret}}(\pi_{\theta})$) and bias (right; i.e., $L_{\text{bias}}(Q_{\phi})$). 
}
\label{fig:cleansing_results_DroQ}
\end{figure*}
\begin{figure*}[h!]
\begin{minipage}{1.0\hsize}
\includegraphics[clip, width=0.49\hsize]{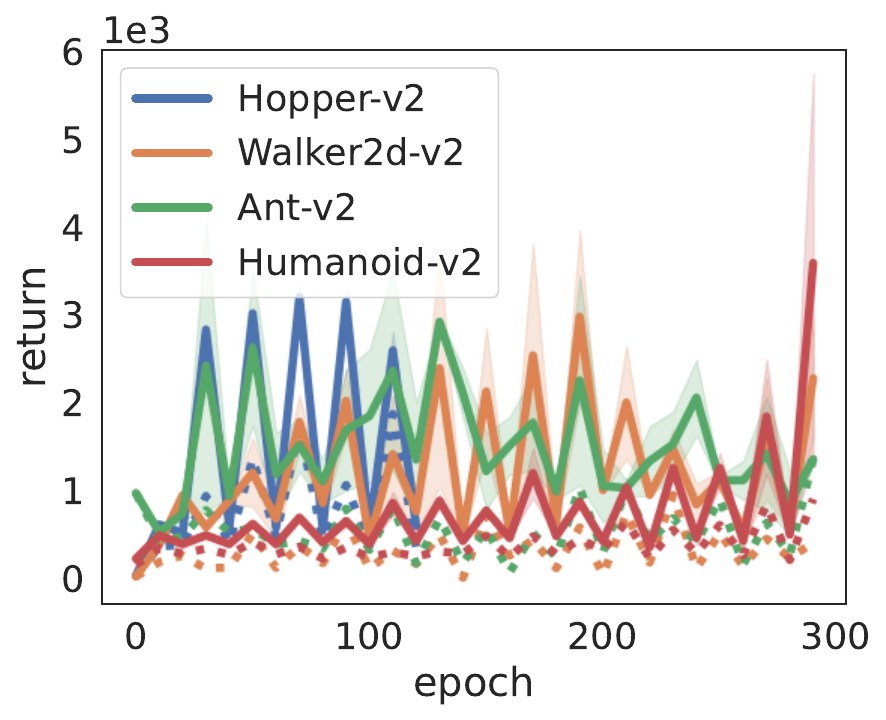}
\includegraphics[clip, width=0.49\hsize]{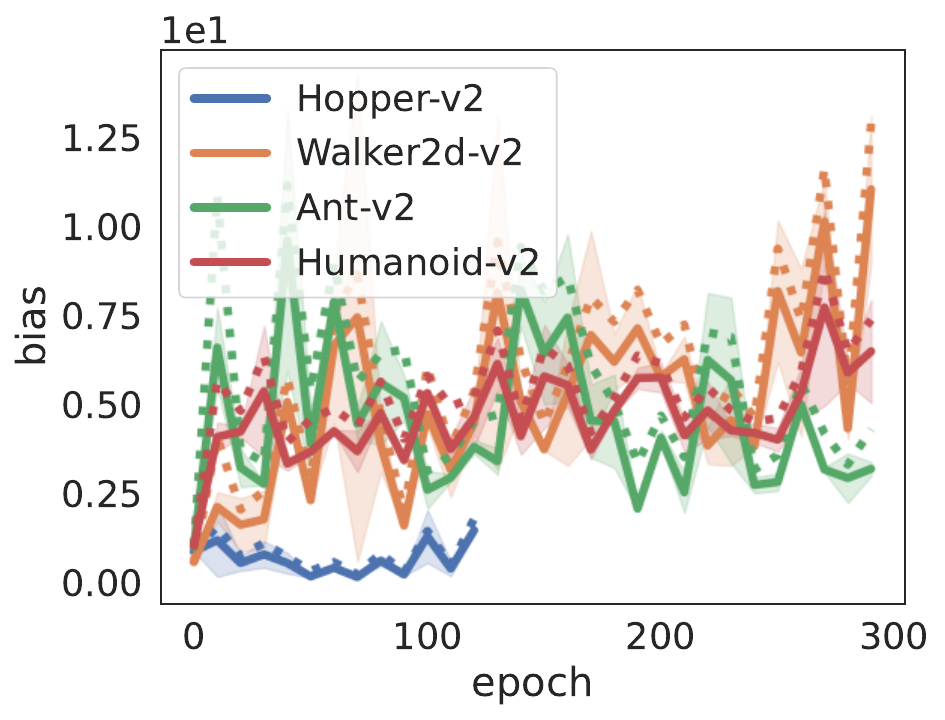}
\end{minipage}
\caption{
Results of the policy amendments (left) and the Q-function amendments (right) for ResetToD in underperforming trials. 
The solid lines represent the post-amendment performances: return for the policy (left; i.e., $L_{\text{ret}}(\pi_{\theta, \mathbf{w}_*})$) and bias for the Q-function (right; i.e., $L_{\text{bias}}(Q_{\phi, \mathbf{w}_*})$). 
The dashed lines show the pre-amendment performances: return (left; i.e., $L_{\text{ret}}(\pi_{\theta})$) and bias (right; i.e., $L_{\text{bias}}(Q_{\phi})$). 
}
\label{fig:cleansing_results_Reset}
\end{figure*}

What experiences negatively influence Q-function or policy performance in the case of DroQToD?
Regarding Q-function performance, older experiences negatively influence Q-estimation bias in the early stages of training (the lower part of Figure~\ref{fig:distribution_of_bias_return_droq} in Appendix~\ref{app:additional_results_droq}). 
Regarding policy performance, some experiences negatively influencing returns are associated with wobbly movements. 
An example of such experiences in the Humanoid environment can be seen in the video at the following link: \url{https://github.com/user-attachments/assets/a47d8a54-a794-4e04-a48d-05e03ad31e9e}

\clearpage
\subsection{Additional experimental results for DroQToD}\label{app:additional_results_droq}
\begin{figure*}[h!]
\begin{minipage}{1.0\hsize}
\includegraphics[clip, width=0.49\hsize]{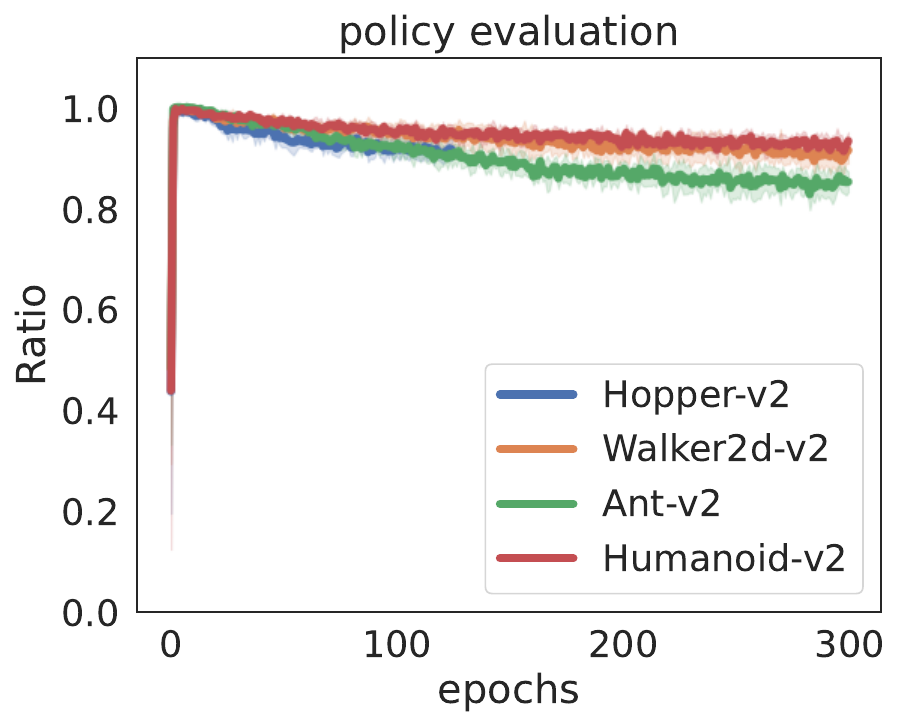}
\includegraphics[clip, width=0.49\hsize]{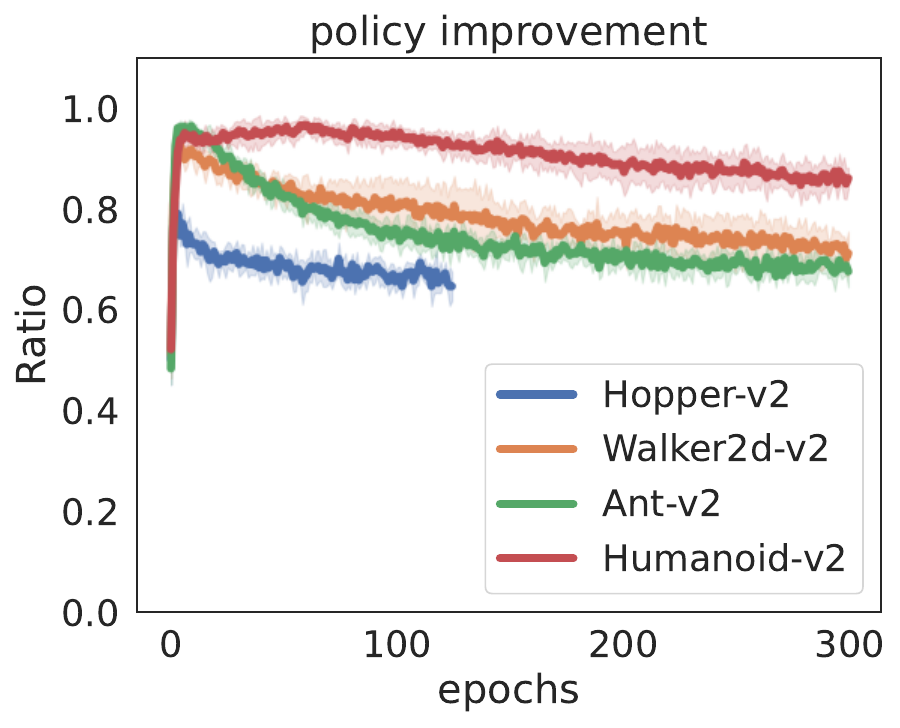}
\end{minipage}
\caption{
The ratio of experiences for which DroQToD correctly estimated their self-influence. 
}
\label{fig:raio_of_experiences_with_pos_neg_influence_droq}
\end{figure*}
\begin{figure*}[h!]
\begin{minipage}{1.0\hsize}
\includegraphics[clip, width=0.245\hsize]{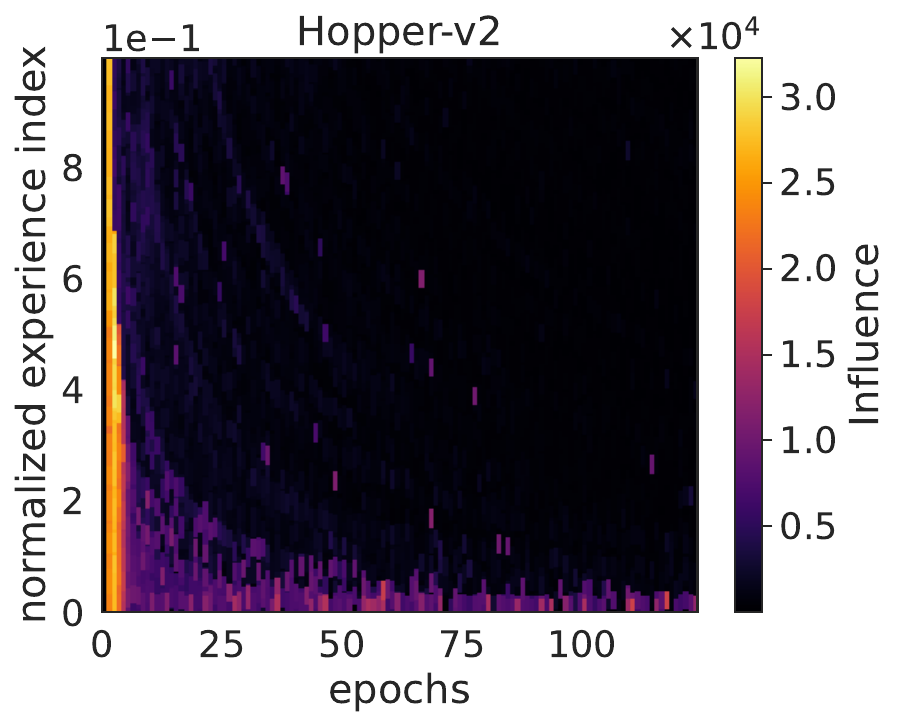}
\includegraphics[clip, width=0.245\hsize]{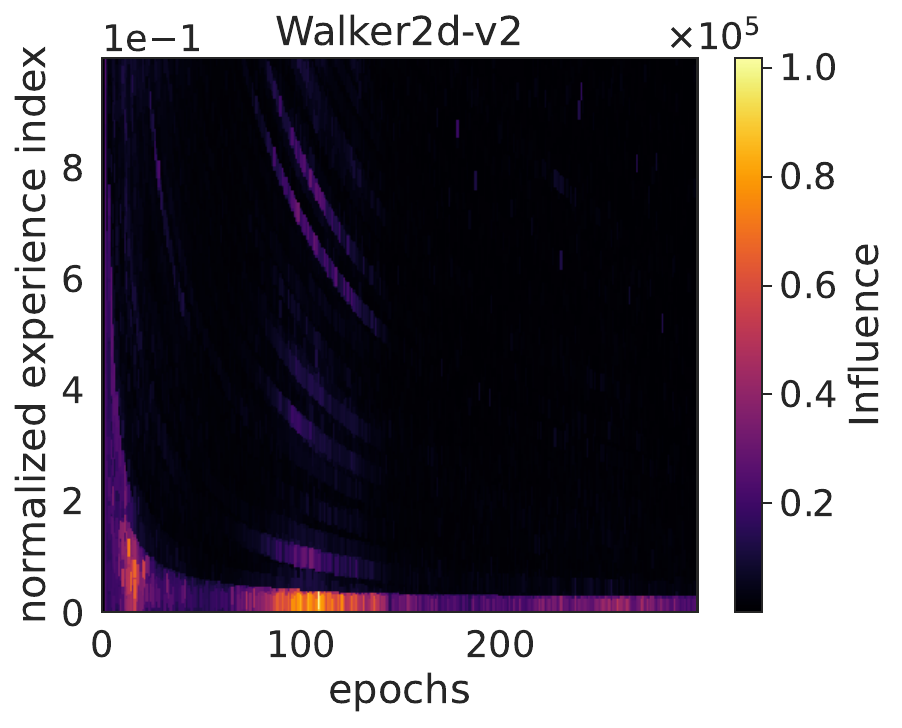}
\includegraphics[clip, width=0.245\hsize]{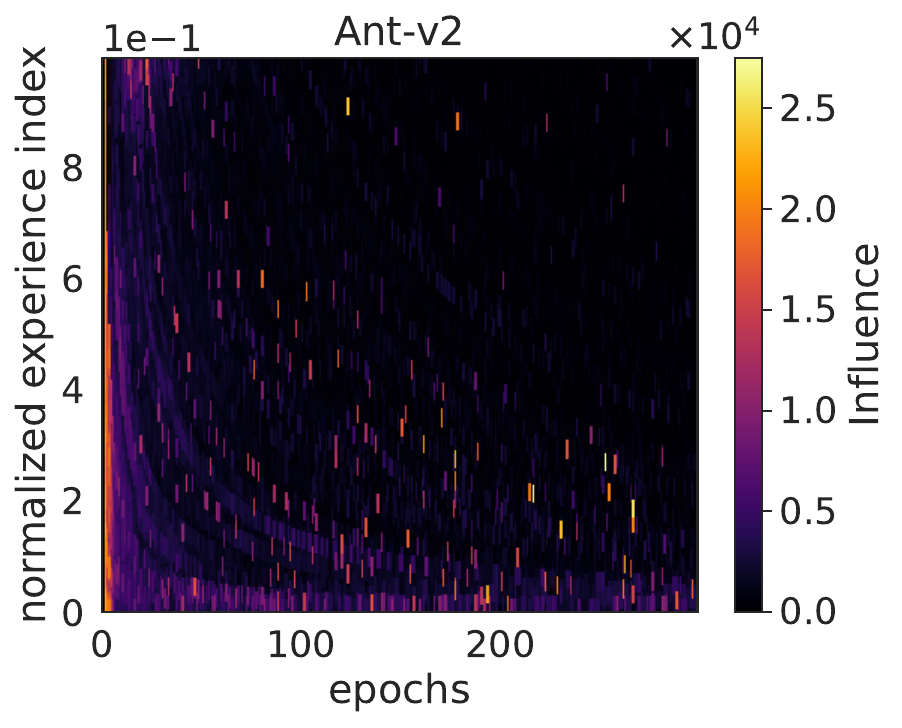}
\includegraphics[clip, width=0.245\hsize]{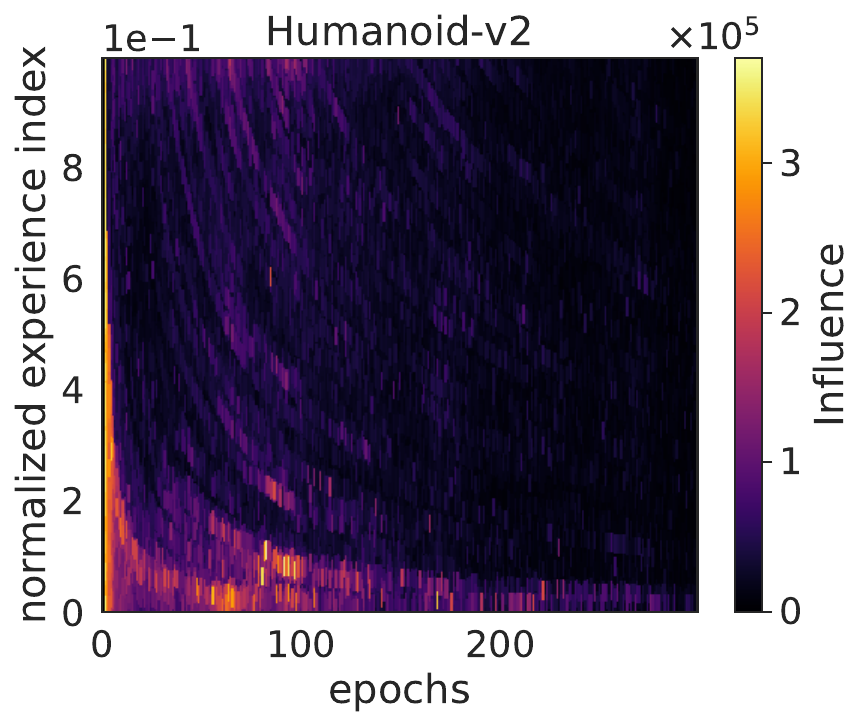}
\subcaption{Distribution of self-influence on policy evaluation (Eq.~\ref{eq:self_infl_q_func}).}
\end{minipage}
\begin{minipage}{1.0\hsize}
\includegraphics[clip, width=0.245\hsize]{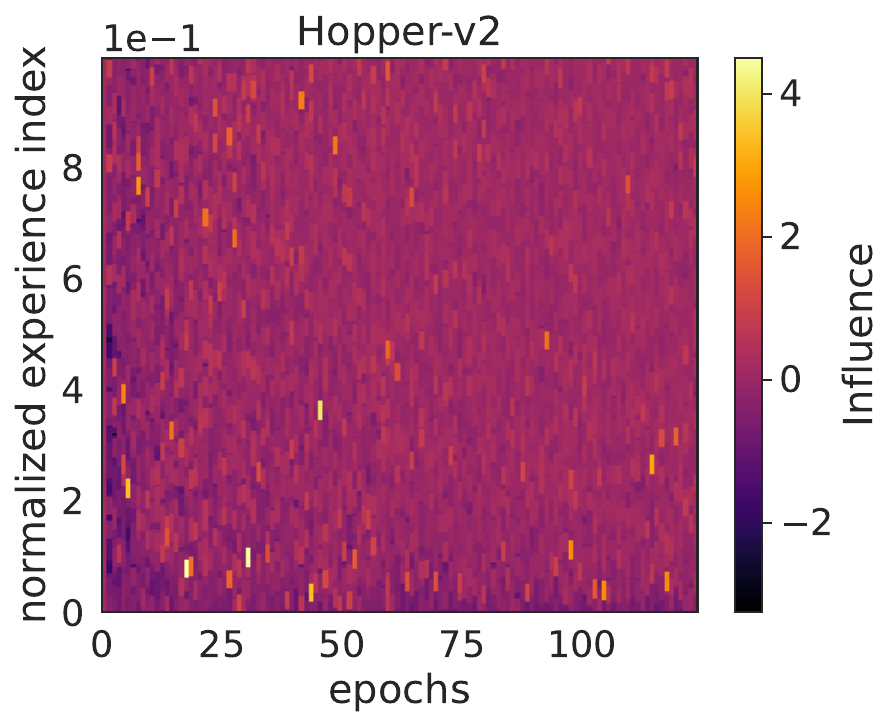}
\includegraphics[clip, width=0.245\hsize]{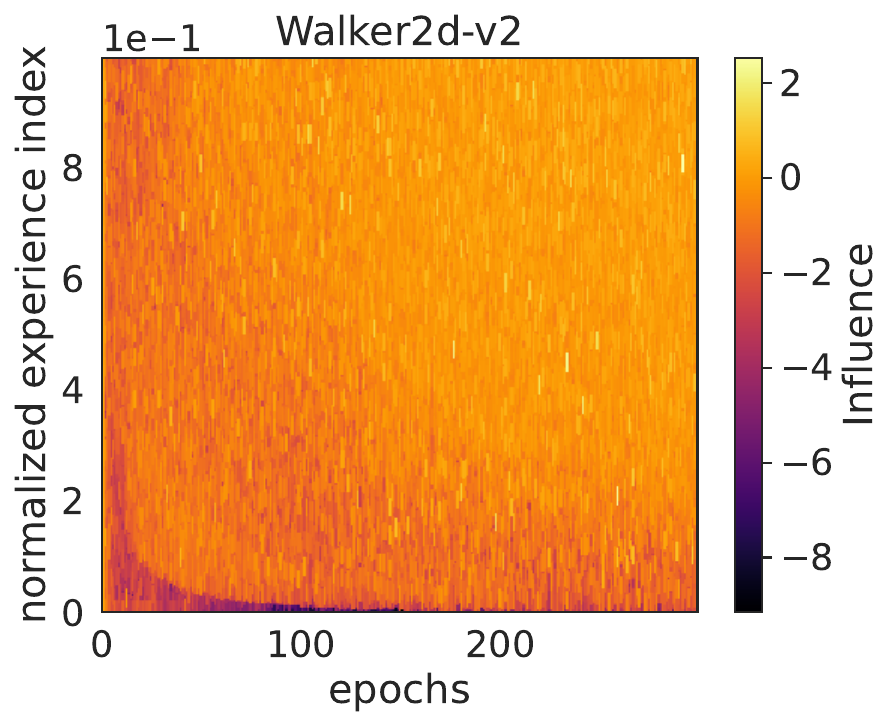}
\includegraphics[clip, width=0.245\hsize]{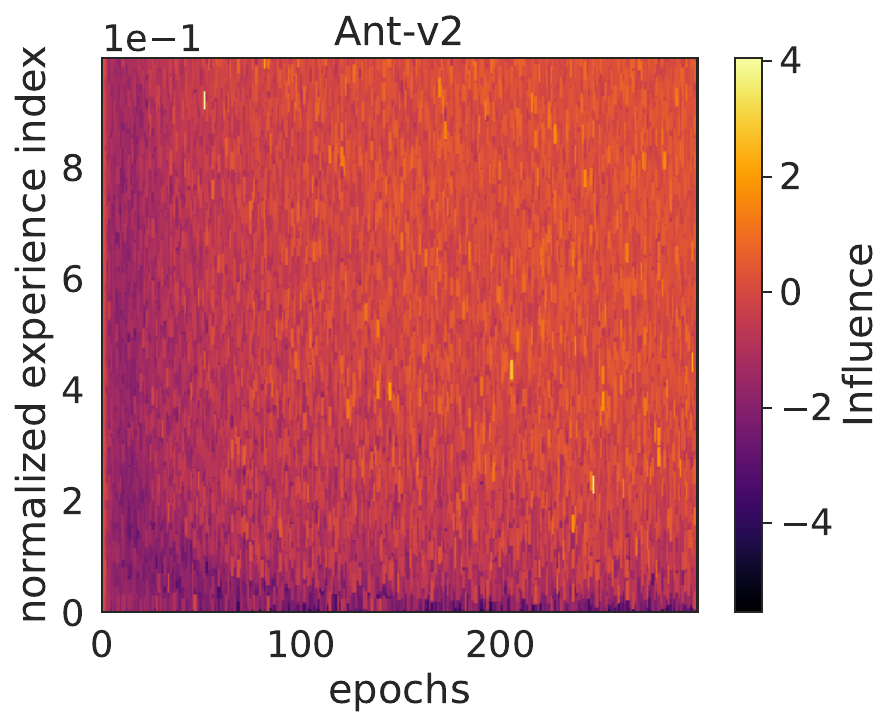}
\includegraphics[clip, width=0.245\hsize]{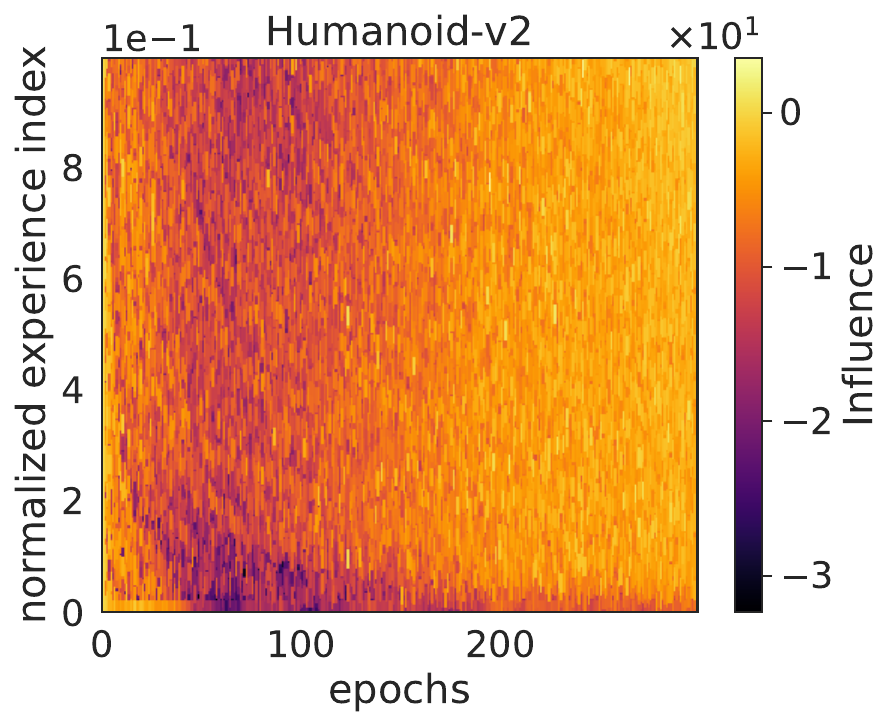}
\subcaption{Distribution of self-influence on policy improvement (Eq.~\ref{eq:self_infl_policy_func}).}
\end{minipage}
%
\caption{
Distribution of self-influence on policy evaluation and policy improvement. 
}
\label{fig:distribution_of_self_influences_droq}
\end{figure*}
\begin{figure*}[h!]
\begin{minipage}{1.0\hsize}
\includegraphics[clip, width=0.49\hsize]{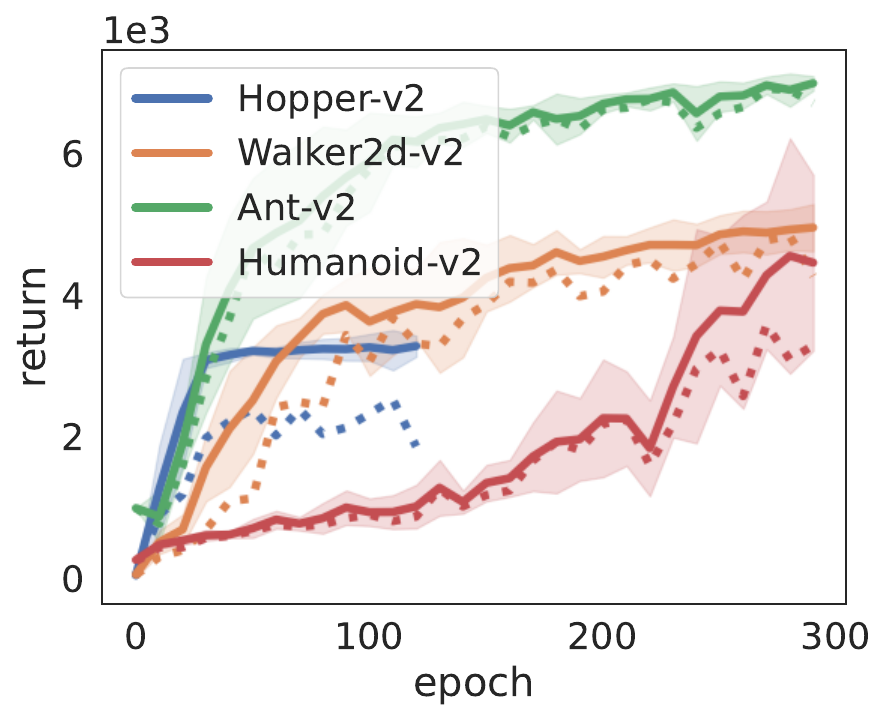}
\includegraphics[clip, width=0.49\hsize]{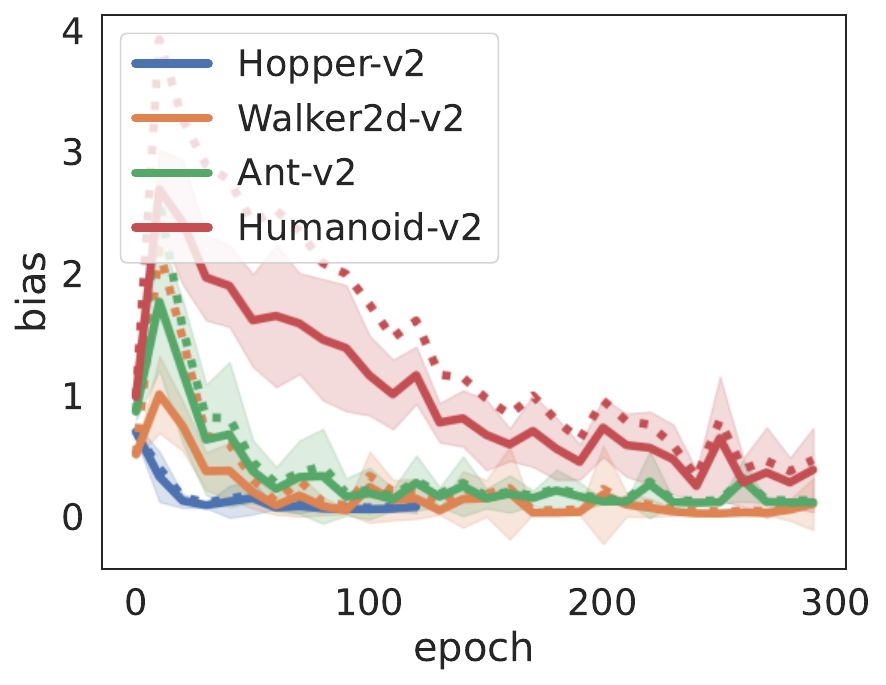}
\end{minipage}
\caption{
Results of policy amendments (left) and Q-function amendments (right) for all ten trials. 
}
\label{fig:cleansing_results_average_case_droq}
\end{figure*}
\begin{figure*}[h!]
\begin{minipage}{1.0\hsize}
\includegraphics[clip, width=0.245\hsize]{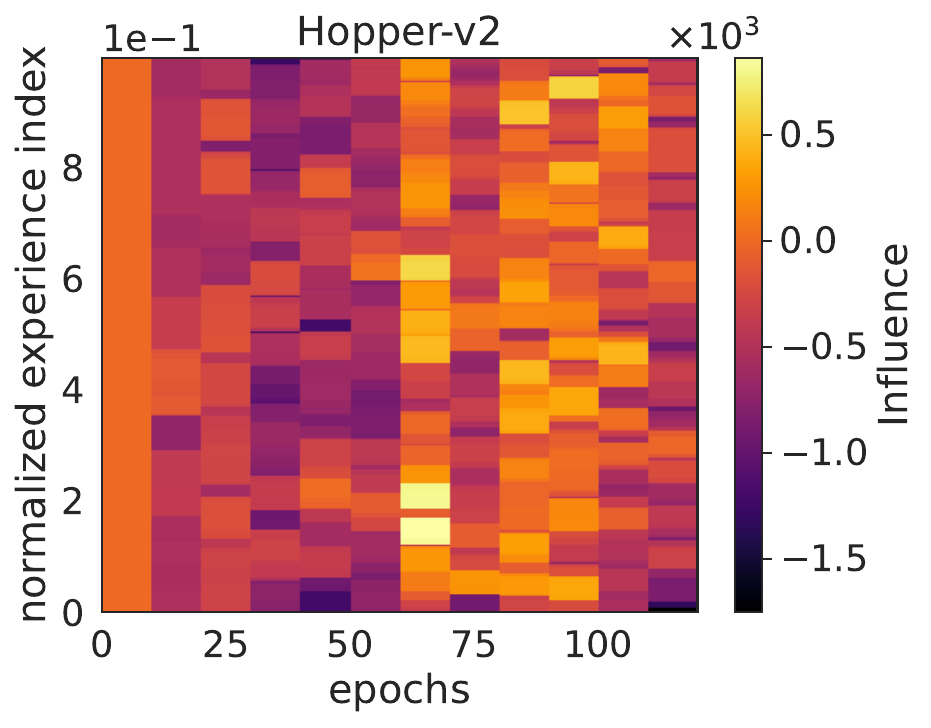}
\includegraphics[clip, width=0.245\hsize]{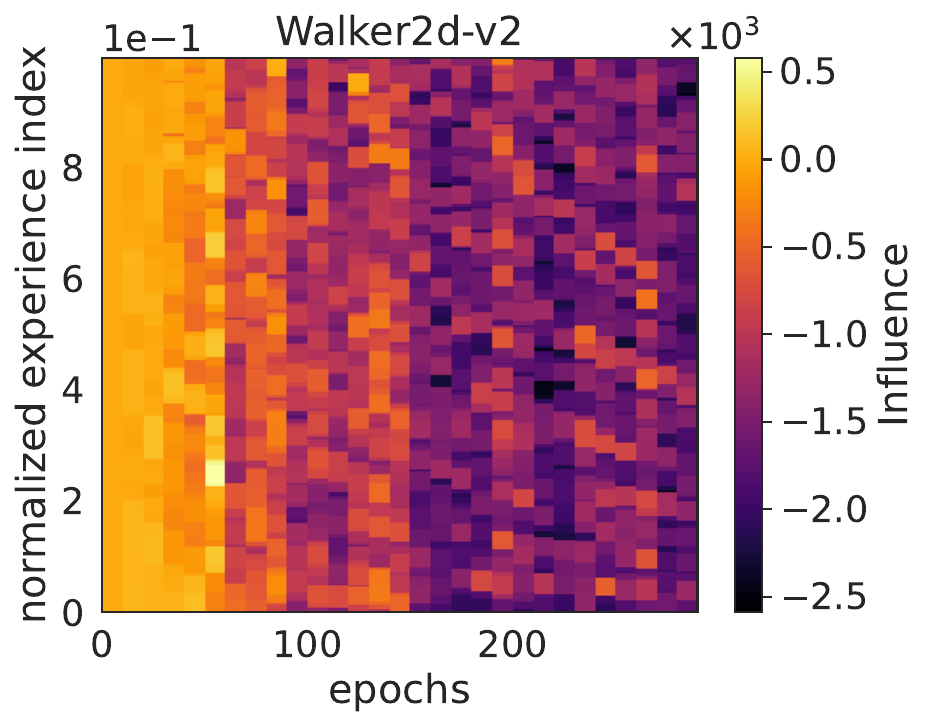}
\includegraphics[clip, width=0.245\hsize]{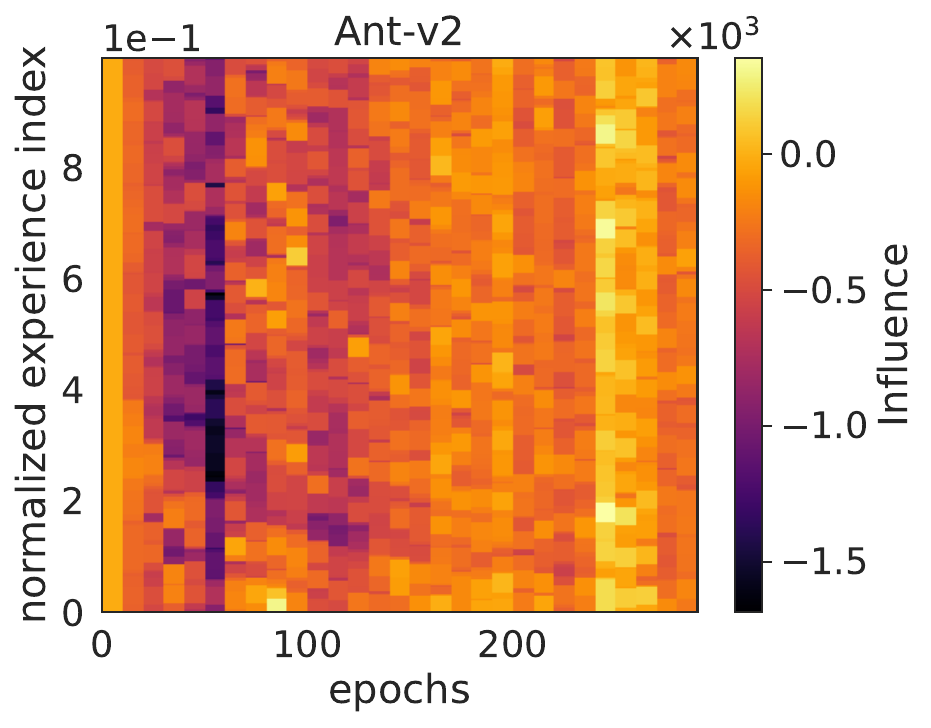}
\includegraphics[clip, width=0.245\hsize]{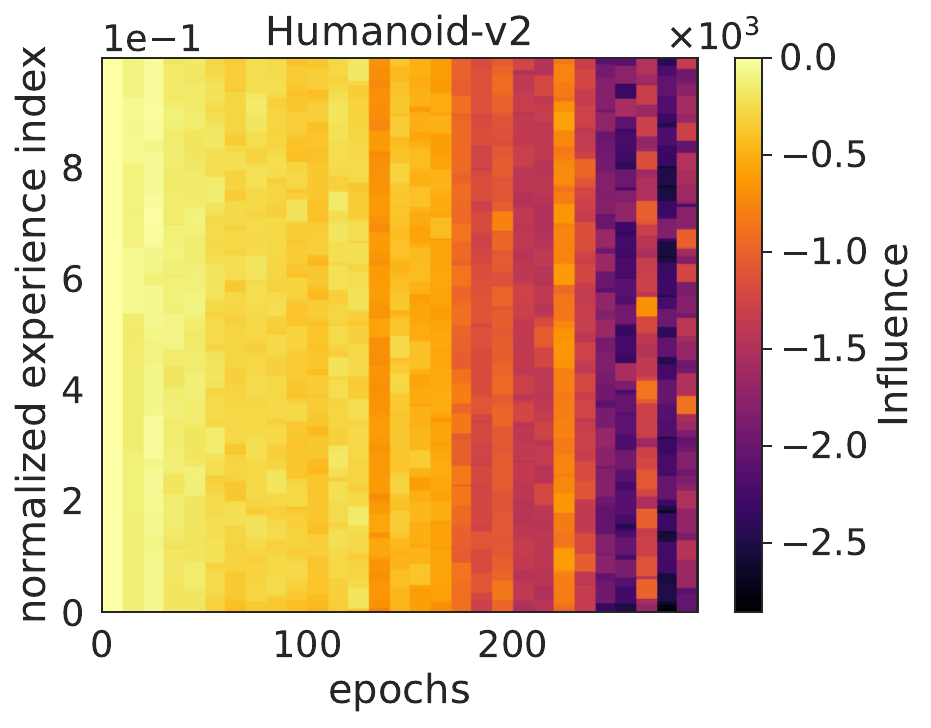}
\subcaption{Distribution of influence on return (Eq.~\ref{eq:influence_return}).}
\end{minipage}
\begin{minipage}{1.0\hsize}
\includegraphics[clip, width=0.245\hsize]{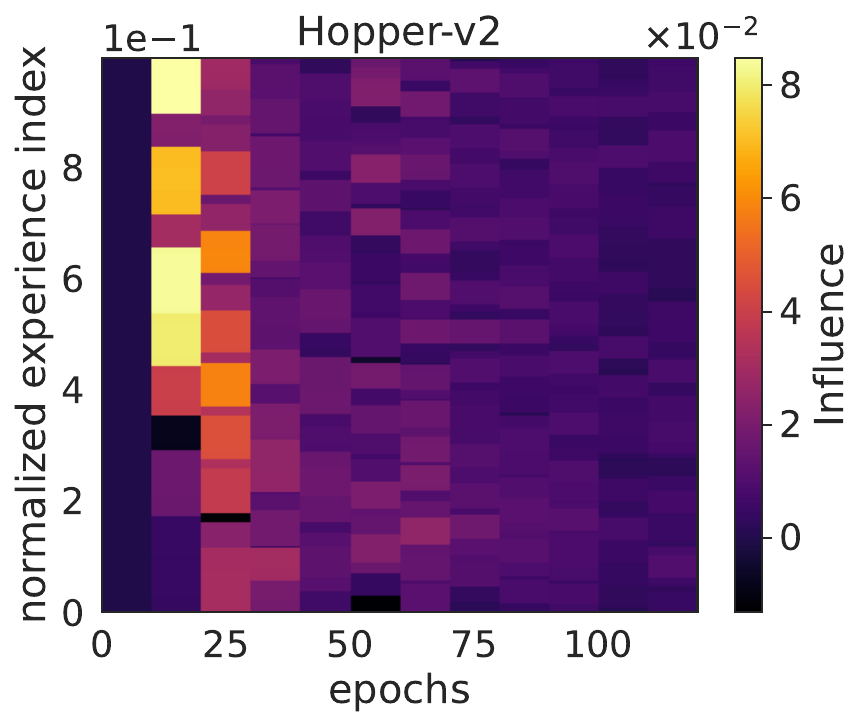}
\includegraphics[clip, width=0.245\hsize]{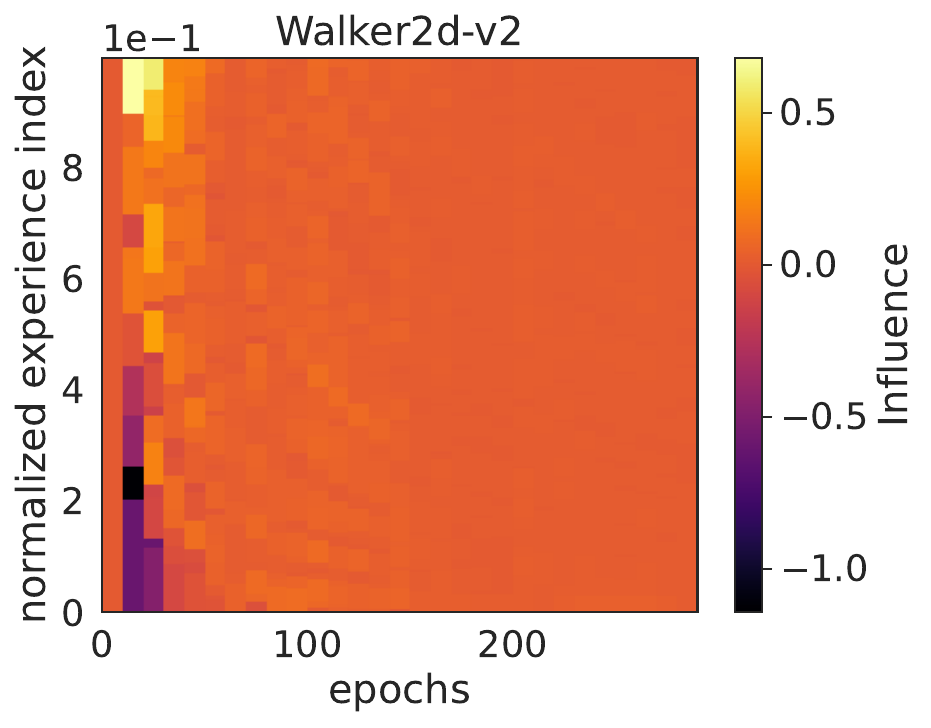}
\includegraphics[clip, width=0.245\hsize]{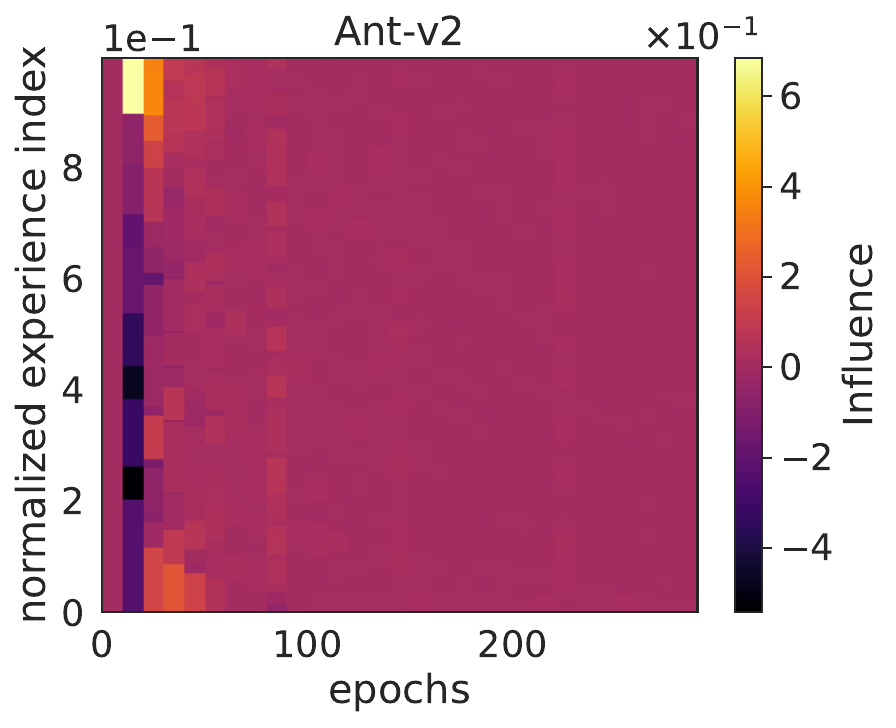}
\includegraphics[clip, width=0.245\hsize]{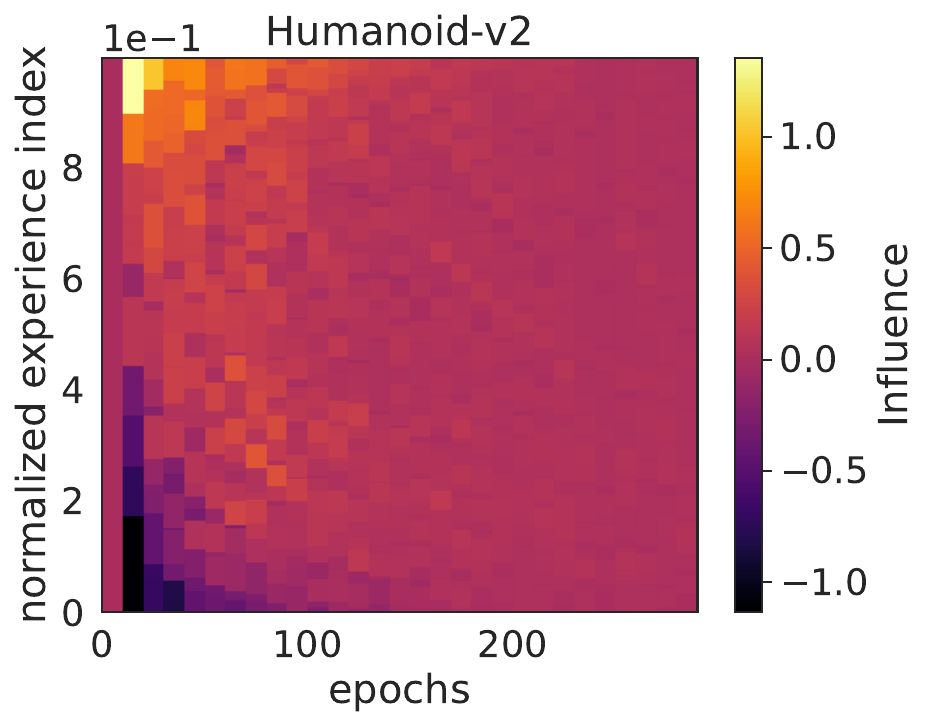}
\subcaption{Distribution of influence on Q-estimation bias (Eq.~\ref{eq:influence_bias}).}
\end{minipage}
%
\caption{
Distribution of influence on return and Q-estimation bias for all ten trials. 
}
\label{fig:distribution_of_bias_return_droq}
\end{figure*}

\clearpage
\subsection{Additional experimental results for ResetToD}\label{app:asstional_results_reset}
\begin{figure*}[h!]
\begin{minipage}{1.0\hsize}
\includegraphics[clip, width=0.49\hsize]{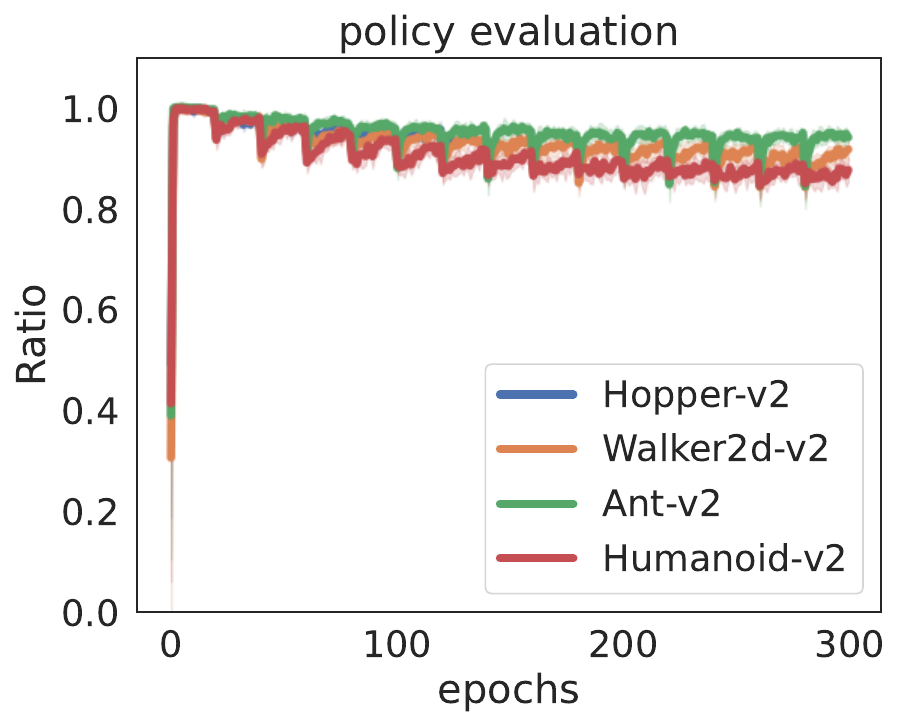}
\includegraphics[clip, width=0.49\hsize]{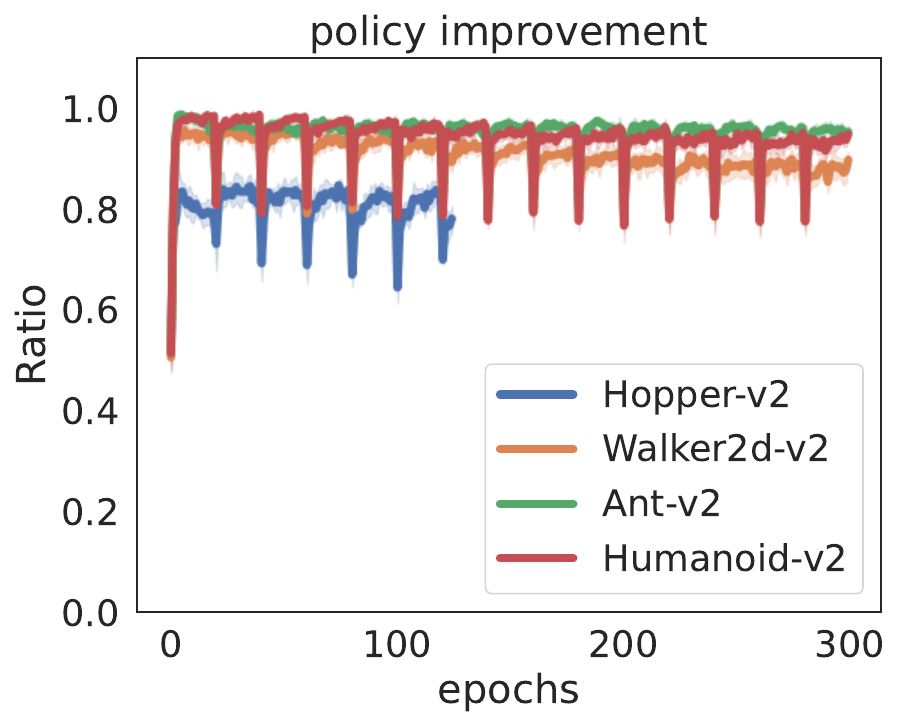}
\end{minipage}
\caption{
The ratio of experiences for which ResetToD correctly estimated their self-influence. 
}
\label{fig:raio_of_experiences_with_pos_neg_influence_reset}
\end{figure*}
\begin{figure*}[h!]
\begin{minipage}{1.0\hsize}
\includegraphics[clip, width=0.245\hsize]{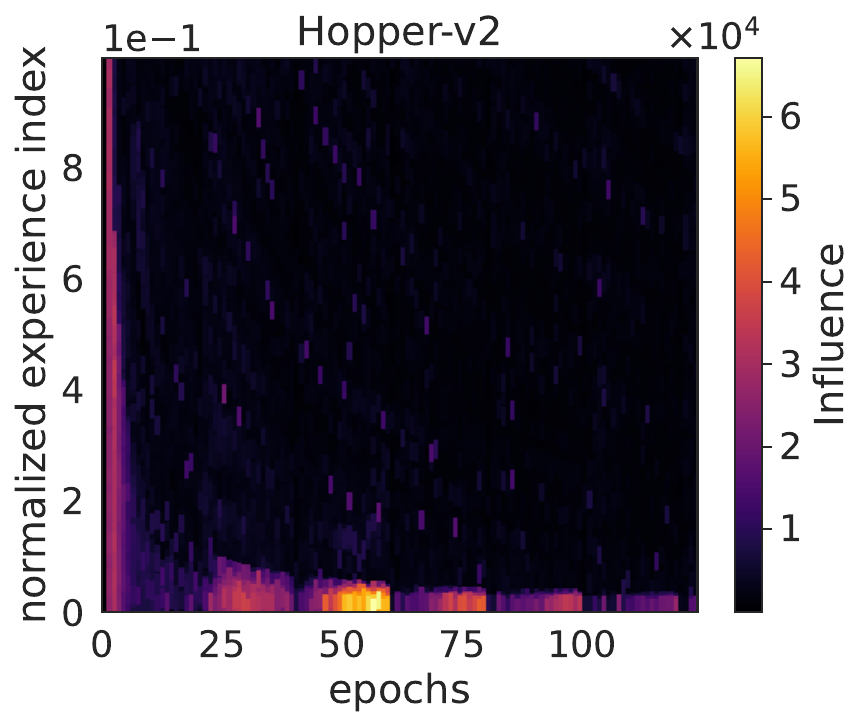}
\includegraphics[clip, width=0.245\hsize]{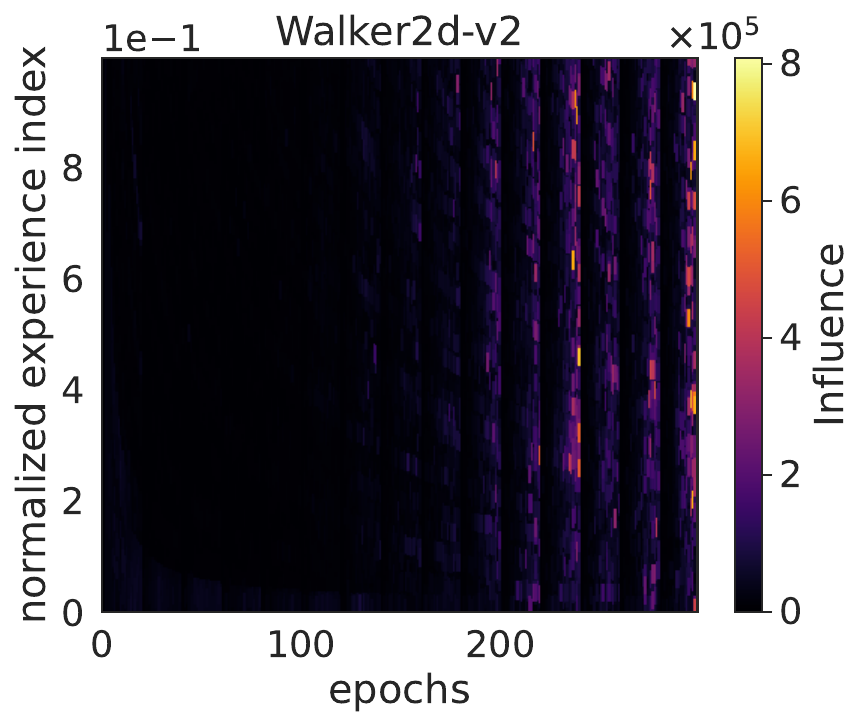}
\includegraphics[clip, width=0.245\hsize]{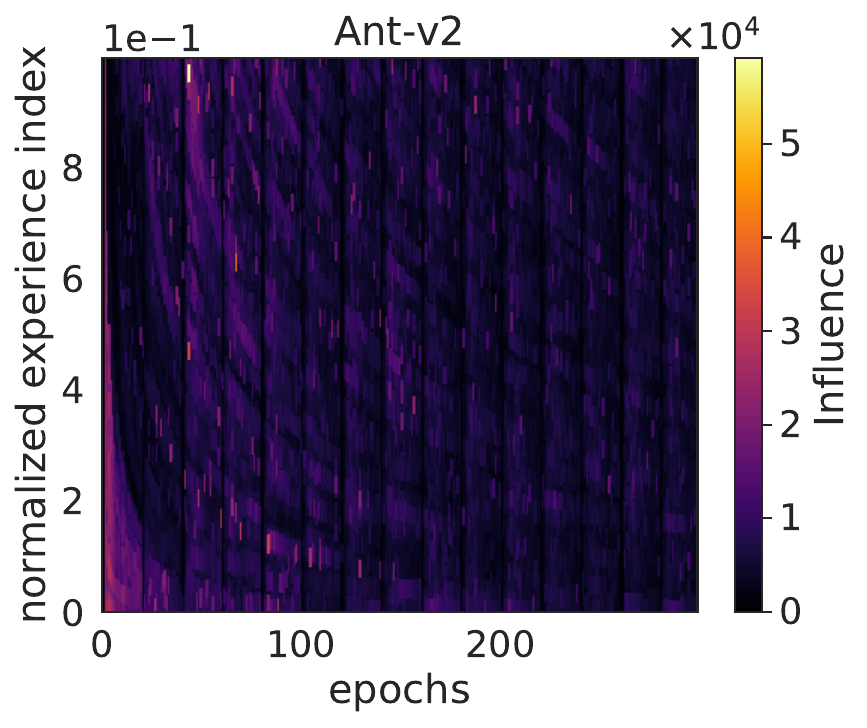}
\includegraphics[clip, width=0.245\hsize]{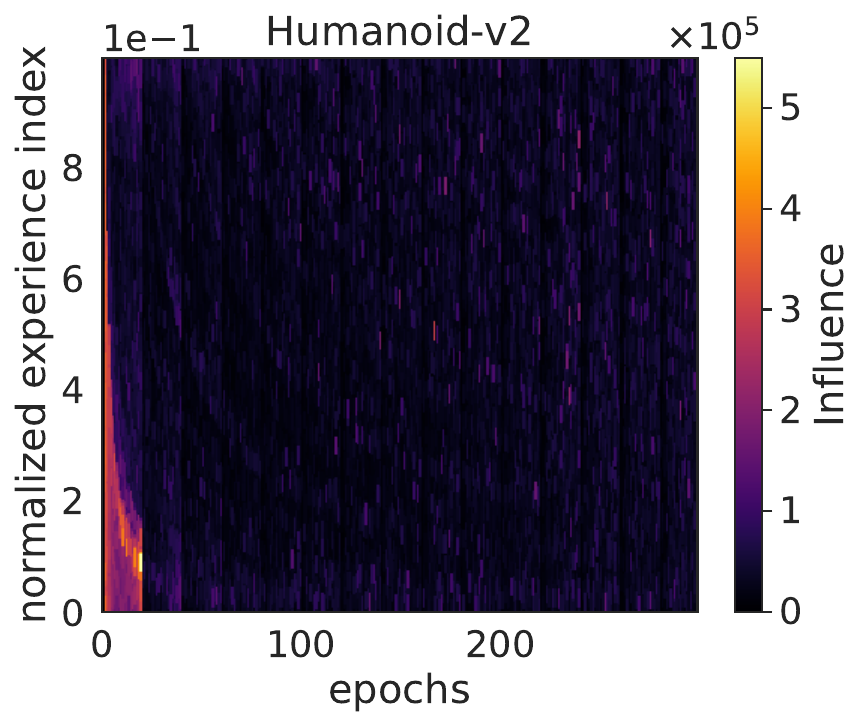}
\subcaption{Distribution of self-influence on policy evaluation (Eq.~\ref{eq:self_infl_q_func}).}
\end{minipage}
\begin{minipage}{1.0\hsize}
\includegraphics[clip, width=0.245\hsize]{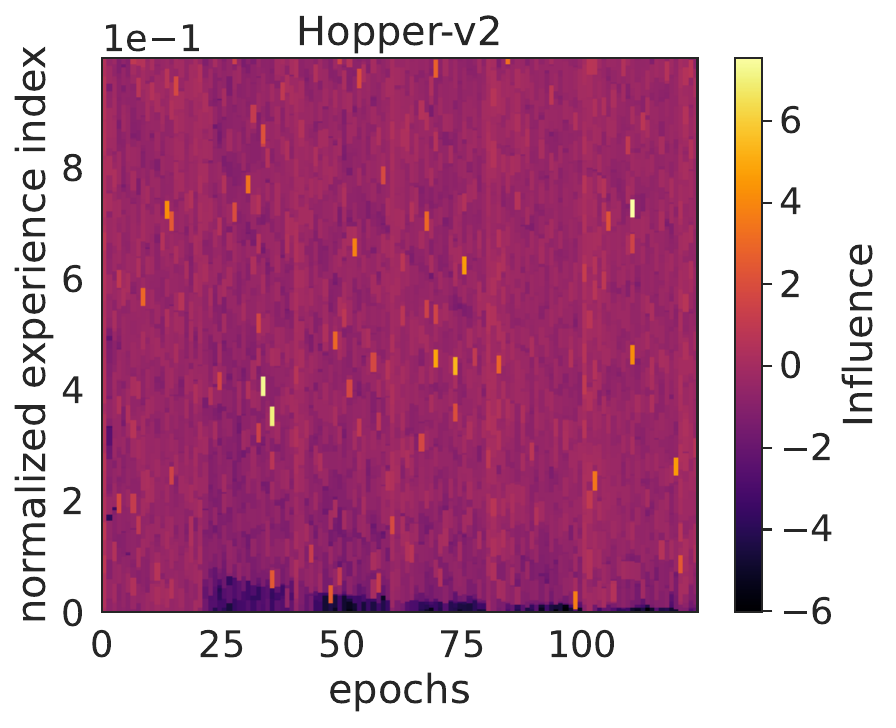}
\includegraphics[clip, width=0.245\hsize]{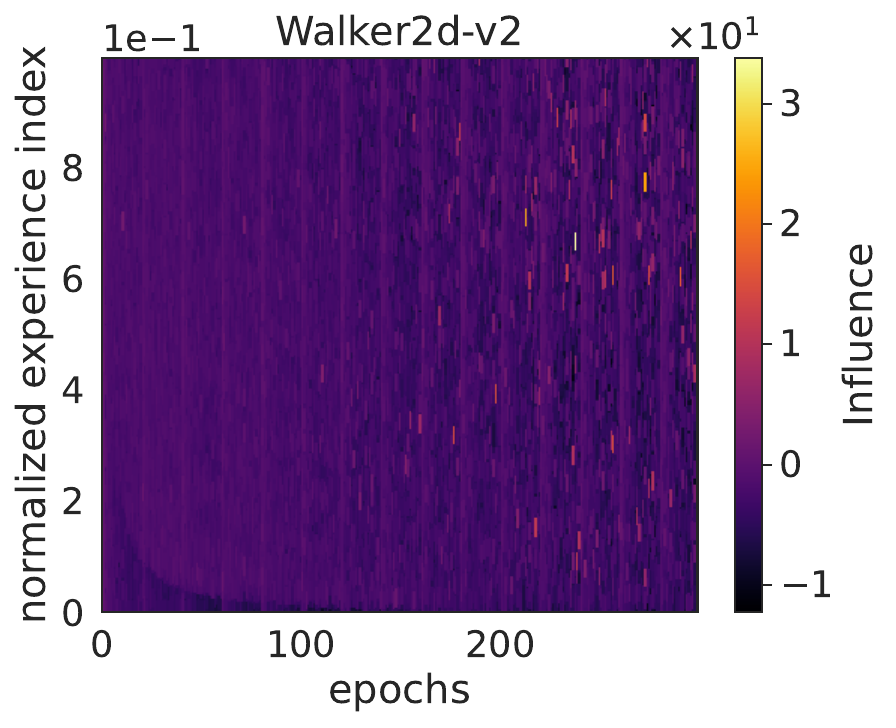}
\includegraphics[clip, width=0.245\hsize]{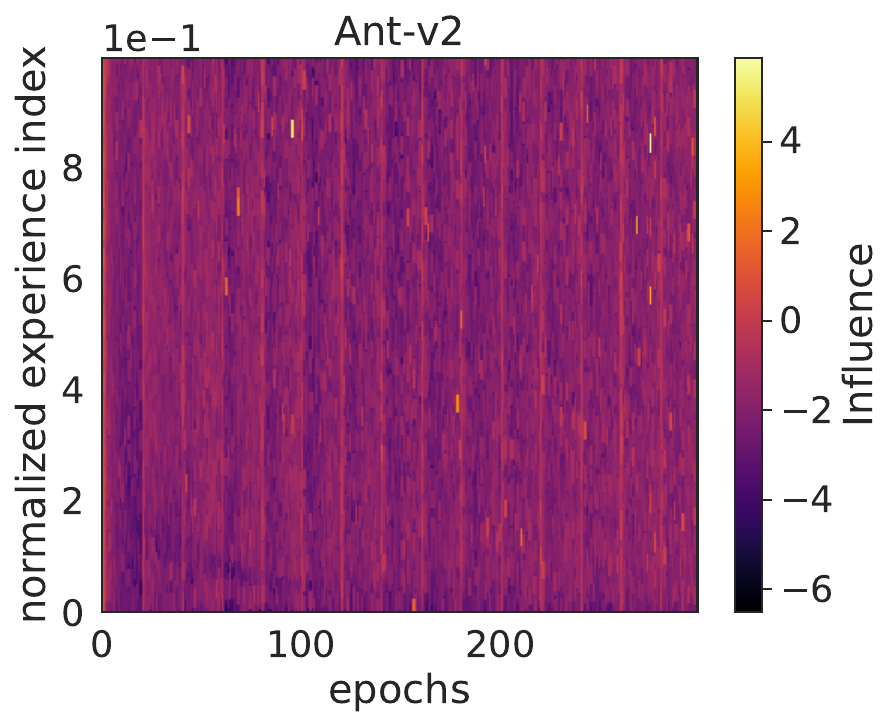}
\includegraphics[clip, width=0.245\hsize]{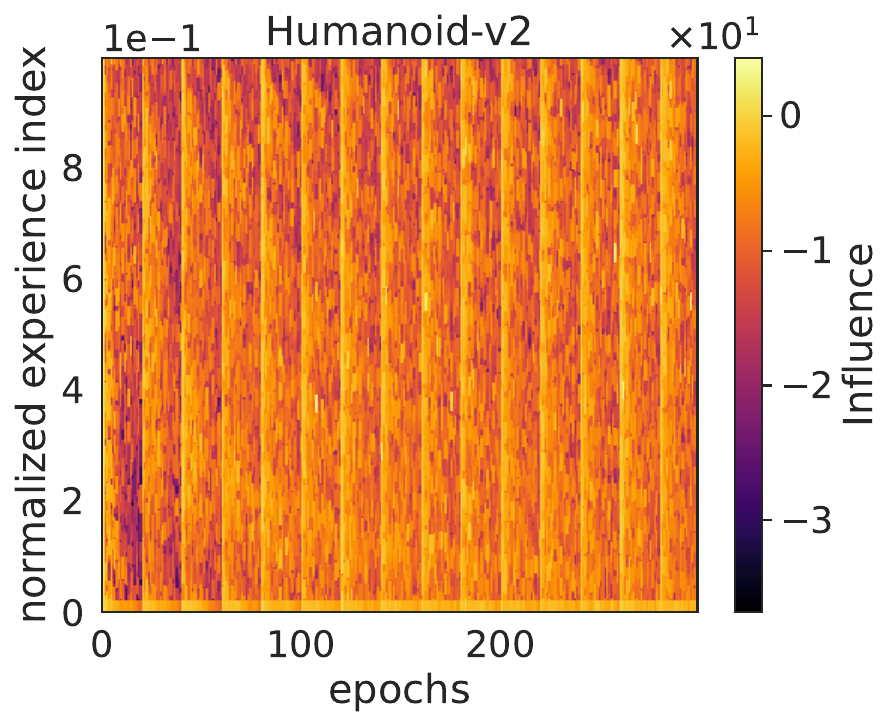}
\subcaption{Distribution of self-influence on policy improvement (Eq.~\ref{eq:self_infl_policy_func}).}
\end{minipage}
%
\caption{
Distribution of self-influence on policy evaluation and policy improvement. 
}
\label{fig:distribution_of_self_influences_reset}
\end{figure*}
\begin{figure*}[h!]
\begin{minipage}{1.0\hsize}
\includegraphics[clip, width=0.49\hsize]{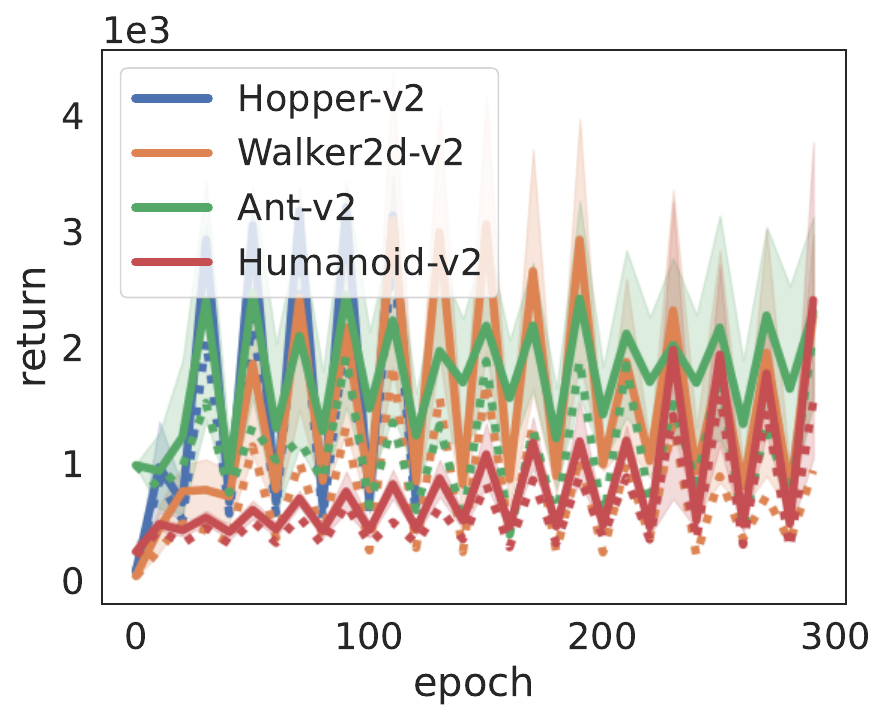}
\includegraphics[clip, width=0.49\hsize]{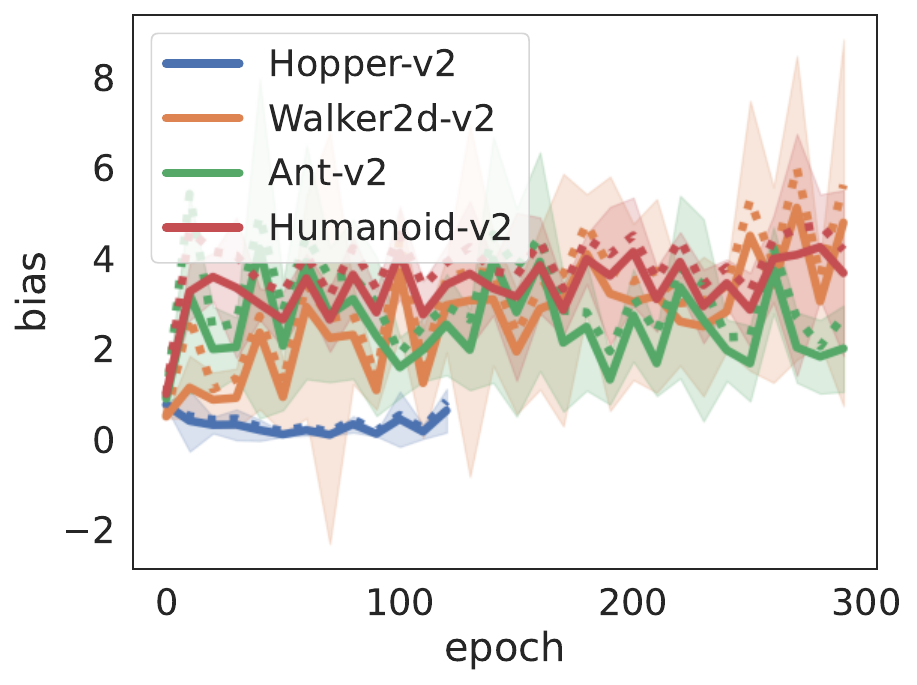}
\end{minipage}
\caption{
Results of policy amendments (left) and Q-function amendments (right) for all ten trials. 
}
\label{fig:cleansing_results_average_case_reset}
\end{figure*}
\begin{figure*}[h!]
\begin{minipage}{1.0\hsize}
\includegraphics[clip, width=0.245\hsize]{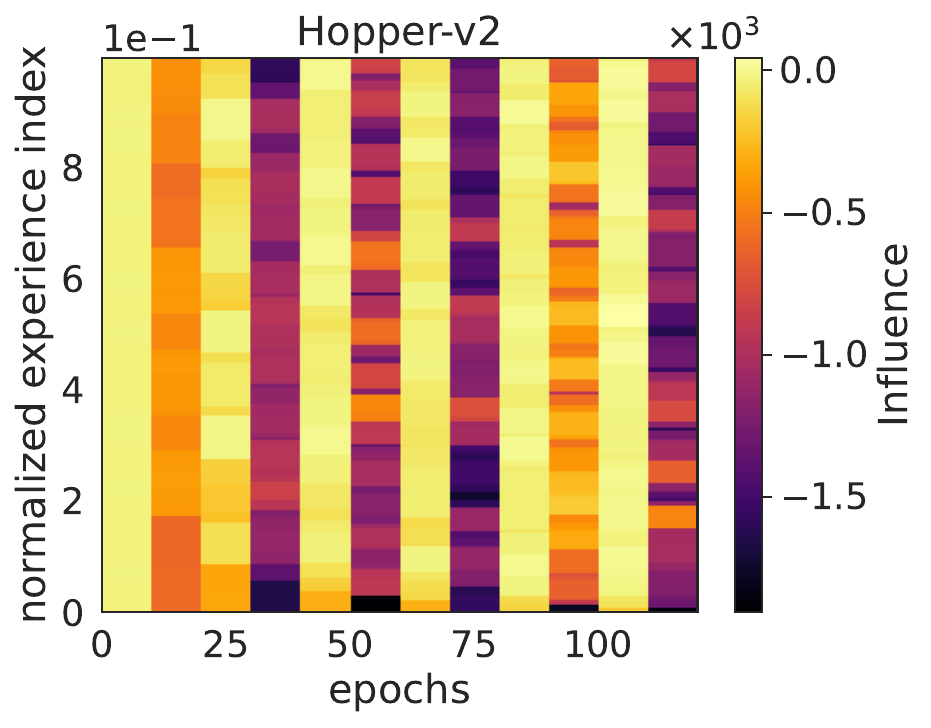}
\includegraphics[clip, width=0.245\hsize]{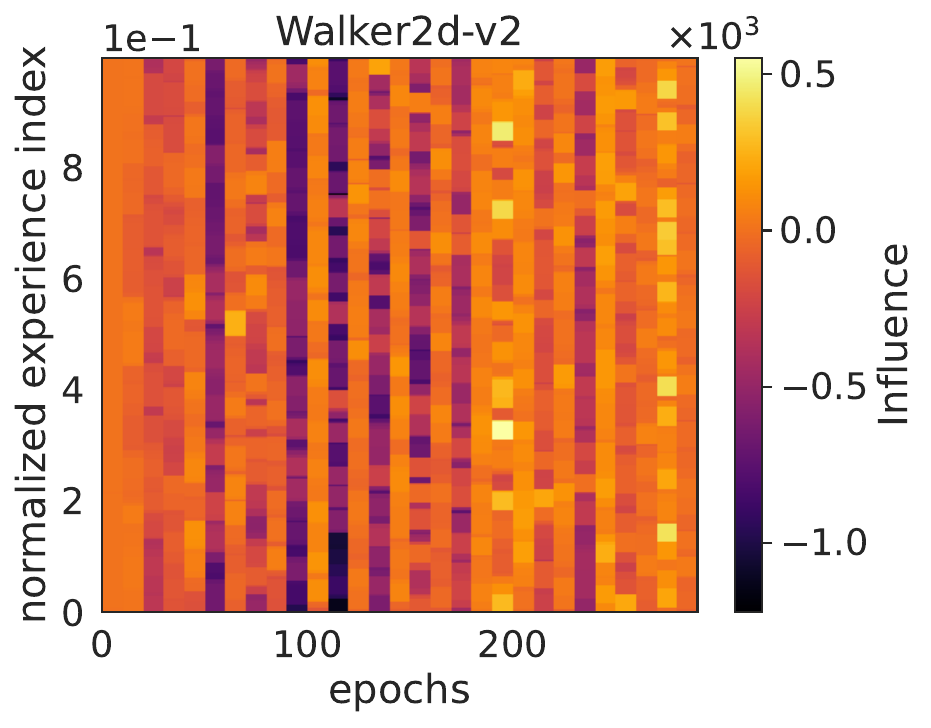}
\includegraphics[clip, width=0.245\hsize]{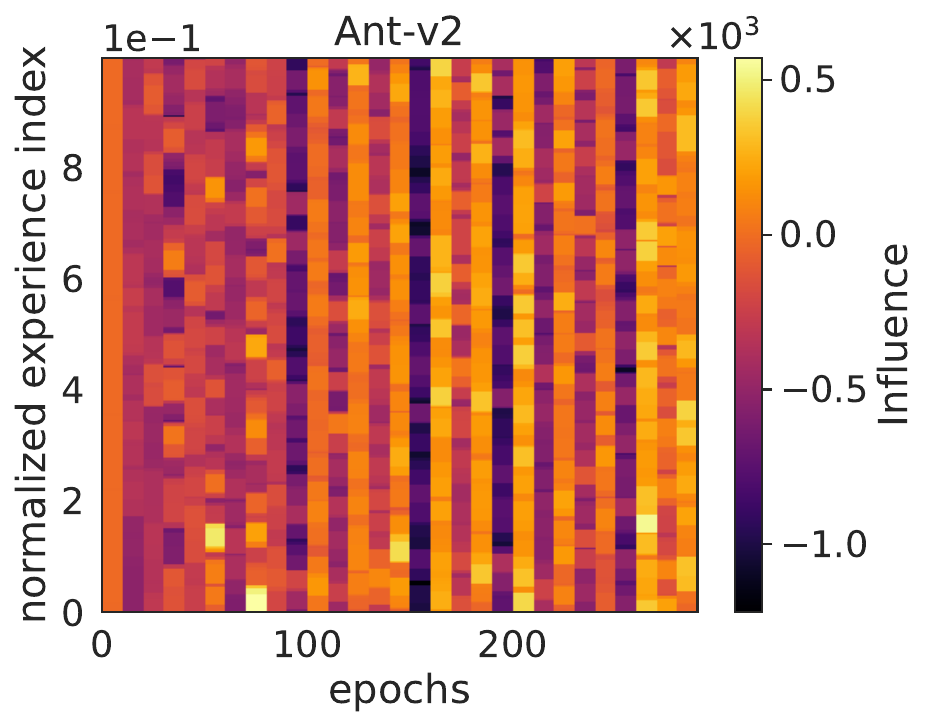}
\includegraphics[clip, width=0.245\hsize]{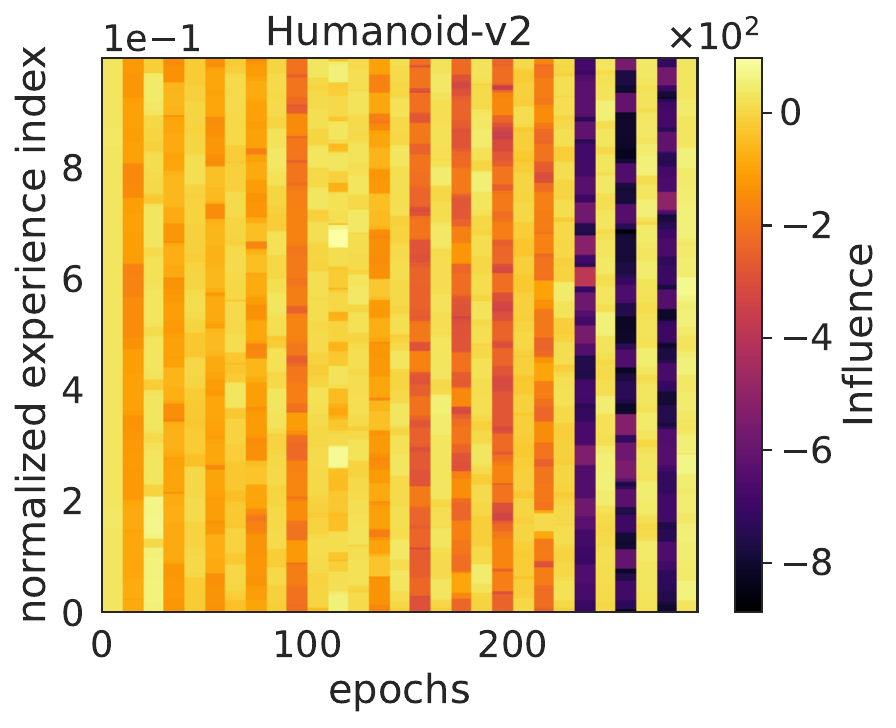}
\subcaption{Distribution of influence on return (Eq.~\ref{eq:influence_return}).}
\end{minipage}
\begin{minipage}{1.0\hsize}
\includegraphics[clip, width=0.245\hsize]{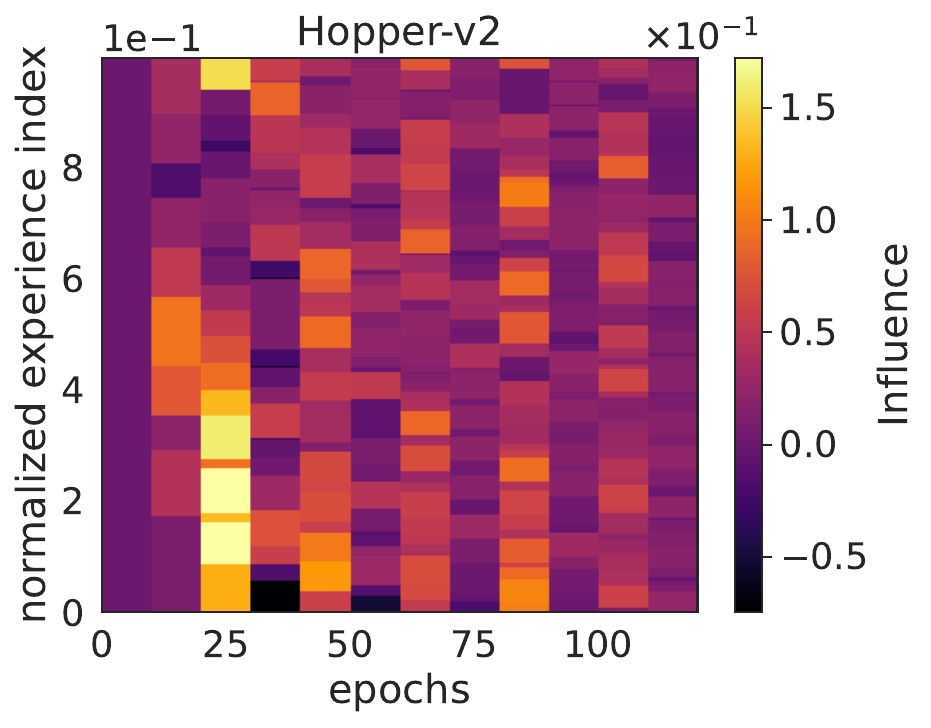}
\includegraphics[clip, width=0.245\hsize]{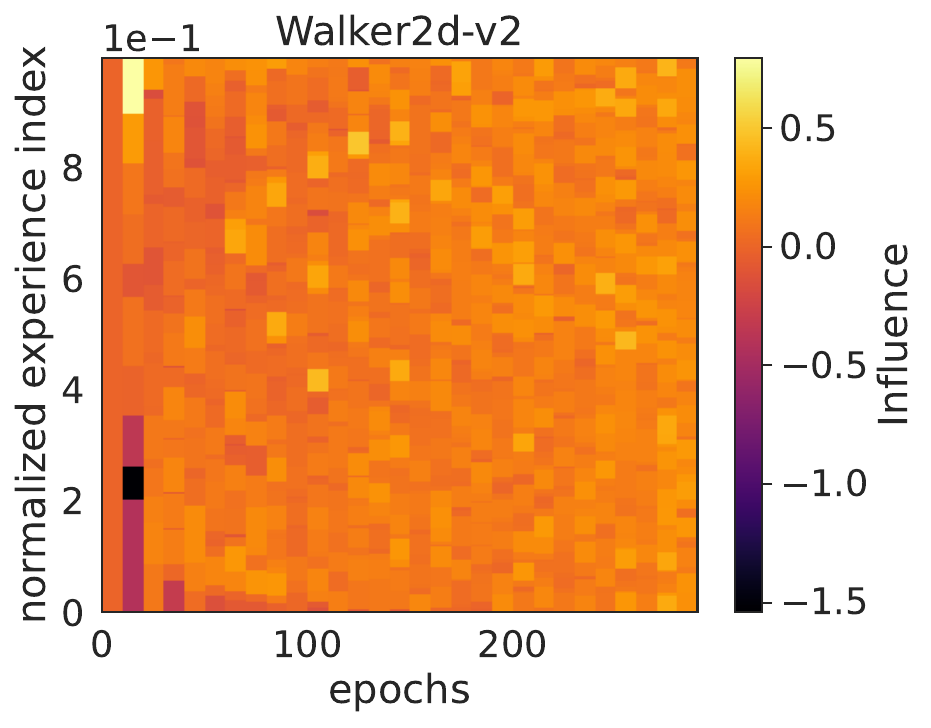}
\includegraphics[clip, width=0.245\hsize]{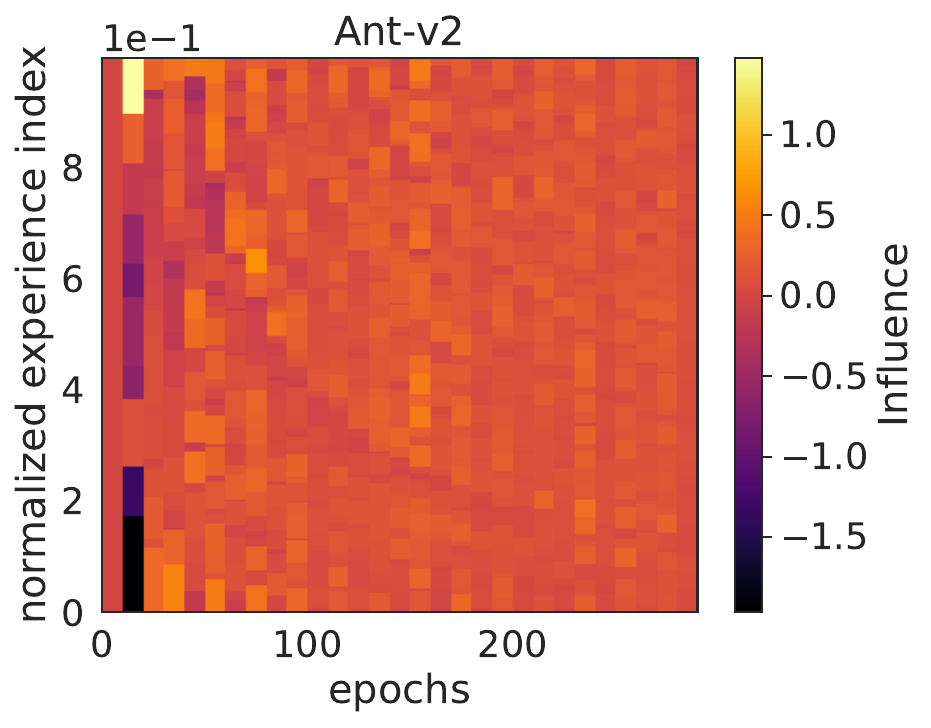}
\includegraphics[clip, width=0.245\hsize]{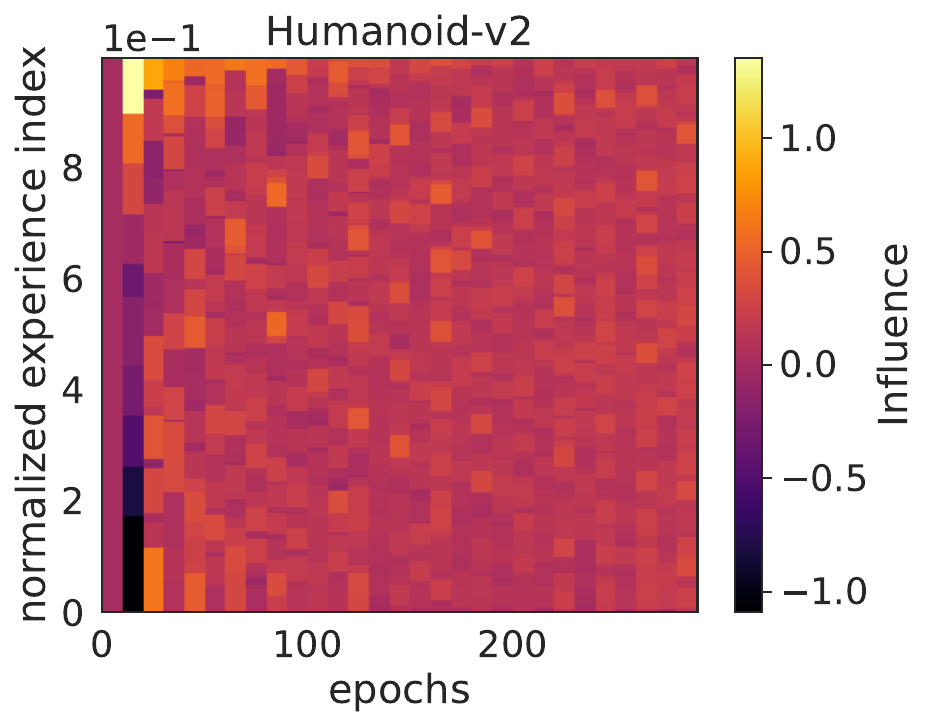}
\subcaption{Distribution of influence on Q-estimation bias (Eq.~\ref{eq:influence_bias}).}
\end{minipage}
%
\caption{
Distribution of influence on return and Q-estimation bias for all ten trials. 
}
\label{fig:distribution_of_bias_return_reset}
\end{figure*}

\clearpage
\section{Limitations and future work}\label{sec:limitations_future_work}
\textbf{Influence on exploration in RL algorithms.} In this paper, we do not consider the influence of removing experiences on the exploration process of RL algorithms (i.e., line 3 in Algorithm~\ref{alg2:PI_PIToD}). 
Considering such an influence would be an interesting direction for future research. 

\textbf{Refining implementation decisions for PIToD.} 
PIToD employs a dropout rate of $0.5$ (Section~\ref{sec:proposed_method} and Appendix~\ref{sec:analyzing_overlap_in_mask}), which can often lead to degradation in learning performance. 
To mitigate this issue, we have considered various design choices in the implementation of PIToD (Appendix~\ref{sec:practical_implementation}). 
However, further refinement may still be necessary to improve the practicality of PIToD.

\textbf{Overlap of masks.} 
PIToD assigns each experience a randomly generated binary mask (Section~\ref{sec:proposed_method}). 
When there is significant overlap between mask elements, applying the flipped mask to delete the influence of a specific experience also deletes the influence of other experiences. 
As a pathological example, if the masks $\mathbf{m}_i$ and $\mathbf{m}_{i'}$ corresponding to the experiences $e_i$ and $e_{i'}$ have a $100\%$ overlap, applying the flipped mask $\mathbf{w}_i$ completely deletes the influence of both $e_i$ and $e_{i'}$. 
Additionally, significant overlap between masks may hinder the fulfillment of Assumption~\ref{ass:sparsity} and thus compromise the theoretical property derived in Appendix~\ref{app:theory}. 
We set the dropout rate of the mask elements to minimize this overlap, but a $50\%$ overlap may still occur (Appendix~\ref{sec:analyzing_overlap_in_mask}). 
Developing practical methods to reduce mask overlap across experiences would be an important direction for future work. 

\textbf{Invasiveness of PIToD.} 
PIToD introduces invasive changes to the base PI method (e.g., DDPG or SAC) to equip it with efficient influence estimation capabilities (Section~\ref{sec:proposed_method}). 
Specifically, PIToD incorporates turn-over dropout, which may affect the learning outcomes of the base PI method. 
Consequently, PIToD may not be suitable for estimating the influence of experiences on the original learning outcomes of the base PI method. 
One direction for future work is to explore non-invasive influence estimation methods.

\textbf{Exploring surrogate evaluation metrics for amendments.} 
To amend RL agents in Section~\ref{sec:application}, we used the return-based evaluation metric $L_{\text{ret}}$, which requires additional environment interactions for evaluation. 
In our case, evaluating $L_{\text{ret}}$ required as many as $3 \cdot 10^6$ interactions (Figure~\ref{fig:additional_environment_interaction_for_amendment} in Appendix~\ref{app:additional_results}). 
These additional interactions may become a bottleneck in settings where interaction with environments is costly (e.g., real-world or slow simulator environments). 
Exploring surrogate evaluation metrics that do not require additional interactions is an interesting research direction.

\textbf{Exploring broader applications of PIToD.} 
In this paper, we applied PIToD to amend RL agents in single-task RL settings (Section~\ref{sec:application}, Appendix~\ref{app:adversarial_dmc}, and Appendix~\ref{app:amend-droq-reset}). 
However, we believe that the potential applications of PIToD extend beyond single-task RL settings. 
For instance, it could be applied to multi-task RL~\citep{vithayathil2020survey} (including multi-goal RL~\citep{liu2022goal} or meta RL~\citep{beck2023survey}), continual RL~\citep{khetarpal2022towards}, safe RL~\citep{gu2022review}, offline RL~\citep{levine2020offline}, or multi-agent RL~\citep{canese2021multi}. 
Investigating the broader applicability of PIToD in these settings is a interesting direction for future work. 
Additionally, in this paper, we estimated the influence of experiences by assigning masks to experiences. 
We may also be able to estimate the influence of specific hyperparameter values by assigning masks to them. Exploring such applications is another interesting direction for future work.

\clearpage
\section{Computational resources used in experiments}\label{app:computational_resources}
For our experiments in Section~\ref{sec:eval_computational_time}, we used a machine equipped with two Intel Xeon E5-2667 v4 CPUs and five NVIDIA Tesla K80 GPUs. 
For the experiments in Section~\ref{app:adversarial_dmc}, we used a machine equipped with
two Intel Xeon Gold 6148 CPUs and four NVIDIA V100 SXM2 GPUs.

\section{Hyperparameter settings}\label{app:hypara}
The hyperparameter settings for our experiments (Sections~\ref{sec:experiments} and \ref{sec:application}) are summarized in Table~\ref{tab:hyerparametsers}.
Basic hyperparameters, such as the learning rate and discount factor, are set as in \citet{chen2021randomized}. 
In contrast, the number of hidden units and the replay ratio are set to smaller values than in \citet{chen2021randomized} due to computational resource constraints. 
The replay buffer size is set large enough to store all experiences collected during training. 
The masking (dropout) rate is set to minimize overlap between masks, following the discussion in Section~\ref{sec:analyzing_overlap_in_mask}. 
We use different values of $I_{\text{ie}}$ in Sections~\ref{sec:experiments} and \ref{sec:application}.  
In Section~\ref{sec:experiments}, we employ computationally lighter implementations of the evaluation metric $L$ (i.e., $L_{\text{pe}, i}$ and $L_{\text{pi}, i}$), which allows influence estimation to be performed more frequently; therefore, we set $I_{\text{ie}} = 5000$.  
In contrast, in Section~\ref{sec:application}, we use heavier implementations of $L$ (i.e., $L_{\text{ret}}$ and $L_{\text{bias}}$), and thus set $I_{\text{ie}} = 50000$.
\begin{table}[h!]
\caption{Hyperparameter settings}
\label{tab:hyerparametsers}
\begin{center}
\scalebox{0.9}{
\begin{tabular}{l|l}\hline
 Parameter                               & Value          \\\hline\hline
 optimizer                               & Adam~\citep{kingma2014adam}           \\\hline
 learning rate                           & $0.0003$ \\\hline
 discount rate $\gamma$                & 0.99           \\\hline
 target-smoothing coefficient $\rho$   & 0.005          \\\hline
 replay buffer size                      & $2 \cdot 10^6$         \\\hline 
 number of hidden layers for all networks & 2              \\\hline
 number of hidden units per layer        & 128           \\\hline 
 mini-batch size                         & 256            \\\hline
 random starting data                    & 5000           \\\hline
 replay (update-to-data) ratio           & 4             \\\hline
 masking (dropout) rate & 0.5 \\\hline
 influence estimation interval $I_{\text{ie}}$ & 5000 for Section~\ref{sec:experiments} and 50000 for Section~\ref{sec:application} \\\hline
\end{tabular}
}
\end{center}
\end{table}


\end{document}